\pdfoutput=1
\documentclass{article}
\usepackage[preprint]{neurips_2024}
\usepackage[T1]{fontenc}
\usepackage[utf8]{inputenc}
\usepackage{amsmath,amssymb,amsthm,mathtools}
\usepackage[authoryear,round]{natbib}
\usepackage{booktabs,array,microtype,xcolor}
\usepackage{graphicx}
\usepackage[most]{tcolorbox}
\usepackage{enumitem}
\usepackage[colorlinks=true,linkcolor=blue!55!black,citecolor=blue!55!black,urlcolor=blue!55!black]{hyperref}

\newcommand{\R}{\mathbb R}
\newcommand{\E}{\mathbb E}
\newcommand{\Kset}{\mathcal K}
\newcommand{\Yset}{\mathcal Y}
\newcommand{\Nset}{\mathcal N}
\newcommand{\Aset}{\mathcal A}
\newcommand{\Tset}{\mathcal T}
\newcommand{\Pcal}{\mathcal P}
\newcommand{\Xcal}{\mathcal X}
\newcommand{\Prob}{\mathbb P}
\newcommand{\ind}{\mathbf 1}
\newcommand{\indic}{\mathbb{I}}
\newcommand{\norm}[1]{\lVert#1\rVert}
\newcommand{\PR}{\mathrm{PR}}
\newcommand{\CKA}{\mathrm{CKA}}
\newcommand{\tr}{\operatorname{tr}}
\DeclareMathOperator{\Comp}{Comp}
\DeclareMathOperator{\NN}{NN}

\theoremstyle{definition}
\newtheorem{definition}{Definition}
\newtheorem{remark}{Remark}
\theoremstyle{plain}
\newtheorem{theorem}{Theorem}
\newtheorem{proposition}{Proposition}
\newtheorem{lemma}{Lemma}
\newtheorem{corollary}{Corollary}
\newtheorem{conjecture}{Conjecture}
\newtheorem{hypothesis}{Hypothesis}

\tcbset{housebox/.style={enhanced,colback=blue!5,colframe=blue!35!black,boxrule=0.6pt,arc=1.5pt,left=7pt,right=7pt,top=5pt,bottom=5pt}}
\newtcolorbox[auto counter]{insight}[2][]{housebox,before skip=8pt plus 2pt,after skip=8pt plus 2pt,
  before upper={\textbf{Insight~\thetcbcounter: #2.}\ },#1}

\title{From Plato to Aristotle: Toward a General Theory of Representation Hypotheses}
\author{%
\textbf{Dai Shi}\,\textsuperscript{1}\thanks{Corresponding author: \texttt{ds2213@cam.ac.uk}} \quad
\textbf{Andi Han}\,\textsuperscript{2} \quad
\textbf{Xiaoyu Li}\,\textsuperscript{3} \\
\textbf{Luke Thompson}\,\textsuperscript{2} \quad
\textbf{Jos\'e Miguel Hern\'andez-Lobato}\,\textsuperscript{1} \\
{\footnotesize \textsuperscript{1}University of Cambridge \quad
\textsuperscript{2}University of Sydney \quad \textsuperscript{3}University of New South Wales}
}

\theoremstyle{plain}
\makeatletter
\newcommand{\formalofname}{}
\newtheorem*{formalof@inner}{\formalofname}
\newenvironment{formalof}[1]{\renewcommand{\formalofname}{#1}\begin{formalof@inner}}{\end{formalof@inner}}
\makeatother

\definecolor{thmtint}{HTML}{F3F7FC}
\definecolor{thmblue}{HTML}{1F4FA8}
\newcommand{\thmboxset}[1]{\tcolorboxenvironment{#1}{enhanced jigsaw, breakable,
  boxrule=0pt, frame hidden, colback=thmtint, arc=1.2mm,
  borderline west={1.5pt}{0pt}{thmblue!70},
  left=7pt, right=7pt, top=4pt, bottom=4pt,
  before skip=11pt, after skip=11pt}}
\AtBeginDocument{\thmboxset{theorem}\thmboxset{proposition}\thmboxset{lemma}%
  \thmboxset{corollary}\thmboxset{conjecture}\thmboxset{formalof}}

\begin{document}
\maketitle

\begin{abstract}
The Platonic Representation Hypothesis (PRH) holds that increasingly
capable models converge toward a common representation. However, recent
studies have shown different views on aligning representations via their
neighborhoods, spanning trees, and distances. We show that these views
can differ because the alignment criteria behind them sit at different
points of one lattice, and we formalize this explanation as the Lattice
Representation Hypothesis (LRH). Under the LRH, the convergence claimed
by the PRH is agreement at the top of the lattice, whereas its evidence
comes from neighbor overlap, a criterion near the bottom. The
Aristotelian Representation Hypothesis (ARH) restricts the claimed
convergence to local neighborhoods because calibration keeps the growth
of neighbor overlap with model size but largely removes that of centered
kernel alignment. We show that this removal can come from calibrating a
maximum over layers, so global convergence remains open. Back into
Plato's Cave (BPC) infers coarse shared structure from the collapse of
neighbor overlap on larger evaluation sets, whereas we trace this
collapse at least partly to a fixed number of neighbors covering shorter
distances. We measure the spanning-tree overlap used by Structure
over Geometry (SOG) and show that it mostly repeats
one-nearest-neighbor overlap. The LRH also extends to independently
trained models, whose training procedure and objective decide the
criteria met reliably across runs. We prove this claim for the training
procedure in a linear autoencoder and for the objective in a finite
contrastive Gibbs model. Experiments on dozens of pretrained vision and
language models and on thousands of runs of both theoretical models
support these conclusions. We hope this work paves the way toward a
general theory of representation hypotheses, with a framework for
analyzing these hypotheses and comparing their evidence. Our project page can be
found at \url{https://eethanshi.github.io/LRH}.
\end{abstract}

\section{Introduction}
\label{sec:intro}

Representation alignment across modalities has become central to
training large multimodal models \citep{liu2023visual} and matching
entities between multimodal knowledge graphs \citep{liu2021visual},
especially as models become larger and more capable
\citep{kaplan2020scaling} and their representations appear to grow more
alike. The Platonic Representation Hypothesis (PRH)
\citep{huh2024platonic} states that networks trained with different
objectives, data, and modalities converge toward a shared statistical
model of reality. PRH supports this claim by investigating whether
vision and language models share the $k$ nearest neighbors of each
image--caption pair and how this overlap grows with model capability,
using the mutual $k$-nearest-neighbor overlap (mKNN).
The Aristotelian Representation Hypothesis (ARH)
\citep{groger2026revisiting} questions and further studies this evidence
by calibrating scores against randomly paired inputs. Global
measures such as centered kernel alignment (CKA) largely lose the
size trend kept by mKNN, and ARH narrows the claimed convergence to
shared local neighborhoods. The Back into Plato's Cave study (BPC)
\citep{koepke2026back} tests this neighbor evidence on larger evaluation
sets and concludes that the shared structure is coarse, since mKNN at a
fixed $k$ collapses unless $k$ grows in proportion. The Structure over
Geometry study (SOG) \citep{you2026structure} challenges ARH's local view
by separating relational structure from metric geometry, and finds
this structure converging locally in neighbor sets and globally in
minimum spanning trees (MSTs).

The four studies thus offer competing hypotheses about the structure
shared by the same pretrained models. However, these hypotheses are drawn
from measures of different strength, applied under different settings
such as the numbers of evaluation samples and neighbors, the compared
layers, and calibration. If the models represent the same reality, it is natural to
ask \emph{whether the conflict comes from different ways of viewing their
representations}, which would require a general theory relating these
views.

To answer this question, we define an alignment criterion as a test that
compares how two models represent a shared set of evaluation samples,
through one relation among these samples (e.g., the neighbor sets in mKNN
or the distances in calibrated CKA), and accepts the pair of models if
the two relations agree within a tolerance under a fixed scale and
protocol. We call one
criterion stronger than another if every pair accepted by the first is
accepted by the second, so one pair of models can satisfy a criterion and
fail a stronger one. The criteria then form a lattice, on which
Figure~\ref{fig:lattice} places the four studies. Furthermore, as
studies also claim that networks trained with different objectives learn
similar representations \citep{huh2024platonic,bansal2021revisiting}, it
then remains to explore which criteria can be satisfied by a pair of
independently trained models. These two questions motivate our
\textit{lattice representation hypothesis}.

\begin{figure}[t]
\centering
\includegraphics[width=5.5in]{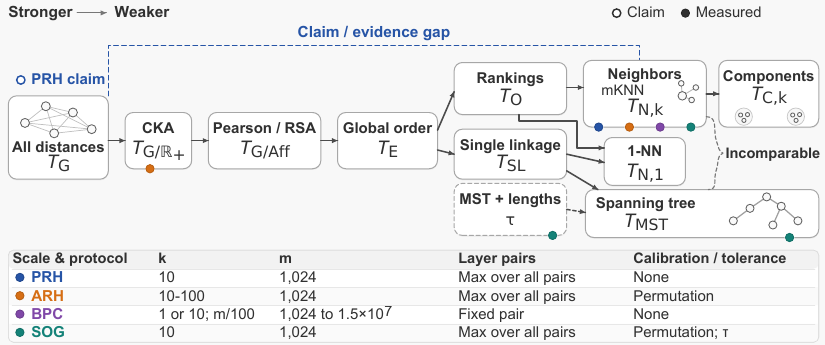}
\caption{The lattice of alignment criteria for the four studies, from strongest (left) to weakest (right).}
\label{fig:lattice}
\vspace{-8pt}
\end{figure}

\begin{tcolorbox}[housebox]
\textbf{Lattice Representation Hypothesis (LRH).} \emph{In the lattice of alignment criteria, the set of criteria accepting two models describes what is shared by their representations. Studies have conflicting conclusions because they apply
criteria at different points of this lattice. For independently trained models, the training
procedure and the objective decide the criteria met reliably across runs,
each in a different part of the lattice.}
\end{tcolorbox}

Under this lattice view, we obtain the following results.
\begin{enumerate}[leftmargin=1.6em,itemsep=0pt,topsep=2pt,parsep=0pt]
\item PRH's claim amounts to agreement under the criterion at the top of
the lattice, where its idealized learners agree on all distances
(Section~\ref{sec:geometry}). Its
evidence comes from mKNN, a criterion near the bottom that compares
neighbor sets (Proposition~\ref{prop:lattice}).

\item The loss of the CKA trend in ARH can come from calibrating a maximum
over layers, whose chance level grows with effective dimension
(Proposition~\ref{prop:baselines}, Corollary~\ref{cor:aggregation}). On
LLaMA-1, calibration turned the trend $0.68$ into $-0.85$, a drop
predicted from single-model features. With the preprocessing of PRH, the
trend of calibrated CKA in language-model size keeps its sign at mid depth
and at the last layer, so global convergence remains open
(Section~\ref{sec:exp-aggregation}).

\item A larger evaluation set at a fixed $k$ draws the neighbor boundary
closer to each input, so the collapse of mKNN reported by BPC is at least
partly a scale effect (Theorem~\ref{thm:crossing}). As in BPC,
mKNN with sentence encoders (preregistered) and with last-layer vision
features (exploratory) stays close to stable at a fixed $k/m$ for $k=10$
and $k=100$ (Section~\ref{sec:exp-scale} and Appendix~\ref{app:limitations}).

\item Every MST contains the nearest-neighbor edge of each input
(Proposition~\ref{prop:mst}), and on the studied vision--language pairs
these edges make up about $88\%$ of the shared tree edges
(Section~\ref{sec:exp-criteria}). On these pairs, the MST overlap used by
SOG thus mostly repeats one-nearest-neighbor overlap, while the number
of other shared edges exceeds its chance level without an established size trend.

\item The training procedure of a linear autoencoder decides whether
distances agree across runs by selecting the invertible map undone by its
linear head (Theorem~\ref{thm:gauge}). In a finite
contrastive Gibbs model, two learners trained with different contrastive
objectives agree on the order of distances for large samples and differ
in geometry, since
they learn one relation through different increasing functions
(Theorem~\ref{thm:link}). Preregistered tests confirm the main
predictions of both theorems in these models, whereas results on trained
networks outside the premises of the theorems contradict several of these
predictions (Sections~\ref{sec:exp-gauge} and~\ref{sec:exp-link}).
\end{enumerate}

\section{Alignment Criteria and Their Lattice}
\label{sec:criteria}

To state the lattice representation hypothesis formally as
Hypothesis~\ref{hyp:lrh}, we define alignment criteria, their order by
strength, and the probability that two independently trained models meet
a criterion.

\paragraph{Setting and Notation.}
We compare two models with representation maps $h_A$ and $h_B$ on $m$
evaluation samples $S=((x_i^A,x_i^B))_{i=1}^m$, each pairing an input
$x_i^A$ of model A with the corresponding input $x_i^B$ of model B, such
as an image and its caption. On $S$, model $a\in\{A,B\}$ has the Gram
matrix $(K_a)_{ij}=\langle h_a(x_i^a),h_a(x_j^a)\rangle$ of its
preprocessed features, the centered Gram matrix $\bar K_a=HK_aH$ with
$H=I-\ind\ind^\top/m$, and the squared distances $D_a=D(K_a)$ with
$D(K)_{ij}=K_{ii}+K_{jj}-2K_{ij}$. The preprocessing map of our
protocols (Section~\ref{sec:protocol}) normalizes features to unit
length, so Gram matrices have unit diagonal unless stated otherwise.
Let $\Nset_k^a(i)$ be the set of the $k$ inputs $l\ne i$ nearest to
input $i$ under $D_a$, with $1\le k\le m-2$ and a fixed tie-breaking
rule. We write $\Tset_a$ for the edge set of the minimum spanning tree
(MST) of the complete graph on the $m$ inputs weighted by the distances
$d_a=D_a^{\circ1/2}$, i.e., the entrywise square roots of $D_a$. Mutual
$k$-nearest-neighbor overlap (mKNN), MST overlap, and linear centered
kernel alignment (CKA) are
$s_N=(mk)^{-1}\sum_i|\Nset_k^A(i)\cap\Nset_k^B(i)|$,
$s_{\Tset}=|\Tset_A\cap\Tset_B|/(m-1)$, and
$\CKA(K_A,K_B)=\langle\bar K_A,\bar K_B\rangle_F/(\norm{\bar K_A}_F\norm{\bar K_B}_F)$
\citep{huh2024platonic,you2026structure,kornblith2019similarity}, with
the conventions of Appendix~\ref{app:measurement-definitions}.

\subsection{Observables, Protocols, and Criteria}
\label{sec:observables}
\label{sec:protocol}

Each of these scores depends on the two Gram matrices through one kind
of relation among the inputs, such as the neighbor sets for mKNN or the
spanning-tree edges for MST overlap, and an observable maps a Gram
matrix to such a relation.

\begin{definition}[Observable and Information Refinement]
\label{def:refinement}
An \emph{observable} is a measurable map $T:\Kset\to\Yset_T$ on a common
domain $\Kset\subseteq\mathbb S_+^m$ of admissible Gram matrices. For
observables on the same domain, $T_2$ refines $T_1$, written
$T_1\preceq T_2$, if $T_1=q\circ T_2$ for
some measurable map $q:\Yset_{T_2}\to\Yset_{T_1}$, i.e., the value of
$T_2$ determines the value of $T_1$.
\end{definition}

Besides the squared distances $T_G(K)=D(K)$ and the neighbor sets
$T_{N,k}$, Appendix~\ref{app:measurement-definitions} defines four
observables built from graphs on the inputs, i.e., the connected
components $T_{C,k}$ of the $k$-nearest-neighbor graph, the
single-linkage dendrogram $T_{SL}$, the MST edge set $T_{\mathrm{MST}}$,
and the nearest-neighbor map $T_{N,1}$. The remaining observables of this
appendix are the rankings $T_O$ of the other inputs by their distance to each input, the global
order $T_E$ of all distances, and the classes $T_{G/\mathrm{Aff}}$ and
$T_{G/\R_+}$ of $D(K)$ under positive affine maps and rescaling. The Pearson correlation of two distance matrices is a function of their
classes under positive affine maps, and linear CKA is a function of the
classes of the two distance matrices under rescaling.

\begin{proposition}[Two Branches below the Global Order]
\label{prop:lattice}
On Gram matrices whose distances $D(K)_{ij}$, $i<j$, are pairwise
distinct, the following hold:
\textup{(1)}~$T_{C,k}\preceq T_{N,k}\preceq T_O\preceq T_E\preceq
T_{G/\mathrm{Aff}}\preceq T_{G/\R_+}\preceq T_G$;
\textup{(2)}~$T_{\mathrm{MST}}\preceq T_{SL}\preceq T_E$ and
$T_{N,1}\preceq T_{SL}$;
\textup{(3)}~for every $m\ge4$ and $1\le k\le m-2$,
$T_{\mathrm{MST}}\not\preceq T_O$, hence
$T_{\mathrm{MST}}\not\preceq T_{N,k}$, and
$T_{N,k}\not\preceq T_{\mathrm{MST}}$.
\end{proposition}

The third part places neighbor sets and spanning trees on separate branches below
$T_E$ (Appendix~\ref{app:proofs-prelim} proves all three parts).
Since learned representations carry nonzero error, a criterion on an
observable $T$ also fixes a tolerance $\varepsilon\ge0$ on a discrepancy
$\Delta\ge0$ between values of $T$, such as $\Delta_N=1-s_N$ for
neighbor sets, whereas the exact criterion $(T,=)$ takes $\varepsilon=0$
and $\Delta=\indic\{T(K)\ne T(K')\}$. The neighbor count $k$ and the
number $m$ of inputs set the scale of a criterion, i.e., the range of
distances over which the two models are compared. BPC changes this scale
at fixed $k$ through the size $m$ of its evaluation sets. The length
tolerance $\tau$ of SOG sets how strongly a mismatch between the
normalized lengths of a shared edge in the two models lowers the weight
of this edge (Equation~\eqref{eq:sog-weight}).

To compare the two models across layers, we write $K_a^\ell$ for the
Gram matrix of model $a$ at layer $\ell$ and $s_{\ell\ell'}$ for the
score computed on $K_A^\ell$ and $K_B^{\ell'}$.
An evaluation protocol $\Pcal=(\mathrm{pre},\mathrm{agg},\mathrm{cal})$
consists of a preprocessing map $\mathrm{pre}$ of the features of each
layer, an aggregation $\mathrm{agg}$ of the layer-pair scores
$(s_{\ell\ell'})_{\ell,\ell'}$ into one score $s$, and a calibration
$\mathrm{cal}$ of $s$. The calibration uses the scores computed by the
same protocol after a uniformly random permutation $\pi$ of the inputs
of one model. The
\emph{random-pairing baseline} of a score is its expectation under $\pi$.
A protocol acts equivalently on a score $s$ with largest possible value
$s_{\max}$ and on its discrepancy $s_{\max}-s$, so the maximum of the
layer-pair scores corresponds to the minimum of the layer-pair
discrepancies. PRH, ARH, and SOG aggregate by the
maximum $\max_{\ell,\ell'}s_{\ell\ell'}$ over layer pairs
(Table~\ref{tab:claims-practice}), whereas a fixed-layer protocol uses
$s_{\ell_0\ell_0'}$. ARH calibrates an observed score $s_{\rm obs}$ to
$s_{\rm cal}=[s_{\rm obs}-c_\alpha]_+/(s_{\max}-c_\alpha)$ with
$[u]_+=\max\{u,0\}$, where the calibration threshold $c_\alpha$ is a $(1-\alpha)$ quantile of the aggregated scores on
the observed pairing and on the permuted pairings of the inputs
\citep{groger2026revisiting}.

\begin{definition}[Alignment Criterion]
\label{def:criterion}
An alignment criterion $Q=(T,\Delta,\varepsilon;\Pcal)$ consists of an
observable $T$, a discrepancy $\Delta$, a tolerance $\varepsilon\ge0$, and
an evaluation protocol $\Pcal$ with fixed scale parameters and
permutation draws. It accepts a pair of models if $\Pcal$ computes a
discrepancy of at most $\varepsilon$ from the layer-pair discrepancies
$\Delta(T(K_A^\ell),T(K_B^{\ell'}))$. Its acceptance set $\Aset_Q$ is
the set of tuples of layer Gram matrices satisfying this condition, and
we treat criteria with equal acceptance sets as the same criterion.
\end{definition}

We write $(K_A,K_B)\in\Aset_Q$ when the tuple of layer Gram matrices of
the two models lies in $\Aset_Q$. On a common domain, we call $Q_2$ stronger than $Q_1$, written
$Q_1\preceq Q_2$, if $\Aset_{Q_2}\subseteq\Aset_{Q_1}$, i.e., every pair
accepted by $Q_2$ is accepted by $Q_1$. If intersections and unions of acceptance sets also count as criteria,
the criteria on a common domain form a lattice under strength
(Appendix~\ref{app:criteria-supp}).\footnote{Figure~\ref{fig:lattice}
shows the named criteria, a subposet of this lattice; their joins and
meets, such as the criterion requiring agreement on both $T_{N,k}$ and
$T_{\mathrm{MST}}$, complete it. Its arrows show refinement of
observables, and in this order the meet of two observables is their
finest common coarsening. The exact criterion on this meet can accept more
pairs than the meet of the exact criteria on the two observables (the
union of their acceptance sets). In contrast, the join of these two exact
criteria equals the exact criterion on the join of the two observables.} Any two criteria then have a weakest criterion stronger
than both (their join) and a strongest criterion weaker than both (their
meet). Since strength reduces to refinement
for exact criteria with a common protocol on the observables used here,
Proposition~\ref{prop:lattice} describes the exact part of this lattice.
One pair of models can satisfy a criterion on one branch and fail a
criterion on the other. To explain such opposite results on one pair, we
identify the quantities that determine the outcome of each criterion
(Section~\ref{sec:reads}). Such a quantity is a function of the two
representations and fixes the discrepancy of the criterion or the
expectation of this discrepancy over the evaluation samples, exactly or
up to a stated error.

\subsection{Reliable Agreement and the Hypothesis}
\label{sec:learning-law}

The criteria above compare two fixed representations, whereas repeated
training can change the representation map of a model. For each
model, we choose where the representation map takes its features, such
as a specified layer, and leave aggregation over layers to the protocol
of Section~\ref{sec:protocol}. The learning condition
$r_a=(P_a,N_a,\mathcal H_a,\mathcal L_a,\mathsf{Train}_a)$ of model $a$
consists of its data distribution, training sample size, model class,
objective, and training procedure. Repeated draws of the training sample
and of the training randomness induce a law $\mu_{r_a}$ of Gram matrices
on evaluation samples drawn independently of training. With independent draws for the two models, $Q$ is met
with probability
$p_Q(r_A,r_B)=\Prob_{\mu_{r_A}\otimes\mu_{r_B}}((K_A,K_B)\in\Aset_Q)$.
An intervention changes some entries of $(r_A,r_B)$, and
$\mathcal I(r_A,r_B)$ denotes the pairs of conditions reached by one
family of interventions, such as all increases of the training sample
size.

\begin{definition}[Reliable Agreement and Obstruction]
\label{def:reliable}
Fix a level $1-\delta$ with $0<\delta<1$ and a separation
$\eta\in(0,1-\delta)$, i.e., the gap between the thresholds for reliable
agreement and for obstruction. For models trained under $r_A$ and $r_B$, agreement under
$Q$ is \emph{reliable at level $1-\delta$} if
$p_Q(r_A,r_B)\ge1-\delta$. A criterion $Q$ is \emph{obstructed} over
$\mathcal I(r_A,r_B)$ if
$\sup_{(r'_A,r'_B)\in\mathcal I(r_A,r_B)}p_Q(r'_A,r'_B)\le1-\delta-\eta$,
and, for a family of pairs of learning conditions indexed by the
training sample size $N$, \emph{asymptotically obstructed} if $\limsup_{N\to\infty}p_Q\le1-\delta-\eta$.
\end{definition}

\begin{hypothesis}[Lattice Representation Hypothesis]
\label{hyp:lrh}
\textup{(1)}~For two fixed representations with Gram matrices $K_A$ and $K_B$,
the set $\{Q:(K_A,K_B)\in\Aset_Q\}$ of criteria accepting the pair
describes what is shared by the two representations. Studies reach
conflicting conclusions about the same models because they apply
criteria at different points of the lattice. \textup{(2)}~For learning
conditions $(r_A,r_B)$ and a level $1-\delta$, let
$\mathcal R_\delta(r_A,r_B)=\{Q:p_Q(r_A,r_B)\ge1-\delta\}$ be the set of
criteria met reliably across runs. The training procedure, the
objective, the sample size, and the model class each decide this set in
a different part of the lattice.
\end{hypothesis}

\section{What Determines the Outcome of a Criterion}
\label{sec:reads}

To support part~1 of Hypothesis~\ref{hyp:lrh}, we explain why the four
studies reach conflicting conclusions about the same models. Two
representations can differ in many ways, and each criterion detects only
some of these differences. For example, a
nonlinear change of all distances that keeps their order leaves every neighbor set
unchanged and can still lower CKA. Starting from mKNN, we change its
scale, replace it with MST overlap and then with CKA, and finally change
the protocol. The four studies differ in these choices and thereby apply
criteria at different points of the lattice.

\subsection{Neighbor Overlap Is Determined by Boundary Crossings at the Scale \texorpdfstring{$k/m$}{k/m}}
\label{sec:neighbors}

For a neighbor criterion such as mKNN, the neighbor count $k$ and the
number $m$ of inputs set the scale, i.e., the distance from each anchor
to its $k$th nearest input. We ask on which differences between the two
representations the outcome of this criterion depends at each scale. Since the neighbor set $\Nset_k^a(i)$ depends on the distances
from its anchor $i$ alone, we compare the two models through the joint
law of their distances from an anchor to a random input. We assume that
the evaluation inputs are independent draws from a law $P_S$ of
evaluation samples, with an image and its caption drawn together as one
sample in the cross-modal case. We call the joint law of the squared
distances $(D_A(x,x'),D_B(x,x'))$ for $x'\sim P_S$ the cross-model
distance coupling at the anchor $x$.

The coupling records how far apart model B places two inputs found at a
given distance in model A. The two neighbor sets of an anchor differ
when model B carries inputs across the neighbor boundary,
which separates the $k$ nearest inputs of the anchor in model A from the
other inputs. A worst-case margin certifies neighbor agreement at an
anchor $i$ whose gap between the $k$th and $(k+1)$th smallest distances
in model A exceeds $2\max_{j\ne i}|D_A(i,j)-D_B(i,j)|$
\citep[cf.][and Lemma~\ref{lem:margins}]{sumiya2026finite}. Since this
gap shrinks with $m$, the margin certifies a smaller share of anchors on
larger evaluation sets, and to describe the typical anchor we count the
inputs crossing the neighbor boundary. At an anchor $i$ with
$b_l=D_B(i,l)$, a threshold $\vartheta$ is crossed by the
$X(\vartheta)=|\{j\in\Nset_k^A(i):b_j>\vartheta\}|$ neighbors of model A
pushed out by model B and the
$Y(\vartheta)=|\{l\notin\Nset_k^A(i)\cup\{i\}:b_l\le\vartheta\}|$
non-neighbors pulled below the threshold.

\begin{theorem}[Crossing Law]
\label{thm:crossing}
For every anchor $i$ and threshold $\vartheta$, the number
$|\Nset_k^A(i)\setminus\Nset_k^B(i)|$ of replaced neighbors lies between
$X(\vartheta)$ and $Y(\vartheta)$. Let $\mathcal G$ be a $\sigma$-field
with respect to which row $i$ of $D_A$ is measurable. If the entries
$b_l$ with $l\ne i$ are independent with continuous laws given
$\mathcal G$, then there is a $\mathcal G$-measurable threshold
$\vartheta^*$ with
$\E[X(\vartheta^*)\mid\mathcal G]=\E[Y(\vartheta^*)\mid\mathcal G]$. For
this $\vartheta^*$, the conditional expectation of the number of replaced
neighbors given $\mathcal G$ lies almost surely within
$(\kappa_i/2)^{1/2}$ of $\kappa_i=\E[X(\vartheta^*)\mid\mathcal G]$, which
is almost surely the same at every such $\vartheta^*$ and is unchanged when one
deterministic strictly increasing function is applied to every entry of
row $i$ of $D_A$ or of $D_B$.
\end{theorem}

By Theorem~\ref{thm:crossing}, the expected fraction of replaced
neighbors is $\kappa_i/k$ up to an error of order $\sqrt{\kappa_i}/k$.
The loss of neighbor overlap thus counts the neighbors pushed across the
neighbor boundary by model B, a count unchanged by deterministic strictly increasing
maps of either row. Fixed models on independent inputs satisfy the assumption of
the theorem whenever the coupling has continuous conditional laws
(Lemma~\ref{lem:rows} in Appendix~\ref{app:proofs-reads}). Under this
condition, the crossing law holds on pretrained models without a noise
model, and neighbor overlap depends on the coupling through the expected crossings at each
anchor.

With a growing ratio $m/k$, the neighbor boundary moves inward and $\kappa_i$
depends on the coupling at smaller distances. At a fixed $k$, a neighbor criterion on a larger
evaluation set therefore compares the two models at a finer scale. The outcomes
of two neighbor criteria differing in $m$ alone depend on the coupling at different
distances, and one pair of models can satisfy one of these criteria and fail the other. Conjecture~\ref{conj:scale} in Appendix~\ref{app:scale} sorts
the limits of $s_N$ at fixed $k$ into three classes, a classification
whose proof awaits a local Poisson limit. The limit is one when the two
embeddings agree up to a local rescaling of distances, lies in $(0,1)$
for other smooth embeddings, and is zero when model B separates inputs
placed arbitrarily close by model A. In these three classes, the heuristic
first-order law behind the conjecture predicts that $s_N$ at fixed $k$
rises, stays flat, or falls with $m$. At a fixed $k/m$, the same law
predicts a stable $s_N$ for large $m$.

The size trend of mKNN reported by PRH and kept by ARH is a trend in the
crossings at the neighbor boundary, measured in each study at one scale set
by its $k$ and $m$. BPC enlarges the evaluation set at a fixed $k$ and
finds that mKNN between vision and language models collapses unless $k$
grows in proportion to $m$. This fall of mKNN with the evaluation size comes
at least partly from the change to a finer scale, and the collapse is
consistent with the third class of Conjecture~\ref{conj:scale}. On one
million WIT pairs, the neighbor criteria compared by BPC are two points
of one lattice, criteria on the observables $T_{N,10}$ and
$T_{N,m/100}$, the second keeping the scale $k/m$ of PRH. PRH and BPC thus reach different conclusions from
neighbor criteria at different scales. At every scale, neighbor overlap depends on the crossings at each
anchor separately, through the ranking of the other inputs by their
distance from this anchor.

\subsection{Spanning Trees Contain the Nearest-neighbor Forest}
\label{sec:mst}

Since neighbor overlap depends on the crossings at each anchor separately,
we keep the inputs and move to the MST. This tree is built by comparing
distances from different anchors, and the observable $T_{\mathrm{MST}}$
lies on the other branch below the order $T_E$.
By Proposition~\ref{prop:lattice}, the rankings from
every anchor leave this tree undetermined, and we ask whether part of
the MST overlap nevertheless comes from one-nearest-neighbor overlap.
Since Kruskal's algorithm builds the MST from the shortest edges first,
part of each tree is decided at the nearest-neighbor scale, and
Proposition~\ref{prop:mst} names this part.

\begin{proposition}[Spanning Trees Contain the Nearest-neighbor Forest]
\label{prop:mst}
If each model $a$ has pairwise distinct distances $D_a(i,j)$ with $i<j$,
the MST of each model contains its nearest-neighbor forest $F_a$, which
joins each input to its nearest neighbor. The other $n^{\rm mut}_a-1$
edges of the MST are bridges joining two trees of $F_a$, where
$n^{\rm mut}_a$ counts the pairs of mutual nearest neighbors, and
$s_{\Tset}\ge\max\{m\,s_{N,1}-n^{\rm mut}_A,\;m\,s_{N,1}/2\}/(m-1)$ for
the one-nearest-neighbor overlap $s_{N,1}$, i.e., $s_N$ at $k=1$.
\end{proposition}

Proposition~\ref{prop:mst} names a part of each tree fixed by the
nearest-neighbor relations alone, the nearest-neighbor forest, and this
forest makes up at least half of each tree because the mutual pairs are
disjoint. The two branches share the nearest-neighbor map
$T_{N,1}$, which determines this forest and lies below both $T_O$ and
$T_{SL}$. Since each tree contains its forest, the shared MST edges
include the shared nearest-neighbor edges, and the bound on $s_{\Tset}$
keeps MST overlap at least half of the one-nearest-neighbor overlap. MST overlap depends on
structure beyond neighbor sets through the bridges, which carry the
information in the MST beyond the nearest-neighbor scale. The tree also leaves the
neighbor sets undetermined (Proposition~\ref{prop:lattice}), and one
pair of models can satisfy a criterion on spanning trees and fail a
criterion on neighbor sets.

SOG infers converging global structure from the size trend of MST
overlap. By Proposition~\ref{prop:mst}, the part of this trend carried
by the shared nearest-neighbor edges depends on the models through their
nearest-neighbor maps alone, as does the trend of one-nearest-neighbor
overlap. The trend of MST overlap supports this conclusion of SOG to the
extent that it is carried by the bridges. Both branches lie below the order $T_E$ of all
distances, and two representations whose distances differ by a strictly
increasing distortion therefore agree on neighbor sets and spanning
trees.

\subsection{Unlike Neighbor Overlap, CKA Also Depends on a Monotone Link}
\label{sec:geometry}

A strictly increasing distortion of all distances can still change the
value of CKA (Proposition~\ref{prop:bernstein} in
Appendix~\ref{app:proofs-reads}), whose observable $T_{G/\R_+}$ lies
above $T_E$ in the chain of Proposition~\ref{prop:lattice}. To find which quantities determine
the outcome of a criterion above $T_E$ in addition to the order of all
distances, we split the difference into a
monotone distortion of the distances and the remaining disagreement. We
write $\CKA(D,D')$ for the CKA of $HDH$ and $HD'H$,
which equals the CKA of the Gram matrices.
For hollow symmetric $m\times m$ matrices $D_A$ and $D_B$, a monotone
link from $D_A$ to $D_B$ is a nondecreasing function
$\phi:[0,\infty)\to\R$ applied entrywise off the diagonal, the residual
around $\phi$ is the hollow symmetric matrix $R=D_B-\phi(D_A)$, and an
isotonic link minimizes $\norm{R}_F$ over all monotone links.

\begin{proposition}[CKA under a Monotone Link and a Residual]
\label{prop:cka}
For each $m$, let $D_A$ be a fixed nonzero hollow symmetric matrix with
nonnegative entries, $\phi$ increasing, and $D_B=\phi(D_A)+R$ for a
random hollow symmetric residual $R$ whose entries above the diagonal are
independent with mean zero, variance $\sigma^2$, and fourth moments
bounded uniformly in $m$. As $m\to\infty$ with $\sigma^2$ and
$\varsigma^2=\norm{H\phi(D_A)H}_F^2/(m(m-1))$ bounded away from zero and
infinity,
$\CKA(D_A,D_B)-\CKA(D_A,\phi(D_A))/\sqrt{1+\sigma^2/\varsigma^2}\to0$ in
probability.
\end{proposition}

By Proposition~\ref{prop:cka}, when both effects are small,
$1-\CKA\approx\bigl(1-\CKA(D_A,\phi(D_A))\bigr)+\sigma^2/(2\varsigma^2)$,
so CKA depends on the monotone link at zeroth order and on the residual at second
order relative to the global spread $\varsigma$ of the distances.
Neighbor overlap instead ignores the link and depends on the residual through
the crossings at the neighbor boundary (Theorem~\ref{thm:crossing}).

The observable of CKA determines the order of all distances, and exact
agreement on this observable implies agreement on every neighbor set and
spanning tree (Proposition~\ref{prop:lattice}). Two representations
sharing the order of all distances differ by a monotone link with a zero
residual, and CKA depends on this link at zeroth order. Agreement can thus hold on the
order of all distances and fail above this order. Under the
distortion of Proposition~\ref{prop:bernstein}, the two representations
agree exactly on $T_E$ and on every observable below this order,
whereas CKA comes close to zero when the inputs vary along few
effective directions relative to $m$.

The convergence claimed by PRH implies agreement under the exact
criterion on $T_G$, the top of the chain in
Proposition~\ref{prop:lattice}, since its idealized learners agree on all
distances (Appendix~\ref{app:related}).
ARH finds that mKNN keeps its size trend after calibration whereas
global measures such as CKA largely lose this trend, and narrows the
convergence inferred by PRH from the mKNN trend to shared local
neighborhoods. By Propositions~\ref{prop:cka} and~\ref{prop:bernstein},
two models whose representations share the order of all distances can
satisfy an mKNN criterion and fail a CKA criterion through a monotone
link alone. PRH and ARH thus share the
trend of mKNN, whose observable lies below $T_E$, and differ in their
conclusions about agreement above this order.
ARH computes its trends of mKNN and CKA after calibration of a maximum
over layers, and the contrast between these trends can also depend on
this protocol.

\subsection{Calibration of a Layer Maximum Can Reverse a Size Trend}
\label{sec:calibration}

To separate the protocol from the observable, we keep the
observables of mKNN and CKA and the representations unchanged and ask
whether calibrating a maximum over layers can change the size trend of a
score. Since calibration subtracts a permutation quantile of the
score, the random-pairing baseline of each score sets the size of
this subtraction at a single layer pair, and we first compute these
baselines.

\begin{proposition}[Random-pairing Baselines]
\label{prop:baselines}
Let $\pi$ be a uniformly random permutation of the inputs of model B, and
let $\E_\pi$ denote the expectation of a quantity computed after this
permutation, taken over $\pi$ and over any random tie-breaking.
\textup{(1)}~If each row of $D_B$ has distinct off-diagonal entries or
ties are broken uniformly at random independently of $\pi$, then after
this permutation the number of neighbors shared by the two models at each
anchor follows the hypergeometric distribution with population $m-1$,
$k$ successes, and $k$ draws, and $\E_\pi s_N=k/(m-1)$.
\textup{(2)}~For nonzero centered Gram matrices,
$\E_\pi\CKA=(\PR(\bar K_A)\,\PR(\bar K_B))^{1/2}/(m-1)$ with the
participation ratio $\PR(\bar K)=(\tr\bar K)^2/\norm{\bar K}_F^2$.
\textup{(3)}~For $m\ge4$, let $\mathrm{HSIC}_{\rm u}$ be the unbiased
estimator of the Hilbert--Schmidt independence criterion
\citep{song2012feature}. If $\mathrm{HSIC}_{\rm u}(K_A,K_A)$ and
$\mathrm{HSIC}_{\rm u}(K_B,K_B)$ are positive, then
$\E_\pi\CKA_{\rm deb}=0$ and $\E_\pi\CKA_{\rm deb}^2=2/(m(m-3))$ for the
debiased CKA
$\CKA_{\rm deb}(K_A,K_B)=\mathrm{HSIC}_{\rm u}(K_A,K_B)/
(\mathrm{HSIC}_{\rm u}(K_A,K_A)\,\mathrm{HSIC}_{\rm u}(K_B,K_B))^{1/2}$.
\end{proposition}

For independent identically distributed representations, part~1 is
Theorem~C.10 and Proposition~C.11 of \citet{groger2026revisiting}, and
part~2 is classical for the RV coefficient, which coincides with linear
CKA \citep{kazi1995refined,josse2008testing}. Every layer pair has the
mKNN baseline $k/(m-1)$, whereas by part~2 the baseline of CKA grows with
the participation ratio of each layer, i.e., its effective dimension. A
calibration of the layer maximum recomputes this maximum inside each
permutation, and we expect its threshold to lie near the largest
baseline over layer pairs, as assumed in
Corollary~\ref{cor:aggregation}.

\begin{corollary}[Aggregation before Calibration, Informal]
\label{cor:aggregation}
For two model pairs indexed $1$ and $2$, let $\bar s_j$ be the maximum of
the layer-pair scores of pair~$j$, and let the calibration threshold of
$\bar s_j$ lie within $\zeta_j$ of the largest random-pairing baseline
$\bar s^0_j$ over these layer pairs, with $\bar s^0_j+\zeta_j<s_{\max}$.
\textup{(1)}~If $\bar s^0_1=\bar s^0_2$ and
$\bar s_2-\bar s_1>\zeta_1+\zeta_2$, calibration does not reverse the
order of the two maxima.
\textup{(2)}~Let $\bar s_1<\bar s_2$, and let the largest baseline grow
from pair~1 to pair~2 so much faster than the maximum that the calibrated
score of $\bar s_1$ at the threshold $\bar s^0_1+\zeta_1$ exceeds that of
$\bar s_2$ at the threshold $\bar s^0_2-\zeta_2$. Then calibration
reverses this order.
\end{corollary}

By Corollary~\ref{cor:aggregation}, the protocol alone can reverse a
size trend on unchanged representations when the largest baseline grows
faster than the aggregated score in the sense of part~2. With one baseline shared by all layer pairs and model pairs,
mKNN falls under part~1, and calibration does not reverse the order of
model pairs whose mKNN scores differ by more than the permutation
fluctuations. Part~3 of Proposition~\ref{prop:baselines} gives debiased
CKA the baseline zero at every layer pair, which places this score under
part~1. The largest CKA baseline belongs to the layers of largest
effective dimension and can outgrow the aggregated score along a model
family, as the condition of part~2 requires. The corollary and its proof
apply unchanged to calibration at layers fixed in advance, with the
score, baseline, and threshold of the fixed layer pair in place of those
of the maximum. Calibration of each layer pair before the maximum also
subtracts baselines growing with these dimensions but lies outside the
corollary. Appendix~\ref{app:proofs-reads} gives the formal version of
the corollary and its proof.

ARH calibrates a maximum over layer pairs, and when the threshold of
this calibration follows the largest baseline over layer pairs, the loss
of the CKA trend in ARH can come from the protocol. The mKNN trend kept
by ARH after calibration is consistent with part~1 of
Corollary~\ref{cor:aggregation}.

On one pair of fixed representations, whether a criterion accepts the
pair can change with its scale and with its observable, both between the
two branches below the order $T_E$ and across this order. With the
representations unchanged, the protocol of a criterion alone can reverse
a size trend along a model family. The four studies differ in these
choices, since PRH and BPC use mKNN at different scales, while SOG
compares spanning trees and ARH uses CKA after calibrating a maximum
over layers. Their conflicting conclusions about the same models can
therefore come from criteria applied at different points of the lattice,
which supports part~1 of Hypothesis~\ref{hyp:lrh}. This section
compares fixed representations, whereas the representations of
independently trained models vary from run to run, and in Section~\ref{sec:learning} we ask how
their training decides the set of criteria met reliably across runs.

\section{What Determines the Criteria Met Reliably across Runs}
\label{sec:learning}

To support part~2 of Hypothesis~\ref{hyp:lrh}, we show in two models
with explicit minimizers how the training procedure and the objective
decide the set $\mathcal R_\delta$ of criteria met reliably across runs.
As the training samples grow, the representations of two runs can stay
apart for two reasons. The loss can have many global minimizers, each
run reaching one of them, or two objectives can have unique but
different minimizers. In Section~\ref{sec:gauge}, the training procedure
of a linear autoencoder chooses among equivalent global minimizers, and
in Section~\ref{sec:objective}, learners trained with two contrastive
objectives converge to unique but different limits. In each case, we
find the criteria met reliably by the two models. For the other
criteria, we find the change of the learning condition that removes
their obstruction and the change of criterion that avoids this
obstruction. Appendix~\ref{app:learning-full} gives the formal versions
of the theorems below and treats the sample size and the model class, the other two conditions
named in part~2, and Table~\ref{tab:learning} pairs all four conditions
with their criteria.

\subsection{The Training Procedure Selects the Gauge of a Linear Head}
\label{sec:gauge}

Two runs of one learner on the same sample share their learning
condition, yet networks trained from different random seeds
differ measurably in their representations
\citep{mehrer2020individual,ding2021grounding}. A model that predicts
through a linear head computes the same outputs after an invertible
change of its representation compensated in the head, so the task and any
amount of its data fix the representation up to such changes. For a
model computing outputs $Wh(x)$ from a representation $h(x)\in\R^{n_h}$
and a linear head $W$, the models $(h,W)$ and $(Vh,WV^{-1})$ with an
invertible $V\in\R^{n_h\times n_h}$ are \emph{gauge equivalent}, and the
\emph{gauge} of a model is the matrix $V$ relating the model to a fixed
reference model of its class. In a linear autoencoder $\hat x=WEx$ trained
with the squared reconstruction loss, we ask how the training procedure
chooses among these changes and which criteria are then met reliably by
two runs. The two runs are compared on fixed evaluation samples drawn
independently of training, which we call probes.
In the linear autoencoder, the probe Gram matrix $K$ is the Gram matrix
of the features $Ex$ on the probes, without the unit normalization of
Section~\ref{sec:criteria}. Every criterion in
this section compares the two probe Gram matrices of the theorems below,
without further preprocessing and without calibration.

\begin{theorem}[The Gauge of a Linear Head, Informal]
\label{thm:gauge}
Assume the conditions of the formal version. The global minimizers of the
reconstruction loss on one sample are gauge equivalent, and the probe
Gram matrix $K$ depends on the gauge $V$ through $\Omega=V^\top V$.
\textup{(1)}~Gradient flow conserves
$C=EE^\top-W^\top W$ at its initial value $C_0$, and a limit at a global
minimizer has $\mathrm{spec}\,\Omega=f(\mathrm{spec}\,C_0)$ for
$f(\lambda)=(\lambda+\sqrt{\lambda^2+4})/2$.
\textup{(2)}~Two such limits on the same or on different samples from
independent initializations with a density for $C_0$ have $\CKA<1$ almost
surely, and $\Prob(\CKA\ge1-\varepsilon)$ is at most a bound that tends to
zero as $\varepsilon\to0$ and depends only on the probes and the laws of
$C_0$.
\textup{(3)}~With weight decay, the global minimizers on one sample
have the same $\Omega$, and with tied weights $W=E^\top$ and no weight
decay, $\Omega=I$.
\textup{(4)}~On one sample, the minimizers of parts~1
and~3 share the projectors onto the column spaces of $K$ and $HKH$.
\end{theorem}

Under gradient flow, the initialization rather than the data determines
the spectrum of $\Omega$, and the bound of part~2 does not depend on the
training samples. More data therefore leaves CKA criteria of small
tolerance obstructed whenever the hypotheses of part~2 hold almost
surely. As the samples grow, weight decay and tied weights remove this
obstruction by a change of the learning condition (part~5 of the formal
version in Appendix~\ref{app:learning-gauge}). By the same part, two runs
under gradient flow that converge to global minimizers come to meet the
projector criterion of part~4 reliably at every positive tolerance as the
samples grow. A change from the CKA criteria to this criterion thus avoids
the obstruction. For large samples in this model, the
training procedure rather than the sample size decides whether CKA
criteria of small tolerance belong to the set $\mathcal R_\delta$ of
Hypothesis~\ref{hyp:lrh}.

Part~2 concerns CKA and leaves open whether two runs under gradient flow
still share the neighbor sets compared by mKNN, which lie lower in the
lattice. Since a gauge stretches distances by bounded factors, neighbor
sets with large relative gaps survive the gauge.

\begin{corollary}[Neighbor Sets under the Gauge]
\label{cor:gauge-neighbors}
Let two global minimizers of Theorem~\ref{thm:gauge} on one sample have
probe Gram matrices $K_A$ and $K_B$ and matrices $\Omega_A$ and
$\Omega_B$, and let $\mathrm{cond}_{AB}$ be the condition number of
$\Omega_A^{-1/2}\Omega_B\Omega_A^{-1/2}$. The two minimizers share the
$k$-neighbor set of every probe $i$ with
$D_{A,i,(k+1)}>\mathrm{cond}_{AB}\,D_{A,i,(k)}>0$, where $D_{A,i,(k)}$ is
the $k$th smallest off-diagonal entry of row $i$ of $D_A=D(K_A)$. If both
minimize the same loss of part~3 of Theorem~\ref{thm:gauge}, then
$\mathrm{cond}_{AB}=1$.
\end{corollary}

On a manifold of moderate dimension, the ratio
$D_{A,i,(k+1)}/D_{A,i,(k)}$ typically exceeds one by a few percent at
$k=10$. The corollary therefore covers the neighbor sets of typical
probes when $\mathrm{cond}_{AB}$ lies within a few percent of one, as
under weight decay, tied weights, or a small initialization under
gradient flow (Appendix~\ref{app:learning-gauge}). Networks are trained with finite steps and often with momentum rather
than by gradient flow, and we record which parts of
Theorem~\ref{thm:gauge} hold under such optimizers.\footnote{We
formulated Corollary~\ref{cor:gauge-neighbors} and
Remark~\ref{rem:optimizers} after the first network experiment of
Section~\ref{sec:exp-gauge} and tested both in a second, preregistered
experiment.}

\begin{remark}[Optimizers and the Selection of the Gauge]
\label{rem:optimizers}
Parts~1 and~2 of Theorem~\ref{thm:gauge} rest on the conservation of $C$
by gradient flow, and finite steps or momentum change $C$ during
training. Part~3 concerns the global minimizers of the regularized loss,
so weight decay selects the gauge under every optimizer that reaches a
global minimizer.
\end{remark}

\subsection{The Objective Decides Whether Contrastive Learners Agree beyond the Order of Distances}
\label{sec:objective}

Unlike two runs of Theorem~\ref{thm:gauge} on one sample, whose global
minimizers differ by a gauge, learners trained with different objectives can have
unique yet different limits. Two such learners can learn one relation through different increasing
functions, and the outcome of a criterion comparing their representations then depends on
the difference between these functions rather than on a difference between
relations. In ARH, this possibility appears as the observation that
different objectives and architectures induce different distance scales
while preserving the relative order of neighbors. We ask whether two
learners trained with different objectives produce this pattern while
both learn one relation, and which criteria the two learners then meet
reliably as their samples grow.

We answer both questions in the finite Gibbs model of
Appendix~\ref{app:gibbs}, a finite contrastive model whose population
limit is the idealized contrastive learner of PRH \citep{huh2024platonic}.
In this model, the positive partner $x'$ of an anchor $x$ in a finite
domain $\Xcal$ with positive masses $P$ is drawn with probability
proportional to $P(x')\exp(K_\star(x,x')/t)$ for a unit-diagonal positive
definite matrix $K_\star$. The masses are balanced if the law of $x'$
given $x$ has the same normalizing constant for every anchor $x$.
Learner A minimizes the empirical InfoNCE risk at temperature $t$, and
learner B minimizes the empirical spectral contrastive risk of
\citet{haochen2021provable}. We write $\hat K_A$ and $\hat K_B$ for their
Gram matrices normalized to unit diagonal, and $D_a=D(\hat K_a)$ and
$D_\star=D(K_\star)$ for the squared distances on probes in $\Xcal$.
For a continuous, positive, positively homogeneous function $\psi$ of
distance matrices, such as the mean off-diagonal distance, a criterion on
scale-normalized geometry has the discrepancy
$\norm{D_A/\psi(D_A)-D_B/\psi(D_B)}_{\max}$ in the entrywise maximum norm.

\begin{theorem}[Objectives with Different Monotone Links, Informal]
\label{thm:link}
Consider the finite Gibbs model with balanced masses under the conditions
of the formal version, and let both samples grow.
\textup{(1)}~$\hat K_A\to K_\star$ and
$\hat K_B\to\exp\circ((K_\star-\ind\ind^\top)/t)$ in probability, so
$D_B\to g_t(D_\star)$ for the increasing concave
$g_t(D)=2(1-e^{-D/(2t)})$. \textup{(2)}~Exact agreement on every
$T\preceq T_E$ becomes reliable, whereas criteria on scale-normalized
geometry are asymptotically obstructed at every tolerance below the
limit $\varepsilon_\psi>0$ of their discrepancy, and CKA tends in
probability to a limit below one. \textup{(3)}~Agreement at every positive
tolerance becomes reliable under the isotonic discrepancy
$\inf_g\norm{g(D_A)-D_B}_{\max}$ over nondecreasing $g$, and so does
agreement of two learners trained with the same objective under every
continuous discrepancy of their distance matrices that vanishes on equal
matrices.
\end{theorem}

Theorem~\ref{thm:link} proves that the pattern invoked by ARH holds for
two contrastive objectives in the finite Gibbs model. Exact agreement becomes
reliable on every observable $T\preceq T_E$, such as the neighbor sets
of PRH and BPC and the spanning trees of SOG, whereas the CKA used by
ARH tends to a limit below one. The two limits differ by
a monotone link with a zero residual in the sense of
Section~\ref{sec:geometry}, so a change of objective alone produces this
pattern while both learners learn one relation. One objective for both learners
removes the obstruction of criteria on scale-normalized geometry by a
change of the learning condition. As both samples grow, two learners
trained with different objectives reliably meet the criteria with the
isotonic discrepancy of part~3 at every positive tolerance, and a change
to these criteria avoids the obstruction. For large samples in this model, the objective
rather than the sample size decides whether criteria on scale-normalized
geometry with a small tolerance belong to $\mathcal R_\delta$. Since the temperature is a parameter of the InfoNCE objective, we also
ask whether a change of temperature alters which criteria are met by two
learners, as does the change of objective in Theorem~\ref{thm:link}.\footnote{We formulated the last
two sentences of Remark~\ref{rem:temperature} after training InfoNCE
networks at two temperatures (Appendix~\ref{app:exp-s21}).}

\begin{remark}[Temperature]
\label{rem:temperature}
Learner A of Theorem~\ref{thm:link} trained at a temperature $t_A$ in
place of $t$ converges in probability to
$(t_A/t)K_\star+(1-t_A/t)\ind\ind^\top$ whenever this matrix is positive
semidefinite, as it is for every $t_A\le t$. The squared distances
$(t_A/t)D_\star$ of this limit are a rescaling of those of $K_\star$, so
every criterion on scale-normalized geometry in Theorem~\ref{thm:link}
accepts the pair of limits at $t_A$ and at $t$. On a fixed sample, a lower
temperature rescales a unique positive definite minimizer of the
empirical risk in the same way. With other images as negatives, this risk
has no positive definite minimizer at any temperature if two augmented
views of one image are each other's only positive partners.
\end{remark}

By Remark~\ref{rem:temperature}, every criterion on scale-normalized
geometry accepts the limits of InfoNCE learners at $t$ and at any lower
temperature, whereas
the change of objective in Theorem~\ref{thm:link} bends the distances by
the concave link $g_t$. Under gradient flow, two runs of the linear
autoencoder on one sample differ by a gauge, which leaves the neighbor set of a probe
unchanged when its relative gap is large
(Corollary~\ref{cor:gauge-neighbors}). As both samples grow, two learners
trained with different objectives instead converge to limits related by
an increasing link, which leaves every observable at or below $T_E$
unchanged. The training procedure and the objective thus decide
$\mathcal R_\delta$ in different parts of the lattice, which supports
part~2 of Hypothesis~\ref{hyp:lrh} for these two conditions in the linear
autoencoder and in the finite Gibbs model. Trained networks lie outside
the premises of both theorems, and Section~\ref{sec:experiments} tests
their predictions inside these models and in networks trained on
CIFAR-5m \citep{nakkiran2021deep}.

\section{Experiments}
\label{sec:experiments}

Our experiments test the explanations of Sections~\ref{sec:reads}
and~\ref{sec:learning}, first on the pretrained models compared by PRH,
ARH, BPC, and SOG and then under controlled training. On the pretrained
models, we test whether the scale, the observable, and the protocol of
the criteria applied by the four studies account for their conflicting
conclusions, as part~1 of Hypothesis~\ref{hyp:lrh} states. For example,
the threshold of the calibrated layer maximum used by ARH follows the
largest random-pairing baseline, so the protocol of ARH can account for
the loss of the CKA trend that separates its conclusion from that of
PRH. These tests follow the order of Section~\ref{sec:reads}, from mKNN
on growing evaluation sets (Section~\ref{sec:exp-scale}) through
spanning trees and CKA (Section~\ref{sec:exp-criteria}) to the
calibration of a layer maximum (Section~\ref{sec:exp-aggregation}).
Under controlled training, we test the predictions of the two theorems
of Section~\ref{sec:learning} for part~2, on the gauge selected by the
training procedure (Section~\ref{sec:exp-gauge}) and on the monotone
link set by the objective (Section~\ref{sec:exp-link}).

\subsection{Experimental Setup}
\label{sec:exp-protocol}

We compare the 17 vision transformers of ARH with the BLOOMZ, OpenLLaMA, and LLaMA-1 language models on the 1{,}024 image--caption pairs of WIT used by PRH. Each score is computed from features at mid depth or at the last layer,
or is the maximum of the layer-pair scores as in PRH, ARH, and SOG. The size
trend of a score is the mean over vision models of its Spearman
correlation with language-model size.
We call an experiment preregistered if we fixed and recorded its predictions,
evaluation settings, and decision rules before running it, which keeps a confirmed prediction from resting on
choices made after seeing the results. To test part~2 of Hypothesis~\ref{hyp:lrh}, we train linear autoencoders and CIFAR-5m networks, and in the finite Gibbs model we compute the minimizers of the empirical InfoNCE and spectral contrastive risks (Appendix~\ref{app:exp}). In controlled training, agreement
under a criterion is reliable if the one-sided $95\%$
Clopper--Pearson lower bound on the fraction of disjoint pairs of runs
within its tolerance fixed in advance exceeds $0.9$, and unreliable if the upper
bound falls below $0.8$ (Definition~\ref{def:reliable} with
$\delta=\eta=0.1$). Appendix~\ref{app:exp} gives the full protocol, and
Table~\ref{tab:exp-verdicts} lists every preregistered prediction with its
outcome.

\subsection{Neighbor Overlap at a Fixed \texorpdfstring{$k$}{k} Stays Flat or Falls with the Evaluation Size According to the Model Pair}
\label{sec:exp-scale}

We first test the conclusion of Section~\ref{sec:neighbors} that a
neighbor criterion at a fixed $k$ compares two models at a finer scale on
a larger evaluation set. This change of scale explains at least partly
why mKNN between vision and language models collapses on the larger
evaluation sets of BPC unless $k$ grows in proportion to $m$
(Theorem~\ref{thm:crossing}). The
heuristic first-order law behind Conjecture~\ref{conj:scale} predicts
that $s_N$ at a fixed $k$ rises, stays flat, or falls with $m$ according
to the class of the model pair. The same law also predicts a stable
$s_N$ at a fixed $k/m$ for large $m$. We test both predictions on COCO evaluation sets of $m=256$ to $16{,}384$ images, each contained in every larger set. In a preregistered experiment, we measure mKNN on 4 pairs of Pythia-410M models trained from different seeds, 10 pairs of vision models, and 15 vision--language pairs. In a second preregistered experiment, we pair the vision models with three sentence encoders instead of the language models. At $k=10$, the median mKNN of the vision--language pairs falls from $0.091$ to $0.0085$ (Figure~\ref{fig:exp-neighbors}a).
The slope of $\ln s_N$ on $\ln m$ is set by the model pair, with
medians of $-0.57$ for vision--language pairs, $-0.21$ for pairs of vision models,
and $-0.01$ for pairs of Pythia-410M models, against $-1.00$ for chance. The rising class was absent, but our slope rule also classifies a synthetic coupling of this class as flat, which leaves the prediction for this class untested (Appendix~\ref{app:exp-s11}).

For the second prediction, the medians for $k=100$ lie within $4\%$ of those for $k=10$ at a shared $k/m$, which BPC holds fixed in one of its neighbor criteria by setting $k=m/100$, whereas those for $k=1$ differ by up to $19\%$ (Figure~\ref{fig:exp-neighbors}b). With sentence encoders, mKNN also falls with $m$ on all 15 pairs, but two preregistered predictions linking the slope to the level of mKNN remained undecided. These
pairs pass the preregistered test of stability at a shared $k/m$ for
$k=10$ and $k=100$.\footnote{We fixed both values and the tolerance after
the results above and before any sentence-encoder score.} The preregistered out-of-sample test of the crossing law needed an estimator of the coupling, and every candidate estimator failed the preliminary check against two baseline predictors on synthetic couplings. Since this test was not run, we report the results in Figure~\ref{fig:exp-neighbors}a and~b as measurements, not as outcomes of preregistered tests. The exceptions are the outcomes of two preregistered secondary predictions on the ordering of the slopes and the cancellation of the elasticities in $k$ and $m$, i.e., the coefficients of a fit of $\ln s_N$ on $\ln k$ and $\ln m$. For the 29 pairs without sentence encoders, the slopes kept the predicted order, and the medians of the summed elasticities for the three types of pairs lay between $0.03$ and $0.06$, near the cancellation preregistered without a tolerance.

\begin{figure}[t]
\centering
\includegraphics[width=5.5in]{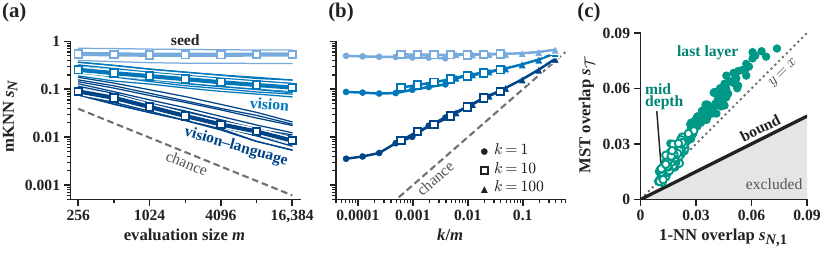}
\caption{\textbf{At a fixed $k$, neighbor overlap stays flat or falls with the evaluation size according to the model pair, and spanning-tree overlap tracks one-nearest-neighbor overlap (Theorem~\ref{thm:crossing}, Conjecture~\ref{conj:scale}, Proposition~\ref{prop:mst}).} (a) mKNN against $m$, for single pairs (thin) and the median of each type of pair (thick). (b) These medians at three values of $k$ against $k/m$, with the vertical axis of (a). (c) MST overlap against one-nearest-neighbor overlap for vision--language model pairs at two depths. Appendices~\ref{app:exp-s11} and~\ref{app:exp-s12} give the details.}
\label{fig:exp-neighbors}
\end{figure}

\begin{insight}[label=ins:scale]{Evaluation size sets the scale of mKNN}
Our results suggest that the evaluation size sets the scale at which mKNN with a fixed $k$ compares two models, and that scores at a shared $k/m$ are comparable across evaluation sets for $k=10$ and $k=100$. With sentence encoders in place of language models, a preregistered test confirms that mKNN stays stable at a shared $k/m$ for both values of $k$. These findings support the conclusion of Section~\ref{sec:neighbors} that the collapse of mKNN reported by BPC comes at least partly from a change of scale.
\end{insight}

\subsection{Shared Spanning-tree Edges Are Mostly Nearest-neighbor Edges, and Pretrained Vision and Language Representations Differ by More than a Monotone Link}
\label{sec:exp-criteria}

Neighbor overlap depends on the distances of two models at a scale set
by $k/m$, whereas MST overlap and CKA also depend on distances beyond the
nearest-neighbor scale. SOG challenges the local view of ARH with the
size trend of MST overlap. Since the MST
contains the nearest-neighbor forest for pairwise distinct distances
(Proposition~\ref{prop:mst}), we ask how many shared MST edges lie in
this forest and whether the overlap of the bridges between its trees has
a size trend. To answer both questions, we measure MST overlap and bridge overlap on WIT for the 136 pairs of a vision model and a language model and classify MST overlap with a preregistered rule. On a larger evaluation set of $m=2047$ inputs, MST overlap tracks one-nearest-neighbor overlap at both depths and exceeds it by a median factor of $1.20$ and $1.32$ (Figure~\ref{fig:exp-neighbors}c). The
nearest-neighbor forest makes up a median $0.875$ and $0.889$ of the
shared MST edges at mid depth and the last layer, near its $84$ to $85\%$
share of each tree (Figure~\ref{fig:exp-mst}b). The preregistered rule
classifies MST overlap as mostly one-nearest-neighbor overlap at
both depths, although its conditions on the forest share and on
$s_{\Tset}-s_{N,1}$ hold nearly automatically. Bridge overlap exceeds its permutation null in
$93\%$ and $100\%$ of pairs, but its BLOOMZ size trend at the last
layer stays unresolved (trend $0.57$, Holm-adjusted $p=0.23$).

By Proposition~\ref{prop:cka}, the CKA used by ARH depends on both a monotone
link between the distances of the two models and the residual around
this link. To see whether the pretrained vision and language representations differ by a monotone link or by a residual, we fit increasingly flexible maps between their squared distances, as a preregistered check rather than a confirmatory test. Beyond an affine fit, a global monotone fit explains a further median $0.0004$ and $0.0012$ of the squared-distance discrepancy, and this share lies below the preregistered level of $0.10$ in all 136 pairs (Figure~\ref{fig:exp-criteria}a).

The conclusion of ARH concerns size
trends, and at a fixed depth a size trend of CKA or mKNN can already
depend on the preprocessing, before any maximum over layers. At mid depth
and after subtraction of the permutation mean, the BLOOMZ trends of CKA
and mKNN are $0.70$ and
$0.84$ with the preprocessing of PRH (Figure~\ref{fig:exp-criteria}b,c).
Under plain $\ell_2$ normalization, our preregistered secondary
preprocessing, they reverse to $-1.00$ and $-0.92$.

\begin{figure}[t]
\centering
\includegraphics[width=5.5in]{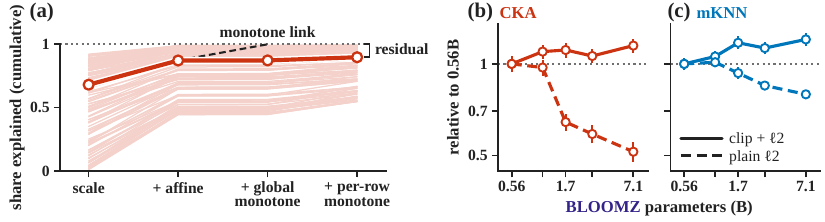}
\caption{\textbf{Beyond an affine map, pretrained vision and language representations differ by a residual rather than a monotone link, and preprocessing can reverse size trends (Proposition~\ref{prop:cka}).} (a)~Share of the squared-distance discrepancy at mid depth explained by increasingly flexible fits, per pair (thin), median (thick), and under a monotone link as in Theorem~\ref{thm:link} (dashed). (b,~c)~CKA and mKNN after subtraction of the permutation mean, as mean $\pm$ s.e.\ over vision models, where clip~$+$~$\ell_2$ is the preprocessing of PRH. Appendix~\ref{app:exp-s12} gives the details.}
\label{fig:exp-criteria}
\end{figure}

\begin{insight}[label=ins:criteria]{Spanning-tree overlap comes mostly from the nearest-neighbor forest, and a monotone link explains little beyond an affine map}
Our results suggest that on these pretrained pairs, spanning-tree overlap comes mostly from the nearest-neighbor forest of Proposition~\ref{prop:mst}, and that the part of the discrepancy between the squared distances of the two models left by an affine map comes mostly from the residual of Proposition~\ref{prop:cka} rather than from a monotone link.
\end{insight}

\subsection{The Threshold of a Calibrated Layer Maximum Can Be Forecast from Single-model Features}
\label{sec:exp-aggregation}

Besides the preprocessing at a fixed depth, a size trend can depend on
the maximum over layer pairs and on its calibration. ARH calibrates each
score after taking its maximum over layer pairs, and
finds that calibration largely removes the size trend of CKA and keeps
that of mKNN. In Section~\ref{sec:calibration} we relate this difference to the
random-pairing baselines of Proposition~\ref{prop:baselines}, which are
equal across layer pairs for mKNN and grow with the effective dimension
of each layer for CKA. Corollary~\ref{cor:aggregation} assumes that the
calibration threshold of the maximum lies near the largest baseline.
Calibration then does not reverse the order of two model pairs with
equal largest baselines and well-separated maxima, and reverses this
order when the largest baseline grows faster than the maximum in the
sense of part~2 of this corollary. We test this assumption and both
parts of the corollary on the layer maximum along four families,
BLOOMZ (0.56B to 7.1B), LLaMA-1 (7B to 30B), OpenLLaMA (3B to 13B), and
Pythia (70M to 6.9B).\footnote{The preregistered BLOOMZ experiment applied the layer-maximum protocol of ARH, and we obtained its result before deriving Corollary~\ref{cor:aggregation}.} For the last three families, we fixed and recorded predictions from single-model features before computing any score involving their new models. On LLaMA-1 and OpenLLaMA, we compared each outcome with its prediction without a confirmatory test, whereas on Pythia we preregistered a confirmatory test of the loss of the CKA trend.

In all four families the threshold of the layer maximum follows the
largest random-pairing baseline of Proposition~\ref{prop:baselines}, in a
post hoc analysis on BLOOMZ and as forecast from single-model features on
LLaMA-1, OpenLLaMA, and Pythia. Part~2 of Corollary~\ref{cor:aggregation} gives a condition under which
calibration reverses the order of two model pairs, and this condition
requires the largest baseline to grow faster than the maximum. The trend is lost or reversed along BLOOMZ ($0.89$ to $-0.02$ under $s_{\rm cal}$ of Section~\ref{sec:protocol}), LLaMA-1 ($0.68$ to $-0.85$), and OpenLLaMA ($0.88$ to $0.18$), where this baseline grows at least as fast as the raw score (Figure~\ref{fig:exp-aggregation}). It is kept along Pythia ($0.93$ to $0.89$),
where the raw score grows faster, by $0.024$ per doubling of parameters
against $0.0063$ for the threshold (Appendix~\ref{app:exp-pythia}).
Along LLaMA-1, calibration produced the drop of at least $0.10$
predicted from single-model features, whereas along Pythia the preregistered
prediction of such a drop was not confirmed. Given the raw scores, the thresholds forecast from
features nevertheless yield a predicted order after calibration in 30 of
the 34 comparisons involving OpenLLaMA 13B and in 352 of the 357 Pythia
comparisons. Every predicted order is correct. In every family, the trends of mKNN and debiased CKA move by
$0.06$ or less under calibration. With the preprocessing of PRH, the
calibrated CKA trend keeps its sign at mid depth and at the last layer in all four families. On LLaMA-1, calibration of each layer pair before the maximum also reversed the CKA trend of the maximum, to $-0.88$, although we had predicted in advance that the trend would keep its sign under this protocol. With three sizes each, every exact $p$-value of a size trend along LLaMA-1 or OpenLLaMA is at least $1/6$, and size, depth, and width grow together along these families (Appendix~\ref{app:exp-layermax}).

\begin{figure}[t]
\centering
\includegraphics[width=5.5in]{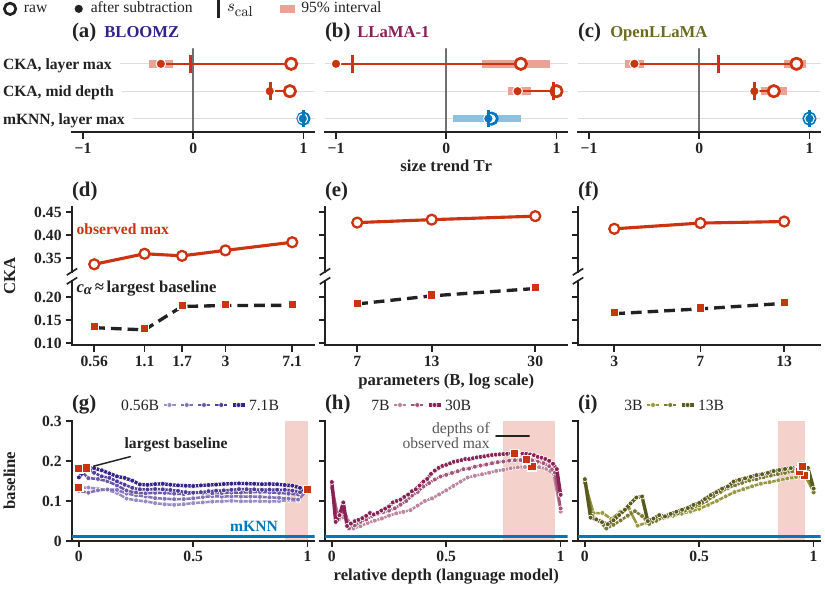}
\caption{\textbf{The threshold of a layer maximum follows the largest baseline, and calibration removes or reverses the CKA trend in the families where this baseline grows at least as fast as raw CKA (Proposition~\ref{prop:baselines}, Corollary~\ref{cor:aggregation}).} (a--c) Size trends of three scores, raw, after subtraction of the permutation mean, and calibrated ($s_{\rm cal}$). (d--f) Squares give the threshold $c_\alpha$ and dashed lines the largest baseline predicted from features. (g--i) Random-pairing CKA baseline of each language-model layer, predicted (lines) and measured (dots). Figure~\ref{fig:app-aggregation} gives the full version.}
\label{fig:exp-aggregation}
\end{figure}

\begin{insight}[label=ins:aggregation]{The calibration threshold of a layer maximum follows the largest baseline}
Our results suggest that the calibration threshold of a layer maximum follows the largest random-pairing baseline of Proposition~\ref{prop:baselines}, which supports the assumption of Corollary~\ref{cor:aggregation}. Along the four families, calibration
kept the CKA trend where the raw score outgrew this baseline and removed
or reversed the trend elsewhere, in line with this corollary. With the
preprocessing of PRH, calibrated CKA at mid depth and at the last layer
keeps the sign of its size trend in every family, so global convergence
remains open.
\end{insight}

\subsection{The Training Procedure Decides Whether Two Runs of a Linear Autoencoder Agree on Distances}
\label{sec:exp-gauge}

The models compared so far are fixed, whereas part~2 of
Hypothesis~\ref{hyp:lrh} concerns the criteria met reliably across
independent training runs. For the training procedure,
Theorem~\ref{thm:gauge} predicts that two gradient-flow runs of a linear
autoencoder keep the gauges set by their initializations and disagree in
geometry at every sample size. Part~2 certifies the resulting obstruction
of CKA criteria only below a tolerance that it leaves unspecified, so our
tolerance of $0.01$ is a choice of the preregistration rather than a
prediction of the theorem. Weight decay and tied weights instead fix
one gauge, and by part~1 a small initialization keeps $\Omega$ close to
the identity in both runs. The global minimizers of parts~1 and~3 on one
sample all share the projectors of part~4. On independent samples, the
projectors of two global minimizers converge in probability to one limit
as the samples grow (part~5 of the formal version of Theorem~\ref{thm:gauge}). We test these predictions in a
preregistered experiment on the linear autoencoder at sample sizes from
$N=2^8$ to the population limit. Each run draws its own training sample
at a finite $N$. Under gradient descent without a regularizer, the
criteria on neighbor sets, mean-normalized squared distances ($G_s$),
and CKA each accepted 0 of 48 pairs of runs at every finite $N$ and 0 of
500 at $N=\infty$. With weight decay, tied weights, or small initialization, agreement under these criteria was reliable for $N\ge2^{12}$, and agreement under the projector criterion of part~4 was reliable at these sizes for every training procedure, as preregistered for $N=2^{16}$ (Figure~\ref{fig:exp-gauge-case}a,b). Outside the preregistered tests, the condition number of the learned $EE^\top$ matches the values predicted by parts~1 and~3 of the formal version of Theorem~\ref{thm:gauge} (Figure~\ref{fig:exp-gauge-case}c). Its medians are $4.66$
against $4.67$ predicted under a small step close to gradient flow and
$1.0330$ against $1.0329$ under weight decay. The preregistered prediction for gradient descent also required the mean discrepancy over pairs of runs at $N=2^{16}$ to lie inside the $95\%$ interval of this mean at $N=\infty$, and this part of the prediction failed on $G_s$ and CKA.

To test how the training procedure selects the gauge outside the
premises of the theorem, we train CIFAR-5m networks with a linear head. The head maps its input features $h$ to coordinates $u=Mh\in\R^9$ and outputs $y=Wu$, and the task fixes $u$ up to a gauge.
In two preregistered network experiments, we train this head by momentum without weight decay, with weight decay, or from a small initialization, on $N=25{,}000$ and $250{,}000$ images. The runs form 96 pairs per training procedure and size, except 32 pairs under momentum from a small initialization in the second experiment. To
separate the gauge kept by gradient flow from the change of $C$ under
momentum (Remark~\ref{rem:optimizers}), the second experiment adds
gradient flow of the head (head flow). For each run, a balanced head
($C=MM^\top-W^\top W=0$) computes the same outputs, so we split the excess
of the discrepancy of $u$ over that of $y$ into two parts. Measured in standard deviations (SD) of the discrepancy of $y$ over pairs, these parts are the same excess computed for the coordinates of the balanced head (the balanced part) and the rest (the imbalance part).
Under head flow, the imbalance part added $0.54$ and $0.59$ SD to the
$G_s$ discrepancy at the two sizes, about one fifth of the imbalance part
under momentum (Figure~\ref{fig:exp-gauge-case}d,e). Of this part,
$0.29$ and $0.46$ SD varied from run to run, so the network realizes the
mechanism of part~2 of Theorem~\ref{thm:gauge} under head flow. From a balanced initialization with $C_0=0$ or a small initialization, momentum created a larger imbalance part than head flow on $G_s$ and $1-\CKA$ in 8 of 8 Holm-corrected tests. With weight decay, $u$ agreed in 96 of 96 pairs on $G_s$ and on neighbor sets, at $250{,}000$ images in the first experiment and at both sizes in the second, without an excess on $G_s$ ($-0.04$ SD at $250{,}000$). The preregistered predictions of Corollary~\ref{cor:gauge-neighbors} held, and at $k=10$ the bound of the corollary certified neighbor agreement for at most $4.7\%$ of the anchors.

Without weight decay, $u$ agreed on neighbor sets in 96 of 96 pairs at $250{,}000$ images, contrary to the preregistered prediction that this agreement would be unreliable at both sizes.
In the second experiment at this size, the mean neighbor-set discrepancy of $u$ lay about $0.002$ below that of $y$ under four training procedures, against the preregistered prediction that the $95\%$ interval of this difference would lie within $\pm0.002$. Under momentum, small initialization fell short of weight decay, against the preregistered predictions that it would make agreement of $u$ reliable at $250{,}000$ images and bring the fraction of agreeing pairs within $0.1$ of that under weight decay. Under head flow, a balanced initialization with $C_0=0$ or a small initialization left a balanced part of about $3$ SD, so the preregistered prediction that agreement of $u$ on $G_s$ would be reliable at both sizes failed under both initializations. This agreement reached the preregistered level of reliability in one of the four combinations of initialization and size, small initialization at $250{,}000$ images (Appendices~\ref{app:exp-s22} and~\ref{app:exp-s22-v11}).

\begin{figure}[t]
\centering
\includegraphics[width=5.5in]{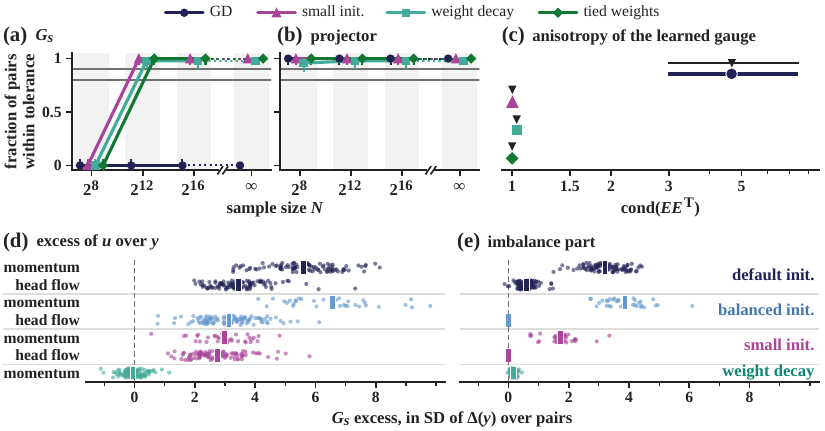}
\caption{\textbf{Gradient flow keeps the gauge set by the initial imbalance of a linear head at every sample size, whereas weight decay fixes one gauge (Theorem~\ref{thm:gauge}, Remark~\ref{rem:optimizers}).} (a--c) Linear autoencoder, with gradient descent (GD) close to gradient flow. Bars in (a, b) are Clopper--Pearson bounds and grey lines the levels for reliable and unreliable agreement. (c) Medians and 10th to 90th percentiles at $N=\infty$, with the predictions in black. (d, e) CIFAR-5m networks, with one dot per pair of runs and bars at the means. }
\label{fig:exp-gauge-case}
\end{figure}

\begin{insight}[label=ins:gauge]{The training procedure selects the gauge of a linear head}
Our results confirm the preregistered predictions of Theorem~\ref{thm:gauge} that weight decay and tied weights fix one gauge of the linear autoencoder and that a small initialization keeps $\Omega$ close to the identity in both runs. They also suggest that gradient descent at a small step keeps the gauge set by the initialization, which supports part~1 of the theorem. In trained networks, the imbalance kept by head flow and the imbalance created by momentum both add an excess to the geometric discrepancy of $u$ between two runs, and weight decay removes this excess on $G_s$, which supports Remark~\ref{rem:optimizers}.
\end{insight}

\subsection{Agreement of InfoNCE and Spectral Learners on the Order of Distances Becomes Reliable in the Finite Gibbs Model}
\label{sec:exp-link}

In the linear autoencoder, the training procedure decides whether agreement of two runs on distances becomes reliable for large samples. For the objective, Theorem~\ref{thm:link} states that a learner minimizing InfoNCE and a learner minimizing the spectral contrastive loss learn one relation through different increasing links in the finite Gibbs model. As both samples grow, exact agreement of the two learners on
the order of distances becomes reliable, whereas a criterion on
scale-normalized geometry with a small tolerance is asymptotically
obstructed. We test these predictions in two
preregistered experiments, one inside the premises of the theorem and
one in CIFAR-5m networks outside these premises. In the finite Gibbs
model, both learners return certified minimizers of their exact
empirical risks, in 192 runs per learner at each of eight sample sizes
from $2^{20}$ to $2^{40}$ examples. Exact agreement
across objectives on the order of all distances among eight probe inputs
held in 0 of 96 pairs up to $2^{28}$ examples and in 94 of 96 at
$2^{34}$. The size $2^{34}$ lies near the value $N^*=2^{33.75}$ predicted
before the runs for agreement in $90\%$ of pairs
(Figure~\ref{fig:exp-link}a). At $2^{40}$ examples, the
largest difference of mean-normalized distances across objectives on all
256 inputs settled at a median $0.348$, the limit given by the theorem
with these inputs as the probes
(Figure~\ref{fig:exp-link}b,c). Under the scale-normalized
discrepancy, all 96 pairs stayed outside the tolerance, whereas under the
isotonic discrepancy every pair fell within this tolerance. After Holm correction, all nine preregistered confirmatory predictions held, but two secondary Monte Carlo predictions failed, those of the ratio of the discrepancy across objectives to the mean discrepancy within an objective on neighbor sets and on spanning trees (Appendix~\ref{app:exp-s23}).

This experiment tests
Theorem~\ref{thm:link} and our analysis pipeline inside the premises of
the theorem, not trained networks, and the sample size it needs for exact
agreement on the global order (about $2^{34}$ examples here) reflects how
fine that criterion is rather than a practical training budget. The
preregistered prediction failed in CIFAR-5m networks outside these
premises, where at $500{,}000$ images all 64 pairs of InfoNCE and
spectral runs disagree even on neighbor sets and a fitted monotone link
removes a median $15\%$ of the residual left by an affine map, against
$100\%$ for the exact link of the theorem applied to an InfoNCE Gram
matrix as a positive control (Appendix~\ref{app:exp-s21} and
Figure~\ref{fig:app-s21-net}).

\begin{figure}[t]
\centering
\includegraphics[width=5.5in]{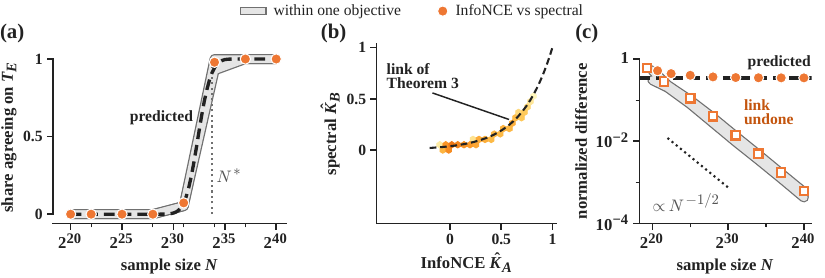}
\caption{\textbf{In the finite Gibbs model of Theorem~\ref{thm:link}, InfoNCE and spectral learners come to agree on the order of distances and differ in geometry by the link of the theorem.} (a)~Share of pairs with the same global order $T_E$, against the share predicted before the runs. (b)~Density of the Gram-matrix entries of one pair. (c)~Median over pairs of the largest difference of mean-normalized distances, before and after the link is undone on the spectral run. Appendix~\ref{app:exp-s23} gives the details.}
\label{fig:exp-link}
\end{figure}

\begin{insight}[label=ins:link]{In the finite Gibbs model, learners trained with different objectives come to agree on order and differ by the link}
Our results suggest that certified minimizers of InfoNCE and of the
spectral contrastive loss in the finite Gibbs model come to agree exactly
on the order of distances and differ in geometry by the link, which
confirms the preregistered predictions of Theorem~\ref{thm:link} inside
its premises.
\end{insight}

\section{Potential Applications}
\label{sec:discussion}
\label{sec:open-problems}

We prove our learning results in a linear autoencoder and the finite
Gibbs model and test them there and in CIFAR-5m networks, whereas their
tests on pretrained models remain future work.
Appendix~\ref{app:limitations} states the limitations of our results and
the open problems they leave. Our results also bear on current areas of
machine learning, and we state their implications for seven areas below,
with further material in Appendix~\ref{app:open-problems}.

\paragraph{Pretraining, Scaling, and Multimodal Encoders.}
PRH and later studies relate model size and performance to mKNN between
vision and language models with $k=10$ on $1{,}024$ image--caption pairs
\citep{huh2024platonic,tjandrasuwita2025understanding}. Size trends at
several values of $k/m$ would show whether larger models agree more on
near neighbors or on coarse structure alone (Theorem~\ref{thm:crossing}).
In multimodal models such as LLaVA \citep{liu2023visual}, a trained linear
connector can undo any invertible linear map of the vision features, a
gauge as in Section~\ref{sec:gauge}. Recent work selects encoders before
connector training by CKA \citep{maniparambil2025harnessing} or by a
Gromov--Wasserstein distance \citep{li2026rethinking}, and both scores
change under such maps. An experiment with maps of growing condition
number applied before this training would show the invariance needed by
such a score. Appendix~\ref{app:op-pretraining} also treats a
preregistered test of Corollary~\ref{cor:aggregation} on CLIP models
\citep{cherti2023reproducible} and a conjecture on the losses of CLIP and
SigLIP.

\paragraph{Post-training, Fine-tuning, and Reinforcement Learning.}
Low-rank adaptation (LoRA) adds a product $BA$ of trainable factors to a
frozen weight \citep{hu2022lora}, with the same update from $VA$ and
$BV^{-1}$ for every invertible $V$, a gauge as in Section~\ref{sec:gauge}.
An $\ell_2$ penalty on both factors restricts this gauge to orthogonal
maps, which a Procrustes rotation could remove before LoraHub composes
adapters \citep{huang2024lorahub}. Under Bradley--Terry preferences, a
classifier of pairwise wins \citep{sun2025rethinking} and a Bradley--Terry
reward model have population limits whose scores differ by an increasing
link within each prompt, as in Theorem~\ref{thm:link}. Their best-of-$N$
responses coincide, whereas the advantages of group relative policy
optimization \citep{shao2024deepseekmath} can change with the link. Under
the dynamics of the three-feature task of \citet{shi2026self}, two runs of
reinforcement learning (RL) share the squared overlaps of their feature
directions in the limit with probability
$p_Q=q^2+(1-q)^2$ (Section~\ref{sec:learning-law}). Here $q$ is the
probability that a run starts below an unstable policy. Replay under
the sufficient condition of that work raises $p_Q$ to one, which in this
task needs a base-reward gap above $7/16$. In
Appendix~\ref{app:op-posttraining}, we propose to compare RL with
supervised fine-tuning.

\paragraph{Agents and Harness Engineering.}
Studies of multi-agent systems measure redundancy or diversity by the
cosine similarity between output embeddings of agents
\citep{kim2025towards,subramaniam2025multiagent}. The random-pairing
baseline of the mean cosine between the output embeddings of two agents is
the inner product of their mean normalized embeddings, large when all
outputs share a format. As in Corollary~\ref{cor:aggregation}, subtraction
of these baselines can reverse the order of two agent systems, and we
propose this subtraction for the redundancy predictor in the scaling model
of \citet{kim2025towards}. The fractions of instances resolved by six runs of one
variant of SWE-agent lie within $1.4$ percentage points of each other
\citep{yang2024swe}. Their reported pass@1 and pass@2 imply that two runs
share on average about half of the instances resolved by either run. The
probability $p_Q$ of Section~\ref{sec:learning-law} extends to criteria
$Q$ on the outputs of two sampled runs of an agent. The harness of an agent
comprises the tools and control code around its model, and a change of
harness matters under $Q$ when runs of two variants meet $Q$ less often
than runs of one variant.
Appendix~\ref{app:op-agents} also treats agents that communicate through
key-value caches or activations.

\paragraph{Generative Models.}
Representation alignment (REPA) speeds up the training of diffusion
transformers by aligning a projection of their hidden states with DINOv2
features \citep{yu2025repa}. \citet{yu2025repa} find that CKNNA
\citep{huh2024platonic} between DINOv2 and diffusion transformers trained
without REPA rises with model size and training, at $k=10$ on $10{,}000$
images. CKNNA is CKA restricted to pairs of images that are $k$ nearest
neighbors in both representations, and these pairs change with the scale
$k/m$ (Section~\ref{sec:neighbors}). REPA-E uses CKNNA as a proxy for
generation quality \citep{leng2025repae}, and it remains open whether
trends and rankings of CKNNA at one scale hold at other scales. Losses
matching the similarities among patch features of a tokenizer
\citep{yao2025reconstruction} or of a diffusion transformer
\citep{chen2025sara} to those of an encoder improve generation. Strictly
increasing maps of the target similarities change the values of these
losses and leave neighbor overlap unchanged (Proposition~\ref{prop:lattice}).
An experiment with such maps would test whether generation benefits from
agreement beyond the order of the similarities.
Appendix~\ref{app:op-generative} also treats the gauge of the autoencoder
in latent diffusion.

\paragraph{Mechanistic Interpretability.}
Sparse autoencoders (SAEs) trained from different seeds share part of
their latents \citep{paulo2026sparse}, and the survey of
\citet{shi2026superposition} concludes that latent reproducibility depends
on the objective, the initialization, and the level of comparison.
\citet{gerasimov2026unstable} find that the latents rarely matched across
seeds concentrate in recurring lower-rank subspaces, and suggest that this
seed dependence reflects ambiguity among bases. The exact symmetries of an
SAE are permutations and positive scalings of its latents, and an
approximate gauge can remain within such a subspace. By analogy with
Theorem~\ref{thm:gauge}, we expect the initialization and the
regularization of SAEs to decide whether runs from different seeds agree
reliably on single directions in these subspaces. Since all global
minimizers of the linear autoencoder on one sample share one projector, we expect these
choices to matter less for agreement on projectors onto the subspaces. In
Appendix~\ref{app:op-interp}, we also argue that mKNN can stay high under
superposition where CKA is low \citep{liu2026similarity}, and treat
universal neurons and top-activating tokens.

\paragraph{Proteins, Molecules, and Materials.}
For interatomic potentials, \citet{edamadaka2025universally} find that
conservative Orb-v3 models align strongly with MACE under CKNNA with
$k=25$ on features averaged over each structure.
\citet{li2026platonicmlip} instead find near-zero overlap between the
neighbor sets of single atoms in MP-20 crystals under Orb-v3 and under
equivariant models including MACE. Under part~1 of
Hypothesis~\ref{hyp:lrh}, this conflict can come from criteria applied at
different points of the lattice, or from the data or the level of the
features. A comparison of both criteria on shared samples and features
would separate these causes. The representations of interatomic potentials
also grow more similar to those of the best-performing model with
decreasing energy error, in scores computed without calibration
\citep{edamadaka2025universally}. Since the baselines of CKA and mKNN can
be forecast from single-model features (Proposition~\ref{prop:baselines}),
a recomputation of these trends with the two scores would show whether
this convergence survives calibration (Corollary~\ref{cor:aggregation}).
Appendix~\ref{app:op-science} also treats homology among proteins and
linear heads in molecule--text and protein models.

\paragraph{Dynamical Systems and Operator Learning.}
DeepONet predicts through the inner product of its branch and trunk
outputs \citep{lu2021learning}, and an invertible map applied to the
branch outputs and undone on the trunk outputs leaves every prediction
unchanged. We conjecture that joint training selects this map as in
Theorem~\ref{thm:gauge}, in which case CKA between the branch outputs of
two runs could stay below one for large samples. Fixed orthonormal Wick
features of the driving noise remove this map in several chaos-expansion
operators of \citet{shi2026expanding}, whose learned propagators can then
be compared across runs coordinate by coordinate.
Across grid resolutions, \citet{bartolucci2023representation} find
aliasing errors in Fourier neural operators and note that discretization
invariance \citep{kovachki2023neural} concerns the limit of infinite
resolution. With the two discretizations of each test function as one
evaluation sample, a neighbor criterion tests the transfer of an operator
between two finite resolutions at a chosen scale $k/m$
(Theorem~\ref{thm:crossing}). Appendix~\ref{app:op-dynamics} also treats
DeepONets with orthonormal trunk outputs and criteria on dynamics and on
long-run statistics.

In these areas, a conclusion drawn from one score computed on the Gram
matrices of two representations rests on one criterion of the lattice, at
one scale and under one protocol. We hope that this lattice view offers a
framework for comparing such conclusions.

\bibliographystyle{plainnat}
\bibliography{references_v3}

\appendix
\numberwithin{definition}{section}
\numberwithin{remark}{section}
\numberwithin{theorem}{section}
\numberwithin{proposition}{section}
\numberwithin{lemma}{section}
\numberwithin{corollary}{section}
\numberwithin{conjecture}{section}
\numberwithin{equation}{section}
\numberwithin{table}{section}
\numberwithin{figure}{section}

\section{Related Work}
\label{app:related}

\paragraph{Representation Hypotheses and Current Evidence.}
PRH formulates convergence in terms of the similarity structures induced
by representations \citep{huh2024platonic}. Its main empirical measure is
the mean size of the intersection of the two models'
$k$-nearest-neighbor sets over matched inputs, normalized by $k$, the
mutual $k$-nearest-neighbor overlap (mKNN). This measure scores
agreement on neighbor identities and leaves the corresponding distances
unconstrained, whereas convergence to a common kernel in PRH's idealized
analysis implies agreement of the induced pairwise distances.
ARH \citep{groger2026revisiting} finds that mKNN keeps its size trend
after permutation calibration whereas several geometric measures largely
lose this trend, and accordingly posits convergence toward shared
neighborhood relations rather than shared distances. BPC finds that mKNN
between vision and language models at a fixed $k$ falls sharply on
larger evaluation sets and stays stable when $k$ grows in proportion
\citep{koepke2026back}. SOG adds the overlap of minimum spanning tree
(MST) edge sets to score global structure and finds a size trend in this
overlap \citep{you2026structure}. In SOG, the size trend of the
length-weighted overlaps is weaker at a smaller length tolerance $\tau$,
which the authors interpret as convergence in structure rather than
geometry. These four studies apply different criteria and reach
conflicting conclusions about the same models. In this work, we place
the criteria of all four studies in one lattice and identify the
differences between two representations that determine the outcome of
each criterion.

\paragraph{Alignment without Paired Inputs.}
The method of \citet{schnaus2026rosetta} uses the geometry shared by two
models to align vision and language embeddings without any paired
inputs, fitting one orthogonal map after centering and normalization. On
COCO the true partner of a query ranks on average behind $15\%$ of the
candidates, against $0.06\%$ for CLIP, so the recovered correspondence
is coarse. These authors compute CKA with
the unbiased estimator of the Hilbert--Schmidt independence criterion,
which gives the debiased CKA of Proposition~\ref{prop:baselines} with its
zero random-pairing baseline. Pooled over four caption corpora, this score
predicts the success of their alignment better than mKNN at $k=10$ and
than the triplet and quadruplet scores of \citet{soares2026scalable}.
With one speech encoder applied to two speakers reading the same
sentences, the identity map recovers the correspondence exactly, while
their aligner leaves a mean share of $0.265$ of candidates closer than
the true partner. This example suggests that the success of alignment
without pairs also depends on whether the correspondence can be
identified from each space alone.

\paragraph{Similarity Measures and Their Invariances.}
Representational similarity measures are commonly classified by the
transformations under which they are invariant
\citep{kornblith2019similarity,williams2021generalized,klabunde2025similarity},
and \citet{sucholutsky2023getting} organize alignment studies around the
measurement taken from each system and the function comparing the two
measurements. Closer to our ordering of criteria,
\citet{cloos2025differentiable} estimate which values of one similarity
score remain reachable at a given value of a second score, and
\citet{harvey2024what} relate several measures to the agreement of linear
decoders. \citet{lin2024topology} interpolate between geometry and
topology through monotone transformations of dissimilarities.
\citet{soares2026scalable} show that agreement on all anchored triplet
comparisons is equivalent to mutual nearest-neighbor agreement at every
$k$ when the distances from each anchor are distinct. These works
compare measures on fixed representations, whereas we order the criteria
themselves, including their tolerance, scale, and evaluation protocol,
and ask which learning conditions make agreement under each criterion
reliable across independent runs.

\paragraph{Identifiability and Reproducibility of Learned Representations.}
Contrastive learning recovers a latent relation under suitable
assumptions on the data-generating process
\citep{zimmermann2021contrastive,lyu2022finite}, and the idealized
contrastive learner of \citet{huh2024platonic} recovers a pointwise
mutual information kernel. Models trained on prediction tasks with a
linear head learn representations that are identifiable up to an
invertible linear map \citep{roeder2021linear},
and the symmetries of network parameterizations act on hidden
representations \citep{godfrey2022symmetries}.
\citet{nelson2026statistical} quantify identifiability up to a
transformation group within a tolerance. \citet{ziyin2025proof} prove
that, in embedded deep linear networks, the minimizers of the loss
modeling finite-step stochastic gradient descent are aligned up to
rotation, although most global minima of the original loss lack this
alignment. \citet{doshi2026theory} model languages that share the upper
levels of a probabilistic grammar and use the Bayes-optimal predictor to
derive when representations share neighborhood geometry and when they
share coordinates. In their experiments, the representations of trained
transformers match these predictions. Networks trained from different
random seeds differ measurably in their representations
\citep{mehrer2020individual,ding2021grounding}. The probability that two
runs with shared internal randomness return the same output on
independent samples has been formalized as reproducibility
\citep{impagliazzo2022reproducibility}. In contrast, $p_Q$ is the
probability that two runs with independent randomness meet a specified
criterion $Q$, and we call $Q$ obstructed if this probability stays low
across a family of interventions on the learning conditions
(Definition~\ref{def:reliable}).

\paragraph{Ordinal Geometry and Stability.}
For dense samples in a fixed Euclidean dimension, ordinal information
can determine geometry, since maps preserving distance comparisons are
close to similarity transformations \citep{kleindessner2014uniqueness}.
In the same setting, $k$-nearest-neighbor graphs suffice to recover an
embedding up to similarity when $k$ grows suitably with the sample
\citep{terada2014local}. These results assume a common low dimension,
whereas learned representations typically have high rank. In
Proposition~\ref{prop:bernstein}, a representation and a full-rank
transform of this representation agree on the order of all distances,
while their CKA can come close to zero when the inputs vary along few
effective directions relative to $m$. Worst-case stability of discrete
relations under perturbation is classical, and MST sensitivity analysis
computes the largest change of each edge length that leaves the tree
unchanged \citep{tarjan1982sensitivity}. By the cycle condition behind
this analysis, the MST is unchanged when every distance moves by less
than half of the smallest gap between a non-tree edge and the heaviest
tree edge on its cycle. Persistence diagrams are stable in the
bottleneck distance \citep{cohensteiner2007stability}, and
nearest-neighbor sets are preserved when points move by less than a
quarter of the gap between the $k$th and $(k+1)$th neighbor distances
\citep{sumiya2026finite}. In this work, we complement these worst-case
bounds with a typical-case law for neighbor sets
(Theorem~\ref{thm:crossing}). Under the assumptions of this theorem, the
conditional expected number of replaced neighbors at an anchor lies
almost surely within $\sqrt{k/2}$ of the prediction of this law at every
evaluation size.

\section{Scores, Observables, and the Four Studies}
\label{app:measurement-definitions}

This appendix fixes the conventions behind the scores of
Section~\ref{sec:criteria}, defines the observables of
Proposition~\ref{prop:lattice}, and lists the criteria of the four
studies.

\paragraph{Score Conventions.}
As in PRH \citep{huh2024platonic}, the neighbor sets $\Nset_k^a(i)$ of
mKNN exclude the input itself. Neighbor and MST ties are broken by the
input index. The tree $\Tset_a$ is computed on the complete graph
weighted by $d_a$, which gives the same tree as weighting by $D_a$. For
linear CKA,
$\langle\bar K_A,\bar K_B\rangle_F=\sum_{i,j}(\bar K_A)_{ij}(\bar K_B)_{ij}$
and $\norm{\bar K}_F^2=\langle\bar K,\bar K\rangle_F$
\citep{kornblith2019similarity}. The score requires nonzero centered Gram
matrices and equals one exactly when the two centered matrices are
positive scalar multiples.

In the calibration formula of Section~\ref{sec:protocol}, the threshold
$c_\alpha<s_{\max}$ is the rank-based $(1-\alpha)$ quantile of the
aggregated scores on the observed pairing and on the permuted pairings
\citep{groger2026revisiting}. Permutations relabel the matched
inputs of one model while preserving each representation. When the
protocol aggregates by a maximum over layer pairs, this maximum is
recomputed within every permutation, with the same permutation of inputs
at every layer. A maximum over layer pairs is high when the score of a
single layer pair is high, and it carries an upward bias growing with
the number of layers \citep{groger2026revisiting}.
Section~\ref{sec:learning-law} therefore fixes one representation map
per model and treats aggregation over layers as part of the protocol.

In the preprocessing map of their protocols, PRH, ARH, and SOG clip
every coordinate at one threshold per model and then normalize each
feature vector to unit length. This threshold is the mean over inputs of
the $95$th percentile of the absolute coordinates of each input, pooled
over all layers of the model.

For the length-weighted overlaps of SOG, $\mathcal E_a$ contains either
the $mk$ directed neighbor edges or the $m-1$ MST edges of model $a$
\citep{you2026structure}, and the score is
\[
 s_{\mathcal E}^{(\tau)}
   =\frac{\sum_{(i,j)\in\mathcal E_A\cap\mathcal E_B}
                    w_{ij}^{(\tau)}}{|\mathcal E_A|}.
\]
Distances are normalized by $\bar d_a$, the $90$th percentile of the
positive off-diagonal entries of $d_a$, and the edge weight is
\begin{equation}
 w_{ij}^{(\tau)}=
 \exp\!\left(-\frac{|\log\widetilde d^A_{ij}
                       -\log\widetilde d^B_{ij}|}{\tau}\right),
 \qquad
 \widetilde d^a_{ij}=\frac{d_a(i,j)}{\bar d_a},
 \quad\tau>0,
 \label{eq:sog-weight}
\end{equation}
which uses distances rather than squared distances. Smaller $\tau$
penalizes length mismatch more strongly, and $\tau\to\infty$ recovers
unweighted edge overlap.

\paragraph{Observables.}
Observables are the measurable maps $T:\Kset\to\Yset_T$ of
Definition~\ref{def:refinement}. If an observable needs information
outside the Gram matrix, its domain is enlarged to include that
information. Under $D(K)$, let $\Nset_k^K(i)$ be the neighbor set of
input $i$ and $o_i^K$ the ranking of the other inputs by their distance
to input $i$, and let the $k$-nearest-neighbor graph $G_k(K)$ join $i$
and $j$ if $j\in\Nset_k^K(i)$ or $i\in\Nset_k^K(j)$. The
observables of Proposition~\ref{prop:lattice} are as follows:
\begin{itemize}
\item $T_G(K)=D(K)$, all squared distances;
\item $T_{G/\R_+}$, the class of $D(K)$ under multiplication by a
positive constant, whose values on two Gram matrices determine their
linear CKA since $\bar K=-\tfrac12HD(K)H$;
\item $T_{G/\mathrm{Aff}}$, the class of $D(K)$ under the off-diagonal
maps that multiply every entry by one positive constant and add a second
constant, whose values on two Gram matrices determine the Pearson
correlation of their distance matrices;
\item $T_E$, the order of the entries $D(K)_{ij}$, $i<j$, whose values
on two Gram matrices determine the Spearman correlation of these
entries;
\item $T_O(K)=(o_i^K)_{i=1}^m$, the rankings of the other inputs by
their distance to each input;
\item $T_{N,k}(K)=(\Nset_k^K(i))_{i=1}^m$, the neighbor sets, and
$T_{C,k}(K)=\Comp(G_k(K))$, the connected components of $G_k(K)$;
\item $T_{SL}$, the single-linkage dendrogram, i.e., the MST edges
together with their order of addition;
\item $T_{\mathrm{MST}}$, the MST edge set, and $T_{N,1}$, the map
sending each input to its nearest neighbor.
\end{itemize}

\paragraph{Criteria of the Four Studies.}
Table~\ref{tab:claims-practice} lists the claim of each of the four
studies together with the criterion behind its scores.

\begin{table}[ht]
\centering\small
\setlength{\tabcolsep}{4pt}
\renewcommand{\arraystretch}{1.15}
\begin{tabular}{@{}p{.08\linewidth}p{.16\linewidth}p{.22\linewidth}p{.17\linewidth}p{.13\linewidth}p{.13\linewidth}@{}}
\toprule
Study & Claim & Observable and score & Scale and length tolerance & Layers & Calibration \\
\midrule
PRH & shared statistical model of reality; idealized analysis: common kernel up to offset & mKNN & $k=10$, $m=1024$ & max over all layer pairs and concatenated features & none \\
ARH & shared local neighborhoods rather than global geometry & mKNN, CKNNA, cycle-kNN; linear, RBF, and unbiased CKA; RV; SVCCA, PWCCA; Procrustes & $k\in\{10,20,50,100\}$, $m=1024$ & max over all layer pairs & $s_{\rm cal}$ of Section~\ref{sec:protocol}, max recomputed inside each permutation \\
BPC & multimodal representations share coarse structure without evidence of convergence & mKNN for the $m$ sweep; RSA and linear CKA on ImageNet class galleries & $k\in\{1,10\}$ or $k=m/100$; $m$ from $1{,}024$ (WIT) to $1.5\times10^7$ (LAION) & best pair selected on WIT-1024, then fixed & none for the $m$ sweep; permutation-calibrated variants in an appendix \\
SOG & relational structure converges at local and global scale, metric geometry weakly & mKNN; MST overlap; length-weighted overlaps & $k=10$, $m=1024$, $\tau$ from $1$ to $10^{-4}$ & max over all layer pairs & as ARH, $500$ permutations \\
\bottomrule
\end{tabular}
\caption{Claims and criteria of four studies of representational
convergence. All four evaluate pretrained models.}
\label{tab:claims-practice}
\end{table}

\section{Supplements to Section~\ref{sec:criteria}}
\label{app:criteria-supp}

This appendix first shows that the strength order of criteria with
positive tolerances depends on their tolerances as well as on their
observables, and then completes the lattice of criteria. It also bounds
the probability that models from two learning conditions meet an exact
criterion by the probabilities of agreement within each condition,
defines transitions, and proves Proposition~\ref{prop:lattice}.

\subsection{Refinement and Criteria with Positive Tolerances}

Information refinement of the observables alone does not determine the
strength order of criteria with positive tolerances. A pair of models
may fail a stringent neighbor criterion because of a near tie between
the $k$th and $(k+1)$th nearest inputs of one anchor, and still satisfy
a loose distance criterion.
In the opposite case, a CKA criterion with a small tolerance can accept a
pair of independent representations of large width, while the
expectation of their mKNN stays at its chance level.
\citet{groger2026revisiting} show that linear CKA at a fixed number of
inputs tends to one with growing width even when the two representations
share a part of their signal and differ in the rest (their
Proposition~C.5). The argument of this proposition applies unchanged to
independent representations, whereas mKNN between independent
representations has expectation $k/(m-1)$ at every width (their
Theorem~C.10). Strength is therefore an order on complete criteria,
including their tolerances.

\subsection{The Lattice of Criteria}

Up to equivalence, an observable is the partition of $\Kset$ into its
fibers, and partitions ordered by refinement form a lattice. Rankings,
neighbor sets, and components are invariant under strictly increasing
maps of each row of $D$. The observables $T_E$, $T_{SL}$, and
$T_{\mathrm{MST}}$ are invariant under one strictly increasing map of
all distances, and $T_{G/\mathrm{Aff}}$ and $T_{G/\R_+}$ under positive
affine maps and positive rescaling, respectively. For $T_O$, $T_E$,
$T_{G/\mathrm{Aff}}$, and $T_{G/\R_+}$, the invariances listed above
exhaust the invariances on matrices with distinct entries, and the
coarser observables are invariant under more maps. These invariances fix
the differences between two representations ignored by each observable.
The pair $(T(K_A),T(K_B))$ determines the outcome of an exact criterion
$(T,=)$ in the sense of Section~\ref{sec:observables}. A distance
criterion accepts a pair of models if the largest absolute entry or the
root mean square of the entries of $D_A-D_B$ is at most its tolerance.

For fixed-layer criteria $Q_1$ and $Q_2$ with positive tolerances
$\varepsilon_1$ and $\varepsilon_2$, the join and the meet of
Section~\ref{sec:observables} are again criteria in the sense of
Definition~\ref{def:criterion}. Both have the observable $(T_1,T_2)$ and
tolerance one, with discrepancy
$\max\{\Delta_1/\varepsilon_1,\Delta_2/\varepsilon_2\}$ for the join and
$\min\{\Delta_1/\varepsilon_1,\Delta_2/\varepsilon_2\}$ for the meet. A
smaller tolerance for a fixed observable and discrepancy gives a
stronger criterion, and so does a larger set of compared input pairs for
a criterion on the largest distance error over these pairs, at a fixed
tolerance. A distance-error bound
implies neighbor agreement at anchors whose neighbor gap exceeds twice
the error (Equation~\eqref{eq:neighbor-stability}). Criteria that differ
in the neighbor count or in the protocol are in general unordered by
strength.

\subsection{Shared Corpus and the Within-condition Ceiling}

If the models share a random training corpus $\mathcal C$, the product law
in the definition of $p_Q$ (Section~\ref{sec:learning-law}) is instead
formed from
$\mu_{r_A}(\cdot\mid\mathcal C)$ and $\mu_{r_B}(\cdot\mid\mathcal C)$,
provided the remaining randomness is independent. The average over
$\mathcal C$ then gives the probability of agreement for models trained
on one shared corpus, which in general differs from $p_Q(r_A,r_B)$. For
exact criteria, the probabilities of agreement within each of two
learning conditions bound the probability of agreement across the two
conditions, and the same holds for expected scores that are inner
products.

\begin{remark}[Within-condition Ceiling]
\label{rem:ceiling}
For an exact criterion $(T,=)$ whose observable $T$ takes countably many
values, let $\mu^T_r$ be the law of $T(K)$ under $\mu_r$. Then
$p_{(T,=)}(r_A,r_B)$
is the inner product of the probability mass functions of
$\mu^T_{r_A}$ and $\mu^T_{r_B}$, so
$p_{(T,=)}(r_A,r_B)\le\bigl(p_{(T,=)}(r_A,r_A)\,p_{(T,=)}(r_B,r_B)\bigr)^{1/2}$
by the Cauchy--Schwarz inequality. The same argument bounds every
expected score of the form $\E\langle\Theta(K_A),\Theta(K_B)\rangle$ with
independent $K_A\sim\mu_{r_A}$ and $K_B\sim\mu_{r_B}$, including
$\E s_N$, $\E|\Tset_A\cap\Tset_B|$, and $\E\CKA$ with
$\Theta(K)=\bar K/\norm{\bar K}_F$. The bound does not extend to
criteria with a positive tolerance. With a real-valued $T$, the
discrepancy $|T(K)-T(K')|$, a tolerance $\varepsilon>0$, $\mu^T_{r_A}$
uniform on $\{0,2\varepsilon\}$, and $\mu^T_{r_B}$ a point mass at
$\varepsilon$, we have $p_Q(r_A,r_A)=1/2$ and $p_Q(r_B,r_B)=1$, whereas
$p_Q(r_A,r_B)=1>2^{-1/2}$. Within one learning condition, Jensen's
inequality gives
$\E_{\mathcal C}\norm{\E[\Theta(K)\mid\mathcal C]}^2\ge\norm{\E\Theta(K)}^2$,
so a shared corpus does not lower such an expected score.
\end{remark}

The seed-pair probability of a learning condition is the probability
that two runs with independent training samples and seeds under this
condition meet a criterion. By Remark~\ref{rem:ceiling}, the
probability that models from two conditions meet an exact criterion is
at most the geometric mean of the two seed-pair probabilities, and the
same bound holds for the expected scores above. For criteria with a positive
tolerance, the seed-pair probability is a reference rather than a bound.

\subsection{Transitions}

When models share one relation but fail a stronger criterion, we ask
whether a change of the learning conditions can remove the remaining
discrepancy. In the learning condition $r_a$ of
Section~\ref{sec:learning-law}, the training procedure
$\mathsf{Train}_a$ fits a function from the model class $\mathcal H_a$
to a training sample $\mathcal D_a\sim P_a^{N_a}$ under the objective
$\mathcal L_a$, with a training budget and randomness $\xi_a$. To test
whether more training samples remove the discrepancy, we compare $p_Q$
before and after an increase of $N_a$ at a fixed data distribution,
objective, model class, and training procedure. A comparison across
architectures, training budgets, or objectives instead tests whether the
changed entry of the learning condition limits $p_Q$. The criterion and
the evaluation inputs stay fixed throughout such a comparison, so a
change of $p_Q$ comes from a change in the learned representations. A
change of criterion alters which learned differences count as
disagreement under fixed representations, whereas an intervention alters
the representations under a fixed criterion. Division of all distances by
their mean and restriction of a comparison to pairs within groups are
two such changes of criterion, which make the criterion ignore a uniform
rescaling and the discrepancies between groups, respectively.

We expect the outcome of a criterion obstructed over one intervention
family to depend on a difference between the two learned representations
ignored by some weaker criterion and reduced by an intervention of
another kind. Since $Q_1\preceq Q_2$ implies $p_{Q_1}\ge p_{Q_2}$, the
set $\mathcal R_\delta(r_A,r_B)$ of Hypothesis~\ref{hyp:lrh} contains
every criterion weaker than one of its members, so we describe the
effect of an intervention on this set by the change of its strongest
members. In Theorem~\ref{thm:gauge}, more data leaves the obstruction of
CKA criteria of small tolerance in place, while weight decay removes
this obstruction. In Theorem~\ref{thm:link}, criteria on
scale-normalized geometry with a small tolerance are asymptotically
obstructed for learners trained with different objectives, while one
objective for both learners removes this obstruction.

\begin{definition}[Transition]
\label{def:transition}
Fix the level $1-\delta$ and the separation $\eta$ of
Definition~\ref{def:reliable}. For a weaker criterion $Q_c$ and a
stronger criterion $Q_f$ with $Q_c\preceq Q_f$, an intervention from
$(r_A,r_B)$ to $(r'_A,r'_B)$ produces a \emph{transition} if
\begin{equation}
 p_{Q_c}(r_A,r_B)\ge1-\delta,\qquad
 p_{Q_f}(r_A,r_B)\le1-\delta-\eta,\qquad
 p_{Q_f}(r'_A,r'_B)\ge1-\delta.
 \label{eq:attainable-transition}
\end{equation}
\end{definition}

The separation $\eta$ keeps a transition from being declared on the basis
of a probability that merely crosses the level $1-\delta$, and the
supremum in Definition~\ref{def:reliable} keeps an obstruction from
vanishing in the limit of the intervention family. A proof of the
transition in Equation~\eqref{eq:attainable-transition} therefore needs
a lower bound on $p_{Q_f}$ after the intervention and an upper bound on
$p_{Q_f}$ before the intervention. A sufficient condition for agreement
under $Q_f$ that fails at the original conditions leaves the second
inequality unproved.
Theorems~\ref{thm:gauge}, \ref{thm:link}, and~\ref{thm:peer} prove
obstructions in the asymptotic form.

\subsection{Proof of Proposition~\ref{prop:lattice}}
\label{app:proofs-prelim}

\begin{proof}[Proof of Proposition~\ref{prop:lattice}]
Part~1. The connected components of $G_k$ are computed from the neighbor
sets, and the neighbor sets are the first $k$ entries of the rankings.
The ranking around $i$ is the order of the entries $D(i,l)$, $l\ne i$,
i.e., the restriction of $T_E$ to the pairs containing $i$. Every
off-diagonal affine map with positive slope preserves the order of the
off-diagonal entries, so $T_E$ is a function of $T_{G/\mathrm{Aff}}$. The
class of $D$ under positive rescaling is contained in its affine class,
so $T_{G/\mathrm{Aff}}$ is a function of $T_{G/\R_+}$, which is a
function of $T_G$.

Part~2. Kruskal's algorithm examines the edges in increasing order of
distance and accepts an edge when this edge joins two trees of the
current forest. Since the MST and the order of acceptance of its edges
depend on $T_E$ alone, $T_{SL}\preceq T_E$, and
$T_{\mathrm{MST}}\preceq T_{SL}$ follows by forgetting the order. For an
input $i$, the edge $\{i,\NN(i)\}$ is the lightest edge crossing the cut
between $\{i\}$ and the remaining inputs, so this edge belongs to the
MST by the cut property. Every other MST edge at $i$ crosses the same
cut and is heavier than this edge. Hence $\NN(i)$ is the MST neighbor of
$i$ whose edge comes first
in the order recorded by $T_{SL}$, and $T_{N,1}\preceq T_{SL}$.

Part~3. All configurations below are point sets in the plane with
pairwise distinct distances. The Gaussian Gram matrix
$K_{ij}=e^{-\norm{\mathbf p_i-\mathbf p_j}^2/2}$ has unit diagonal and is
positive definite, and
$D(K)_{ij}=2-2e^{-\norm{\mathbf p_i-\mathbf p_j}^2/2}$ is a strictly
increasing function of the Euclidean squared distance, so $K$ has the
same $T_E$, $T_O$, $T_{N,k}$, and MST as the point set.

For $T_{\mathrm{MST}}\not\preceq T_O$ at $m=4$, take
$\mathbf p_0=(0,0)$, $\mathbf p_1=(2\pm0.1,0)$, $\mathbf p_2=(2,1)$, and
$\mathbf p_3=(0,1)$, with the plus sign in one configuration and the
minus sign in the other.
The distances are $d_{01}=2\pm0.1$, $d_{02}=\sqrt5$, $d_{03}=1$,
$d_{12}=\sqrt{1.01}$, $d_{13}=\sqrt{(2\pm0.1)^2+1}$, and $d_{23}=2$. In
both configurations the rankings are $(3,1,2)$ around $\mathbf p_0$,
$(2,0,3)$ around $\mathbf p_1$, $(1,3,0)$ around $\mathbf p_2$, and
$(0,2,1)$ around $\mathbf p_3$, while the MST is $\{03,12,23\}$ for the
plus sign and $\{03,12,01\}$ for the minus sign. For $m>4$, add the
points $\mathbf p_{3+j}=(1/2,\,10\cdot3^j)$ for $j=1,\dots,m-4$;
consecutive heights $10\cdot3^j$ and $10\cdot3^{j+1}$ differ by twice the
lower height. From $\mathbf p_{3+j}$, the added points below
$\mathbf p_{3+j}$ are closer than $\mathbf p_0,\dots,\mathbf p_3$, which
follow in the order $\mathbf p_3,\mathbf p_2,\mathbf p_0,\mathbf p_1$,
and the added points above $\mathbf p_{3+j}$ are farther than
$\mathbf p_1$. From $\mathbf p_i$ with $i\le3$, every added point is
farther than all of $\mathbf p_0,\dots,\mathbf p_3$, and the added points
follow in the order of their height. Hence $T_O$ is the same in both
configurations, while the MST is the four-point tree together with the
edges $\{\mathbf p_3,\mathbf p_4\}$ and
$\{\mathbf p_{3+j},\mathbf p_{4+j}\}$, so the two MSTs still differ. An
observable $q\circ T_O$ takes equal values on the two configurations, so
$T_{\mathrm{MST}}\ne q\circ T_O$ for every map $q$, and
$T_{\mathrm{MST}}\not\preceq T_{N,k}$ follows from $T_{N,k}\preceq T_O$.

For $T_{N,k}\not\preceq T_{\mathrm{MST}}$, place
$\mathbf p_2,\dots,\mathbf p_m$ on a line at small generic perturbations
of the positions $0,1,\dots,m-2$, and place $\mathbf p_1$ at $-1/2$ in
one configuration and at $-2m$ in the other. In both configurations the
MST is the path through consecutive points, while the $k$ nearest
neighbors of $\mathbf p_2$ have the indices $\{1,3,\dots,k+1\}$ in the
first configuration and $\{3,\dots,k+2\}$ in the second.
\end{proof}

\section{Proofs and Supplements for Section~\ref{sec:reads}}
\label{app:proofs-reads}

This appendix states the two lemmas behind Section~\ref{sec:reads}, the
margin form of the crossing law, the scale conjecture, the proposition on
ordinal agreement without geometric agreement, and the formal version of
Corollary~\ref{cor:aggregation}. Its last subsection gives the full
version of Theorem~\ref{thm:crossing}, restates
Propositions~\ref{prop:mst} to~\ref{prop:baselines}, and proves both
lemmas and every result of Section~\ref{sec:reads}.

\subsection{Rows, Margins, and the Margin Form}

Neighbor criteria such as mKNN are row-wise, i.e., built from each
anchor's row of distances alone. The following lemma therefore gives the
law of a row of $D_B$ given its anchor and the row of $D_A$ under the
coupling of Section~\ref{sec:neighbors}.

\begin{lemma}[Conditional Independence within a Row]
\label{lem:rows}
Let $x_1,\dots,x_m$ be independent draws from $P_S$. Conditionally on
$x_i$ and on the row $(D_A(i,l))_{l\ne i}$, the entries
$(D_B(i,l))_{l\ne i}$ are independent, and $D_B(i,l)$ follows the
conditional law of $D_B(x_i,x')$ given $D_A(x_i,x')=D_A(i,l)$ under the
coupling at $x_i$.
\end{lemma}

By Lemma~\ref{lem:rows}, the law of a row-wise score at an anchor is a
function of the coupling at this anchor, and its expectation over the
inputs is an average of such functions. We can therefore study the
expected neighbor overlap through the coupling alone. With
$\mathcal G$ generated by $x_i$ and the row $(D_A(i,l))_{l\ne i}$, the
lemma also gives the conditional independence assumed in
Theorem~\ref{thm:crossing}, with continuous laws whenever the coupling
has continuous conditional laws. The order $T_E$ of all distances, spanning trees,
and CKA also depend on the joint law of entries in different rows.
Section~\ref{sec:mst} instead relates spanning trees to neighbor sets,
and Section~\ref{sec:geometry} relates CKA to a monotone link and a
residual.

Classical stability results \citep[e.g.,][]{tarjan1982sensitivity} bound
the change of an observable by its margin, i.e., the size of the smallest
perturbation of the distances that changes the observable. Here observables are regarded as functions of the
off-diagonal distance profile, extended to row profiles in
$\R^{m\times(m-1)}$ for row-wise observables, and balls are taken in that
space. For an observable $T$ and a norm on distance profiles, let
$\rho_T(D)$ be the supremum of the radii of the balls around $D$ on which
$T$ is constant, and let $\mathcal W_T$ be the set of maps on distance
profiles that leave $T$ unchanged.

\begin{lemma}[Margins and Invariances]
\label{lem:margins}
If some map in $\mathcal W_T$ sends $D_A$ to a profile at distance less
than $\rho_T(D_B)$ from $D_B$, then $T(D_A)=T(D_B)$. If $T_1\preceq T_2$,
then $\rho_{T_1}\ge\rho_{T_2}$ pointwise and
$\mathcal W_{T_2}\subseteq\mathcal W_{T_1}$.
\end{lemma}

Coarser observables therefore have larger margins and ignore more
differences. Write $D_{A,i,(k)}$ for the $k$th smallest non-self entry of
row $i$ of $D_A$. For neighbor sets in the row maximum norm,
$\rho_{T_{N,k}}$ at anchor $i$ is half of the gap
$\gamma_{i,k}^A=D_{A,i,(k+1)}-D_{A,i,(k)}$, which gives the worst-case bound
\citep[cf.][]{sumiya2026finite}
\begin{equation}
 1-s_N(K_A,K_B)
 \le\frac1m\sum_{i=1}^m
       \indic\{\gamma_{i,k}^A\le2\,\mathrm{err}_i\},
 \qquad
 \mathrm{err}_i=\max_{j\ne i}|D_A(i,j)-D_B(i,j)|.
 \label{eq:neighbor-stability}
\end{equation}
The bound requires the largest discrepancy in a row over all $m-1$ other
inputs to stay below half of the neighbor gap, and the gap shrinks with
$m$. Unlike this bound, the expected count $\kappa_i$ of
Theorem~\ref{thm:crossing} is unchanged by deterministic strictly
increasing maps of each row, which also leave neighbor sets unchanged. When model B's row is
a strictly increasing function of model A's row plus a symmetric
residual, the crossing law gives a bound that weights each input by its
probability of crossing the boundary.

\begin{corollary}[Margin Form]
\label{cor:margin-form}
Suppose $b_l=\phi(D_A(i,l))+R_{il}$ for $l\ne i$, with $\phi$ strictly
increasing and residual entries $R_{il}$ independent across $l$ and
symmetric about zero. Then
\[
 \E\bigl[1-s_N(i)\bigr]\le\frac1k\inf_\vartheta\sum_{l\ne i}
 \Prob\bigl(R_{il}\ge|\phi(D_A(i,l))-\vartheta|\bigr),
\]
where $s_N(i)=|\Nset_k^A(i)\cap\Nset_k^B(i)|/k$ and the infimum runs over
$\vartheta$ between $\phi(D_{A,i,(k)})$ and $\phi(D_{A,i,(k+1)})$. If
$\phi$ is the identity and $|R_{il}|\le\mathrm{err}_i$ almost surely for
every $l$, the bound is zero at every anchor with
$\gamma^A_{i,k}>2\,\mathrm{err}_i$. With $1-s_N(i)\le1$ at the other
anchors, this gives Equation~\eqref{eq:neighbor-stability}.
\end{corollary}

The worst-case bound counts an anchor as lost whenever the largest
discrepancy in its row reaches half of the neighbor gap, whereas
Corollary~\ref{cor:margin-form} counts the inputs expected to cross the
boundary at this anchor. Neighbor overlap therefore depends on the
density of distances near the neighbor boundary relative to the
dispersion of the coupling at this boundary.

\subsection{Scale Classes}
\label{app:scale}

The neighbor boundary moves inward for larger evaluation sets, since the
$k$ nearest neighbors of an input among $m$ inputs lie within a radius
$\varrho_k$ decreasing in $m/k$. To measure the dispersion of the
coupling in the distances of model A, we work with
$d_a(x,x')=D_a(x,x')^{1/2}$, and this strictly increasing change of every
row leaves each $\kappa_i$ unchanged by Theorem~\ref{thm:crossing}. Fix an anchor $x$, let $\phi_x$ be the
conditional median of $d_B(x,x')$ given $d_A(x,x')$, and let
$\tilde R=\bigl(d_B(x,x')-\phi_x(d_A(x,x'))\bigr)/\phi_x'(d_A(x,x'))$ be
the residual around $\phi_x$ expressed in the distances of model A. The
relative dispersion at distance $\varrho$ is
\[
 \sigma_{\rm rel}(\varrho)=\frac{\E\bigl[\,|\tilde R-\E[\tilde R\mid d_A(x,x')=\varrho]|
 \bigm| d_A(x,x')=\varrho\bigr]}{\varrho}.
\]
We state the resulting classification as a conjecture, since the
derivation below is heuristic and a rigorous version awaits a local
Poisson limit.

\begin{conjecture}[Scale Classes]
\label{conj:scale}
Let the evaluation inputs be independent draws from a continuous
positive density on a compact Riemannian manifold of dimension
$d_{\rm int}$, and let $k$ be fixed as $m\to\infty$.
\begin{enumerate}
\item If both representations are $C^2$ embeddings whose pullback
metrics are conformal to each other, then $s_N\to1$ in probability.
\item If both are $C^2$ embeddings whose pullback metrics are not
conformal on a set of positive measure, then $s_N$ converges in
probability to a constant in $(0,1)$.
\item If $\sigma_{\rm rel}(\varrho)\to\infty$ as $\varrho\to0$ at almost
every anchor, then $s_N\to0$ in probability.
\end{enumerate}
\end{conjecture}

The following heuristic derivation gives the first-order law
$1-s_N\approx\tfrac{d_{\rm int}}2\E_x\sigma_{\rm rel}(\varrho_k)$ with
$\varrho_k\propto(k/m)^{1/d_{\rm int}}$, and each case of
Conjecture~\ref{conj:scale} corresponds to one limiting behavior of
$\sigma_{\rm rel}(\varrho)$ as $\varrho\to0$. Fix an anchor $x=x_i$ and
write $d_B(i,l)=\phi_x(d_A(i,l))+\phi_x'(d_A(i,l))\tilde R_{il}$, so that
$\tilde R_{il}$ is the residual around $\phi_x$ in the coordinates of
model A. Near the neighbor boundary $d_{A,i,(k)}=D_{A,i,(k)}^{1/2}$, the
distances of model A form approximately a homogeneous point process with intensity
$\Lambda_x$ per unit distance, and the residual has a law that varies
slowly over its own spread. For a threshold placed at offset $\vartheta$
from the boundary in the coordinates of model A, this gives
$\E X=\Lambda_x\E[\tilde R-\vartheta]_+$ and
$\E Y=\Lambda_x\E[\vartheta-\tilde R]_+$, which balance at
$\vartheta=\E\tilde R$, so
\[
 \kappa_i\approx\Lambda_x\,\E\bigl|\tilde R-\E\tilde R\bigr|/2 .
\]
For inputs drawn from a density on a manifold of dimension $d_{\rm int}$,
the number of inputs within distance $\varrho$ of $x$ grows as
$\varrho^{d_{\rm int}}$ near the neighbor radius
$\varrho_k\propto(k/m)^{1/d_{\rm int}}$, so $\Lambda_x=d_{\rm int}k/\varrho_k$
and
\[
 1-s_N\approx\frac{\E\kappa_i}k\approx
 \frac{d_{\rm int}}2\,\E_x\,\sigma_{\rm rel}(\varrho_k).
\]
For locally conformal embeddings, the residual at distance $\varrho$ is of
order $\varrho^2$, so $\sigma_{\rm rel}(\varrho_k)\to0$ and $s_N\to1$. For
smooth non-conformal embeddings, the residual is of order $\varrho$, so
$\sigma_{\rm rel}$ tends to a positive constant and $s_N$ to a limit in
$(0,1)$. When $\sigma_{\rm rel}(\varrho)\to\infty$ as $\varrho\to0$, as
happens when the dispersion of the residual stays above a positive
constant, $\sigma_{\rm rel}(\varrho_k)$ grows without bound with $m$ and
the first-order law leaves its range of validity. The expected crossing
count $\kappa_i$ then approaches $k$ at each such anchor, and $s_N\to0$
when this holds at almost every anchor. With
$\sigma_{\rm rel}(\varrho)\propto\varrho^{-\varpi}$, the first-order law
gives $1-s_N\propto(m/k)^{\varpi/d_{\rm int}}$ in the first-order range,
so the slope of $\log(1-s_N)$ against $\log m$ at fixed $k$ estimates
$\varpi/d_{\rm int}$. At a fixed $k/m$, the radius $\varrho_k$ stays
fixed as $m$ grows, and the first-order law predicts a stable $s_N$ for
large $m$.

\subsection{Ordinal Agreement without Geometric Agreement}

For each $\beta>0$, the following proposition gives a strictly
increasing link of all distances that leaves every observable
$T\preceq T_E$ unchanged. Under these links, CKA tends to one as
$\beta\to0$ and to a value set by the participation ratio as
$\beta\to\infty$.

\begin{proposition}[Ordinal Agreement without Geometric Agreement]
\label{prop:bernstein}
Let $K$ be a unit-diagonal positive semidefinite matrix whose off-diagonal
entries are smaller than one, and for $\beta>0$ let
$K_\beta=\exp\circ(\beta(K-\ind\ind^\top))$ entrywise. Then $K_\beta$ is
positive semidefinite with unit diagonal, $T(K_\beta)=T(K)$ for every
observable $T\preceq T_E$, $\CKA(K,K_\beta)\to1$ as $\beta\to0$, and
$\CKA(K,K_\beta)\to(\PR(\bar K)/(m-1))^{1/2}$ as $\beta\to\infty$, where
$\PR(\bar K)=(\tr\bar K)^2/\norm{\bar K}_F^2$ is the participation ratio of
$\bar K$.
\end{proposition}

The participation ratio is the effective number of directions of
variation among the inputs. When this number is small relative to $m$,
two representations can agree on every neighbor set, ranking, and
spanning tree while their CKA is close to zero. The uniqueness results of
\citet{kleindessner2014uniqueness} exclude this case for dense samples in
a common low dimension. The matrix $K_\beta$ lies outside their
assumptions, since it is the Gaussian kernel of distinct unit vectors and
hence positive definite of rank $m$.

\subsection{Aggregation before Calibration}

We give the formal version of Corollary~\ref{cor:aggregation}, whose bounds hold whenever
the calibration threshold of the layer maximum lies within
a known distance of the largest random-pairing baseline over layer pairs.

\begin{formalof}{Corollary~\ref{cor:aggregation}}[Formal Version]
For two model pairs indexed $1$ and $2$, let $\bar s_1$ and $\bar s_2$ be
the maxima of their layer-pair scores, $\bar s^0_1$ and $\bar s^0_2$ the
maxima of the random-pairing baselines of these scores, and $c_1$ and
$c_2$ the calibration thresholds of $\bar s_1$ and $\bar s_2$, with
$|c_1-\bar s^0_1|\le\zeta_1$, $|c_2-\bar s^0_2|\le\zeta_2$,
$\bar s^0_1+\zeta_1<s_{\max}$, and $\bar s^0_2+\zeta_2<s_{\max}$. Let
$s_{{\rm cal},1}=[\bar s_1-c_1]_+/(s_{\max}-c_1)$ and
$s_{{\rm cal},2}=[\bar s_2-c_2]_+/(s_{\max}-c_2)$ be the calibrated
scores of the two maxima. Then
\[
 s_{{\rm lo},1}=\frac{[\bar s_1-\bar s^0_1-\zeta_1]_+}{s_{\max}-\bar s^0_1-\zeta_1}
 \le s_{{\rm cal},1}\le
 \frac{[\bar s_1-\bar s^0_1+\zeta_1]_+}{s_{\max}-\bar s^0_1+\zeta_1}=s_{{\rm hi},1},
\]
and $s_{{\rm lo},2}\le s_{{\rm cal},2}\le s_{{\rm hi},2}$ for the bounds
defined in the same way with index~$2$.
\textup{(1)}~If $\bar s^0_1=\bar s^0_2$ and
$\bar s_1+\zeta_1+\zeta_2<\bar s_2$, then $s_{{\rm cal},1}<s_{{\rm cal},2}$
or $s_{{\rm cal},1}=s_{{\rm cal},2}=0$.
\textup{(2)}~If $\bar s_1<\bar s_2$ and $s_{{\rm lo},1}>s_{{\rm hi},2}$,
then $s_{{\rm cal},1}>s_{{\rm cal},2}$, i.e., calibration reverses the
order of the two maxima.
\end{formalof}

The bound $s_{{\rm lo},1}$ is the calibrated score of $\bar s_1$ at the
threshold $\bar s^0_1+\zeta_1$, and $s_{{\rm hi},2}$ is that of $\bar s_2$
at the threshold $\bar s^0_2-\zeta_2$, as in part~2 of the informal
statement in Section~\ref{sec:calibration}. The proof below also shows that the condition
$s_{{\rm lo},1}>s_{{\rm hi},2}$ of part~2 holds if and only if
$\bar s_1>\bar s^0_1+\zeta_1$ and
$(s_{\max}-\bar s_1)(s_{\max}-\bar s^0_2+\zeta_2)
<(s_{\max}-\bar s_2)(s_{\max}-\bar s^0_1-\zeta_1)$, a form without
positive parts. Together with $\bar s_1<\bar s_2$, this condition implies
$\bar s^0_2-\bar s^0_1>\bar s_2-\bar s_1+\zeta_1+\zeta_2$ and
$c_2-c_1>\bar s_2-\bar s_1$, so it requires the largest baseline and the
threshold to grow by more than the maximum. In
Section~\ref{sec:exp-aggregation} we test whether the
calibration threshold of the layer maximum follows the largest baseline.

Proposition~\ref{prop:baselines} also bears on comparisons across
evaluation sets of different sizes, since at a fixed $k$ the baseline
$k/(m-1)$ of mKNN changes with $m$ while that of debiased CKA stays at
zero. In the fourteen scientific modality pairs of
\citet{schnaus2026rosetta}, the evaluation sets range in size from 123 to
10{,}000 paired samples, so at $k=10$ the baseline of mKNN ranges from
about $0.08$ to $0.001$. This shift can explain part of why mKNN predicts
the success of their alignment less well than debiased CKA on these
pairs. \citet{schnaus2026rosetta} also suggest that the different sizes
of these sets may explain this gap.

\subsection{Proofs}

\begin{proof}[Proof of Lemma~\ref{lem:rows}]
Conditionally on $x_i$, the pairs
$(D_A(i,l),D_B(i,l))=(D_A(x_i,x_l),D_B(x_i,x_l))$, $l\ne i$, are
independent and identically distributed as $(D_A(x_i,x'),D_B(x_i,x'))$.
Given also the first coordinates, the second coordinates stay
independent, each with the conditional law given its own first
coordinate.
\end{proof}

\begin{proof}[Proof of Lemma~\ref{lem:margins}]
Let $D'$ be the image of $D_A$ under a map in $\mathcal W_T$ with
$\norm{D'-D_B}<\rho_T(D_B)$. Then $T(D_A)=T(D')=T(D_B)$, since $D'$ lies
in a ball around $D_B$ on which $T$ is constant. If $T_1=q\circ T_2$ and
$T_2$ is constant on a ball, then so is $T_1$, which gives
$\rho_{T_1}\ge\rho_{T_2}$, and every map leaving $T_2$ unchanged leaves
$q\circ T_2$ unchanged, which gives
$\mathcal W_{T_2}\subseteq\mathcal W_{T_1}$.
\end{proof}

\begin{formalof}{Theorem~\ref{thm:crossing}}[Full Version]
For every anchor $i$ and threshold $\vartheta$, with $b_l=D_B(i,l)$ and
the counts $X(\vartheta)$ and $Y(\vartheta)$ of Section~\ref{sec:neighbors},
\begin{equation}
 \min\{X(\vartheta),Y(\vartheta)\}\le
 \bigl|\Nset_k^A(i)\setminus\Nset_k^B(i)\bigr|
 \le\max\{X(\vartheta),Y(\vartheta)\}.
 \label{eq:sandwich}
\end{equation}
Let $\mathcal G$ be a $\sigma$-field with respect to which the row
$(D_A(i,l))_{l\ne i}$ is measurable, and suppose that, conditionally on
$\mathcal G$, the entries $(b_l)_{l\ne i}$ are independent with continuous
distributions. Then there is a $\mathcal G$-measurable solution
$\vartheta^*$ of
$\E[X(\vartheta)\mid\mathcal G]=\E[Y(\vartheta)\mid\mathcal G]$. For any
such $\vartheta^*$, let
$\kappa_i=\E[X(\vartheta^*)\mid\mathcal G]=
\sum_{j\in\Nset_k^A(i)}\Prob(b_j>\vartheta^*\mid\mathcal G)$, a value that
almost surely does not depend on the choice of $\vartheta^*$. Then, almost
surely,
\begin{equation}
 \Bigl|\,\E\bigl[\,|\Nset_k^A(i)\setminus\Nset_k^B(i)|\bigm|\mathcal G\bigr]
 -\kappa_i\Bigr|\le\sqrt{\kappa_i/2},
 \label{eq:crossing}
\end{equation}
and $\bigl|\E|\Nset_k^A(i)\setminus\Nset_k^B(i)|-\E\kappa_i\bigr|
\le(\E\kappa_i/2)^{1/2}$. Without the independence assumption,
$\E|\Nset_k^A(i)\setminus\Nset_k^B(i)|\le
\inf_\vartheta\{\E X(\vartheta)+\E Y(\vartheta)\}$. The quantity
$\kappa_i$ is unchanged when one deterministic strictly increasing
function is applied to every entry of row $i$ of either $D_A$ or $D_B$.
\end{formalof}

\begin{proof}[Proof of Theorem~\ref{thm:crossing}]
Let $n_\le=|\{l\ne i:b_l\le\vartheta\}|=k-X(\vartheta)+Y(\vartheta)$. If
$n_\le\ge k$, the $k$ smallest entries of the row are at most
$\vartheta$, so $\Nset_k^B(i)\subseteq\{b\le\vartheta\}$. The neighbors of
model A with $b_j>\vartheta$ are then outside $\Nset_k^B(i)$, and the
inputs of $\Nset_k^B(i)\setminus\Nset_k^A(i)$ are among the $Y(\vartheta)$
non-neighbors below $\vartheta$, so
$X(\vartheta)\le|\Nset_k^A(i)\setminus\Nset_k^B(i)|=
|\Nset_k^B(i)\setminus\Nset_k^A(i)|\le Y(\vartheta)$. If $n_\le<k$, every
input with $b_l\le\vartheta$ belongs to $\Nset_k^B(i)$, so the
$Y(\vartheta)$ non-neighbors below $\vartheta$ enter $\Nset_k^B(i)$ and the
neighbors that leave are among the $X(\vartheta)$ neighbors above
$\vartheta$, which gives
$Y(\vartheta)\le|\Nset_k^A(i)\setminus\Nset_k^B(i)|\le X(\vartheta)$. This
proves Equation~\eqref{eq:sandwich}, which holds for any tie rule.

Conditionally on $\mathcal G$, the set $\Nset_k^A(i)$ is fixed. Computed
from the conditional laws of the entries $b_l$ given $\mathcal G$, the
functions $\E[X(\vartheta)\mid\mathcal G]$ and
$\E[Y(\vartheta)\mid\mathcal G]$ are almost surely continuous, the first
nonincreasing from $k$ to $0$ and the second nondecreasing from $0$ to
$m-1-k\ge1$. Hence
$\vartheta^*=\inf\{q\in\mathbb Q:\E[X(q)\mid\mathcal G]\le\E[Y(q)\mid\mathcal G]\}$
is finite and balances by continuity, and it is $\mathcal G$-measurable,
since $\{\vartheta^*<t\}$ is the union over rationals $q<t$ of
$\mathcal G$-measurable events. If $\vartheta_1<\vartheta_2$ both balance, then
$\E[X(\vartheta_1)\mid\mathcal G]\ge\E[X(\vartheta_2)\mid\mathcal G]
=\E[Y(\vartheta_2)\mid\mathcal G]\ge\E[Y(\vartheta_1)\mid\mathcal G]
=\E[X(\vartheta_1)\mid\mathcal G]$, so $\kappa_i$ does not depend on the
choice of $\vartheta^*$. By Equation~\eqref{eq:sandwich}, the count
$|\Nset_k^A(i)\setminus\Nset_k^B(i)|$ differs from $(X+Y)/2$ by at most
$|X-Y|/2$, and $\E[(X+Y)/2\mid\mathcal G]=\kappa_i$ at $\vartheta^*$.
Since $\vartheta^*$ and $\Nset_k^A(i)$ are $\mathcal G$-measurable,
conditionally on $\mathcal G$ the counts $X(\vartheta^*)$ and
$Y(\vartheta^*)$ are sums of
independent indicators over disjoint sets of inputs, so they are
independent with $\mathrm{Var}(X\mid\mathcal G)\le\E[X\mid\mathcal G]$
and $\mathrm{Var}(Y\mid\mathcal G)\le\E[Y\mid\mathcal G]$, and since their
conditional means are equal at $\vartheta^*$,
\[
 \E\bigl[|X-Y|\bigm|\mathcal G\bigr]
 \le\bigl(\mathrm{Var}(X\mid\mathcal G)+\mathrm{Var}(Y\mid\mathcal G)\bigr)^{1/2}
 \le(2\kappa_i)^{1/2},
\]
which gives Equation~\eqref{eq:crossing}. The unconditional bound
follows by taking expectations and applying Jensen's inequality to the
concave square root. Without independence, the sandwich gives
$|\Nset_k^A(i)\setminus\Nset_k^B(i)|\le\max\{X,Y\}\le X+Y$ for every
$\vartheta$. Finally, a deterministic strictly increasing function applied to every
entry of row $i$ of $D_B$ keeps these entries conditionally independent
with continuous laws and sends
the events $\{b_l>\vartheta\}$ to the events in which the transformed
entries exceed the image of $\vartheta$, so balancing thresholds
correspond and $\kappa_i$ is unchanged. A strictly increasing map of
row $i$ of $D_A$ leaves $\Nset_k^A(i)$ unchanged and therefore leaves
$\kappa_i$ unchanged.
\end{proof}

\begin{proof}[Proof of Corollary~\ref{cor:margin-form}]
For $\vartheta$ between $\phi(D_{A,i,(k)})$ and $\phi(D_{A,i,(k+1)})$,
every $j\in\Nset_k^A(i)$ has $\phi(D_A(i,j))\le\vartheta$ and
$\Prob(b_j>\vartheta)\le\Prob\bigl(R_{ij}\ge\vartheta-\phi(D_A(i,j))\bigr)$,
and every non-neighbor $l$ has $\phi(D_A(i,l))>\vartheta$ and, by
symmetry,
$\Prob(b_l\le\vartheta)=\Prob\bigl(R_{il}\ge\phi(D_A(i,l))-\vartheta\bigr)$.
The bound without independence in the full version of
Theorem~\ref{thm:crossing} then gives
the claim. If $\phi$ is the identity and $|R_{il}|\le\mathrm{err}_i$,
every term with $|D_A(i,l)-\vartheta|>\mathrm{err}_i$ vanishes, and
$\vartheta$ at the midpoint of the neighbor gap makes the sum zero when
$\gamma^A_{i,k}>2\,\mathrm{err}_i$, so that $s_N(i)=1$ almost surely.
With $1-s_N(i)\le1$ at the other anchors, the average over anchors gives
Equation~\eqref{eq:neighbor-stability}.
\end{proof}

\begin{formalof}{Proposition~\ref{prop:mst}}[Restated]
Suppose that each model $a$ has pairwise distinct distances $D_a(i,j)$
with $i<j$.
For each model $a$, let $\NN_a(i)$ be the nearest neighbor of input $i$,
let $F_a=\{\{i,\NN_a(i)\}:i\in[m]\}$ be the nearest-neighbor forest, and
let $n^{\rm mut}_a$ be the number of pairs of inputs that are each other's
nearest neighbor. Then $F_a\subseteq\Tset_a$ and $|F_a|=m-n^{\rm mut}_a$,
so $\Tset_a\setminus F_a$ consists of $n^{\rm mut}_a-1$ bridges, i.e.,
edges joining two trees of $F_a$, and
\[
 s_{\Tset}\ge\frac{\max\{m\,s_{N,1}-n^{\rm mut}_A,\;m\,s_{N,1}/2\}}{m-1},
\]
where $s_{N,1}$ is the one-nearest-neighbor overlap of the two models.
\end{formalof}

\begin{proof}[Proof of Proposition~\ref{prop:mst}]
Part~2 of the proof of Proposition~\ref{prop:lattice} shows the
inclusion $F_a\subseteq\Tset_a$. In the directed graph $i\mapsto\NN_a(i)$,
distances strictly decrease along every path until a pair of mutual
nearest neighbors is reached, so every cycle has length two, and the $m$
directed edges give $m-n^{\rm mut}_a$ undirected edges. Since
$F_a\subseteq\Tset_a$ and $\Tset_a$ is acyclic, every edge of
$\Tset_a\setminus F_a$ joins two trees of $F_a$, and there are
$m-1-|F_a|=n^{\rm mut}_a-1$ such edges. Let
$I=\{i:\NN_A(i)=\NN_B(i)\}$, so $|I|=m\,s_{N,1}$. Each $i\in I$ gives an
edge $\{i,\NN_A(i)\}\in F_A\cap F_B\subseteq\Tset_A\cap\Tset_B$, and two
inputs give the same edge exactly when they are mutual nearest neighbors
in model A. The number of distinct shared edges is therefore at least
$|I|/2$ and at least $|I|-n^{\rm mut}_A$.
\end{proof}

\begin{formalof}{Proposition~\ref{prop:cka}}[Restated]
For each $m$, let $D_A$ be a fixed nonzero hollow symmetric matrix with
nonnegative entries,
$\phi:[0,\infty)\to\R$ increasing, and $R$ a random symmetric hollow
matrix whose entries $R_{ij}$, $i<j$, are independent with mean zero,
variance $\sigma^2$, and fourth moments bounded uniformly in $m$. Let
$D_B=\phi(D_A)+R$, so that $R$ is the residual around $\phi$, and let
$\varsigma^2=\norm{H\phi(D_A)H}_F^2/(m(m-1))$. If $m\to\infty$ while
$\sigma^2$ and $\varsigma^2$ stay in a compact subset of $(0,\infty)$,
then
\[
 \CKA(D_A,D_B)-\frac{\CKA(D_A,\phi(D_A))}{\sqrt{1+\sigma^2/\varsigma^2}}
 \;\to\;0
 \quad\text{in probability.}
\]
\end{formalof}

\begin{proof}[Proof of Proposition~\ref{prop:cka}]
Write $\bar D_A=HD_AH$, $\bar D_\phi=H\phi(D_A)H$, and $\bar R=HRH$, so that
$\CKA(D_A,D_B)=\langle\bar D_A,\bar D_\phi+\bar R\rangle_F/
(\norm{\bar D_A}_F\norm{\bar D_\phi+\bar R}_F)$. Since $\bar D_A$ is
centered, $\langle\bar D_A,\bar R\rangle_F=\langle\bar D_A,R\rangle_F
=2\sum_{i<j}(\bar D_A)_{ij}R_{ij}$ has mean zero and variance at most
$2\sigma^2\norm{\bar D_A}_F^2$, and likewise
$\mathrm{Var}\langle\bar D_\phi,\bar R\rangle_F\le2\sigma^2\norm{\bar D_\phi}_F^2$.
A direct computation gives
$\E\norm{\bar R}_F^2=m(m-1)\sigma^2(1-2/m+2/m^2)$, and the uniformly
bounded fourth moments give $\mathrm{Var}\norm{\bar R}_F^2=O(m^2)$.
After division of the numerator by $\norm{\bar D_A}_F\norm{\bar D_\phi}_F$
and of the squared norm $\norm{\bar D_\phi+\bar R}_F^2$ by
$\norm{\bar D_\phi}_F^2=m(m-1)\varsigma^2$, Chebyshev's inequality yields
\[
 \CKA(D_A,D_B)=\frac{\CKA(D_A,\phi(D_A))+O_P(1/m)}
 {\bigl(1+\sigma^2/\varsigma^2+O_P(1/m)\bigr)^{1/2}},
\]
which converges as stated.
\end{proof}

\begin{proof}[Proof of Proposition~\ref{prop:bernstein}]
The matrix $\exp\circ(\beta K)=\sum_{n\ge0}\beta^nK^{\circ n}/n!$ is a
convergent sum of Hadamard powers of a positive semidefinite matrix, each
positive semidefinite by the Schur product theorem, and multiplication by
$e^{-\beta}$ preserves this property and gives unit diagonal. The squared
distances of $K_\beta$ are $D(K_\beta)=2-2e^{-\beta D(K)/2}$, a strictly
increasing function of $D(K)$ applied entrywise, so the order of
distances and every observable below $T_E$ are unchanged. As $\beta\to0$,
$HK_\beta H=\beta HKH+O(\beta^2)$, so $\CKA(K,K_\beta)\to1$. As
$\beta\to\infty$, $K_\beta\to I$ because the off-diagonal entries of $K$
are smaller than one, and
$\CKA(K,I)=\langle\bar K,H\rangle_F/(\norm{\bar K}_F\norm H_F)=
\tr\bar K/(\norm{\bar K}_F\sqrt{m-1})$, since $\bar KH=\bar K$ and
$\norm H_F^2=m-1$.
\end{proof}

\begin{formalof}{Proposition~\ref{prop:baselines}}[Restated]
Let $\pi$ be a uniformly random permutation of the inputs of model B,
write $K_B^\pi$ and $\Nset_k^{B,\pi}$ for the Gram matrix and the
neighbor sets of model B after relabeling, and let $\E_\pi$ denote the
expectation over $\pi$ and over any random tie-breaking.
\textup{(1)}~If the off-diagonal entries of each row of $D_B$ are
distinct, or ties are broken uniformly at random independently of
$\pi$, then for every anchor
$i$ the overlap $|\Nset_k^A(i)\cap\Nset_k^{B,\pi}(i)|$ follows the
hypergeometric distribution with population $m-1$, $k$ successes, and
$k$ draws, and $\E_\pi s_N=k/(m-1)$.
\textup{(2)}~For nonzero centered Gram matrices,
$\E_\pi\CKA(K_A,K_B^\pi)=(\PR(\bar K_A)\,\PR(\bar K_B))^{1/2}/(m-1)$.
\textup{(3)}~With the unbiased estimator $\mathrm{HSIC}_{\rm u}$ of the
Hilbert--Schmidt independence criterion \citep{song2012feature}, let
debiased CKA be $\CKA_{\rm deb}(K_A,K_B)=\mathrm{HSIC}_{\rm u}(K_A,K_B)/
(\mathrm{HSIC}_{\rm u}(K_A,K_A)\,\mathrm{HSIC}_{\rm u}(K_B,K_B))^{1/2}$.
If $m\ge4$ and
$\mathrm{HSIC}_{\rm u}(K_A,K_A)$ and $\mathrm{HSIC}_{\rm u}(K_B,K_B)$ are
positive, then $\E_\pi\CKA_{\rm deb}(K_A,K_B^\pi)=0$ and
$\E_\pi\CKA_{\rm deb}(K_A,K_B^\pi)^2=2/(m(m-3))$.
\end{formalof}

\begin{proof}[Proof of Proposition~\ref{prop:baselines}]
Part~1. Write $j=\pi^{-1}(i)$. The permuted neighbor set of $i$ is
$\pi(S)$ for the set $S$ of the $k$ inputs nearest to $j$ under $D_B$.
If row $j$ of $D_B$ has distinct off-diagonal entries, then
$S=\Nset_k^B(j)$. If ties are broken uniformly at random independently of
$\pi$, the tie-breaking order carried back to the original labels is
uniform and independent of $\pi$, so $S$ is independent of $\pi$ given
$j$. Conditionally on $j$, the restriction of $\pi$ to the inputs other
than $j$ is a uniformly random bijection onto the inputs other than $i$.
Hence $\pi(S)$ is a uniformly random $k$-subset of
$[m]\setminus\{i\}$, and its overlap with the fixed $k$-subset
$\Nset_k^A(i)$ is hypergeometric with mean $k^2/(m-1)$.

Part~2. Since $H$ commutes with every permutation of the inputs, the
permuted centered matrix $\bar K_B^\pi=HK_B^\pi H$ has entries
$(\bar K_B)_{\pi^{-1}(i)\pi^{-1}(j)}$. Its diagonal entries have
expectation $\tr\bar K_B/m$, and its off-diagonal entries have expectation
$(\ind^\top\bar K_B\ind-\tr\bar K_B)/(m(m-1))=-\tr\bar K_B/(m(m-1))$.
Since $\ind^\top\bar K_A\ind=0$, the off-diagonal entries of $\bar K_A$
sum to $-\tr\bar K_A$, and
\[
 \E_\pi\langle\bar K_A,\bar K_B^\pi\rangle_F
 =\frac{\tr\bar K_A\tr\bar K_B}m+\frac{\tr\bar K_A\tr\bar K_B}{m(m-1)}
 =\frac{\tr\bar K_A\tr\bar K_B}{m-1}.
\]
The norms are invariant under permutation, and
$\tr\bar K/\norm{\bar K}_F=\PR(\bar K)^{1/2}$.

Part~3. For a symmetric $m\times m$ matrix $K$, let $K^{\rm u}$ be its
U-centered matrix \citep{szekely2014partial}, which has zero diagonal and
off-diagonal entries
\[
 K^{\rm u}_{ij}=K_{ij}-\frac1{m-2}\sum_{l\ne i}K_{il}
 -\frac1{m-2}\sum_{l\ne j}K_{lj}
 +\frac1{(m-1)(m-2)}\sum_{l\ne l'}K_{ll'},\qquad i\ne j.
\]
The estimator of \citet{song2012feature} expands to
$\mathrm{HSIC}_{\rm u}(K,K')=\langle K^{\rm u},{K'}^{\rm u}\rangle_F/(m(m-3))$,
so $\CKA_{\rm deb}(K_A,K_B)=\langle K_A^{\rm u},K_B^{\rm u}\rangle_F/
(\norm{K_A^{\rm u}}_F\norm{K_B^{\rm u}}_F)$. For $m\ge4$, the two
positivity hypotheses hold if and only if $K_A^{\rm u}$ and $K_B^{\rm u}$
are nonzero. U-centering commutes with relabeling, so the U-centered
matrix of $K_B^\pi$ is $(K_B^{\rm u})^\pi$, and the denominator of
$\CKA_{\rm deb}(K_A,K_B^\pi)$ does not depend on $\pi$. Every U-centered
matrix is symmetric with zero diagonal and zero row sums. The average of
$(K_B^{\rm u})^\pi$ over all permutations is unchanged by every
relabeling and is therefore a linear combination of $I$ and
$\ind\ind^\top$. Zero diagonal and zero row sums make this combination
zero, so $\E_\pi\CKA_{\rm deb}(K_A,K_B^\pi)=0$. For the second moment, the
symmetric matrices with zero diagonal form the permutation module of the
2-subsets of the inputs. For $m\ge4$, this module is the direct sum of one
copy of each irreducible representation of the symmetric group indexed by
the partitions $(m)$, $(m-1,1)$, and $(m-2,2)$
\citep{james1978representation}. The row-sum map from these matrices to
$\R^m$ commutes with relabeling and is surjective, and $\R^m$ carries the
representations indexed by $(m)$ and $(m-1,1)$, so the kernel of this map
is the representation indexed by $(m-2,2)$. This kernel consists of the
matrices with zero row sums and has dimension $m(m-3)/2$. Schur
orthogonality for this absolutely irreducible representation gives
$\E_\pi\langle K_A^{\rm u},(K_B^{\rm u})^\pi\rangle_F^2
=2\norm{K_A^{\rm u}}_F^2\norm{K_B^{\rm u}}_F^2/(m(m-3))$, which is the
second moment after division by the two squared norms. With the small
positive constant added to the denominator of $\CKA_{\rm deb}$ by the code
of ARH, the mean stays zero and the second moment falls below
$2/(m(m-3))$.
\end{proof}

\begin{proof}[Proof of Corollary~\ref{cor:aggregation}]
For $\bar s\le s_{\max}$, the map $c\mapsto[\bar s-c]_+/(s_{\max}-c)$ is
nonincreasing on $c<s_{\max}$, since the map is zero for $c\ge\bar s$
and has derivative $(\bar s-s_{\max})/(s_{\max}-c)^2\le0$ for
$c<\bar s$. Since $\bar s_1$ and $\bar s_2$ are at most the largest
possible value $s_{\max}$, the bounds follow by evaluating this map at
$c=\bar s^0_1\pm\zeta_1$ and at $c=\bar s^0_2\pm\zeta_2$, all below
$s_{\max}$. If $\bar s^0_1=\bar s^0_2$, write $\bar s^0$ for the common
value, so that $\bar s_1-\bar s^0+\zeta_1<\bar s_2-\bar s^0-\zeta_2$. If
$\bar s_2-\bar s^0-\zeta_2\le0$, then $\bar s_1-\bar s^0+\zeta_1<0$, so
$s_{{\rm hi},1}=0$ and $s_{{\rm cal},1}=0\le s_{{\rm cal},2}$. Otherwise
$s_{{\rm cal},1}\le s_{{\rm hi},1}
<(\bar s_2-\bar s^0-\zeta_2)/(s_{\max}-\bar s^0-\zeta_2)=s_{{\rm lo},2}
\le s_{{\rm cal},2}$, since the numerator of $s_{{\rm hi},1}$ is zero or
smaller than the positive number $\bar s_2-\bar s^0-\zeta_2$ and its
denominator $s_{\max}-\bar s^0+\zeta_1$ is at least
$s_{\max}-\bar s^0-\zeta_2>0$. In both cases
$s_{{\rm cal},1}<s_{{\rm cal},2}$ or $s_{{\rm cal},1}=s_{{\rm cal},2}=0$,
which is part~1. Part~2 follows from
$s_{{\rm cal},1}\ge s_{{\rm lo},1}>s_{{\rm hi},2}\ge s_{{\rm cal},2}$.

For the equivalent form of the condition of part~2, note that
$s_{{\rm hi},2}\ge0$, so $s_{{\rm lo},1}>s_{{\rm hi},2}$ requires
$s_{{\rm lo},1}>0$, i.e., $\bar s_1>\bar s^0_1+\zeta_1$. Under this
inequality,
$1-s_{{\rm lo},1}=(s_{\max}-\bar s_1)/(s_{\max}-\bar s^0_1-\zeta_1)$. If
$\bar s_2>\bar s^0_2-\zeta_2$, then
$1-s_{{\rm hi},2}=(s_{\max}-\bar s_2)/(s_{\max}-\bar s^0_2+\zeta_2)$, and
$1-s_{{\rm lo},1}<1-s_{{\rm hi},2}$ becomes the product inequality after
multiplication by the two positive denominators. Otherwise
$s_{{\rm hi},2}=0<s_{{\rm lo},1}$, and the product inequality also holds,
since $s_{\max}-\bar s_1<s_{\max}-\bar s^0_1-\zeta_1$ and
$s_{\max}-\bar s_2\ge s_{\max}-\bar s^0_2+\zeta_2>0$.

Finally, let $\bar s_1<\bar s_2$ and let the condition of part~2 hold.
Then $s_{\max}-\bar s^0_1-\zeta_1>s_{\max}-\bar s_1>s_{\max}-\bar s_2\ge0$,
and dividing the product inequality by $s_{\max}-\bar s_1$ gives
\[
 \bar s^0_2-\bar s^0_1-\zeta_1-\zeta_2
 >(s_{\max}-\bar s^0_1-\zeta_1)\frac{\bar s_2-\bar s_1}{s_{\max}-\bar s_1}
 >\bar s_2-\bar s_1.
\]
Since $c_2-c_1\ge\bar s^0_2-\bar s^0_1-\zeta_1-\zeta_2$, also
$c_2-c_1>\bar s_2-\bar s_1$.
\end{proof}

\section{Learning Conditions: Full Statements and Proofs}
\label{app:learning-full}

This appendix states the finite Gibbs model behind
Section~\ref{sec:objective} and proves the sample-size and model-class
results of Table~\ref{tab:learning}. It then gives the formal versions
of Theorems~\ref{thm:gauge} and~\ref{thm:link}, the full versions of
Corollary~\ref{cor:gauge-neighbors} and Remark~\ref{rem:temperature},
and the proofs of all four results.

Suppose that the Gram matrices of two learners converge in probability to
$K_A^\infty$ and $K_B^\infty$ with growing samples. Then $p_Q$ tends to one
when $(K_A^\infty,K_B^\infty)$ lies in the interior of $\Aset_Q$ and to
zero when the pair lies outside its closure. A criterion can thus stay
obstructed if the limits differ by more than its tolerance or lie on the
boundary of $\Aset_Q$. A criterion can also be obstructed when the loss
has several global minimizers, since the limits then depend on the
training randomness. Table~\ref{tab:learning} pairs these cases with four
learning conditions. In each row, the learning condition makes agreement
reliable under some criteria, and the outcome of each obstructed criterion
depends on the difference changed by this condition and ignored by some
weaker criterion.

\begin{table}[ht]
\centering\scriptsize
\setlength{\tabcolsep}{2.5pt}
\renewcommand{\arraystretch}{1.1}
\begin{tabular}{@{}>{\raggedright\arraybackslash}p{.10\linewidth}>{\raggedright\arraybackslash}p{.14\linewidth}>{\raggedright\arraybackslash}p{.24\linewidth}>{\raggedright\arraybackslash}p{.22\linewidth}>{\raggedright\arraybackslash}p{.10\linewidth}>{\raggedright\arraybackslash}p{.13\linewidth}@{}}
\toprule
Learning condition & Difference changed & Criteria met reliably & Remaining obstruction & Result & Trained networks \\
\midrule
Training procedure & choice among equivalent minimizers (gauge) & geometry under weight decay or tied weights, approximately under small initialization and gradient flow; the projector for every minimizer as the samples grow; on one sample, neighbor sets of every probe whose ratio $D_{A,i,(k+1)}/D_{A,i,(k)}$ exceeds $\mathrm{cond}_{AB}$ (Cor.~\ref{cor:gauge-neighbors}) & CKA of small tolerance under gradient flow from a generic initialization & Thm.~\ref{thm:gauge}, Cor.~\ref{cor:gauge-neighbors} & weight decay removes the excess of the discrepancy of the head coordinates over that of the outputs in scale-normalized geometry (Sec.~\ref{sec:exp-gauge}) \\
Objective & monotone link & exact criteria on every $T\preceq T_E$ & scale-normalized geometry at small tolerance, at all large $N$ & Thm.~\ref{thm:link} & preregistered prediction failed in networks trained on CIFAR-5m \citep{nakkiran2021deep}, which disagree even on neighbor sets (Sec.~\ref{sec:exp-link}) \\
Sample size & estimation error & coarser criteria first, at a size set by the margin; criteria on distant pairs at sizes growing at least exponentially in their squared distance & ties in the target & Lem.~\ref{lem:cooccurrence}, Thm.~\ref{thm:peer} & prediction derived from Lem.~\ref{lem:cooccurrence} failed (App.~\ref{app:exp-s21}) \\
Model class & approximation floor & tolerance above the discrepancy of the limits, which is at least the floor & geometry against a richer model & Prop.~\ref{prop:rank} & untested \\
\bottomrule
\end{tabular}
\caption{Learning conditions, the difference between two
learned representations changed by each, and the resulting transitions
and obstructions. In the models of Theorems~\ref{thm:gauge}
and~\ref{thm:link}, a change to the projector criterion avoids the
obstruction of row 1, and a change to the criteria with the isotonic
discrepancy avoids that of row 2, while the representations stay the
same. The last column reports the results measured on trained networks,
which lie outside these models. Appendices~\ref{sec:data}
and~\ref{sec:capacity} prove the results of rows 3 and 4.}
\label{tab:learning}
\end{table}

\subsection{The Finite Gibbs Model}
\label{app:gibbs}

To obtain results about sample size and objectives, we use a model in
which independently sampled co-occurrences reveal a shared statistical
relation. On a finite domain $\Xcal$ with positive masses $P(x)$, let a
unit-diagonal positive definite matrix $K_\star$ on $\Xcal$ determine the
distribution of positive partners,
\[
 P_\star(x'\mid x)
 =\frac{P(x')\exp(K_\star(x,x')/t)}{Z(x)},\qquad
 Z(x)=\sum_{x''\in\Xcal}P(x'')\exp(K_\star(x,x'')/t),
 \qquad t>0.
\]
Each training example contains an anchor $x$ drawn from $P$, a positive
partner drawn from $P_\star(\cdot\mid x)$, and $J\ge1$ independent
negatives drawn from $P$, so positive pairs follow the law
$P_+(x,x')=P(x)P_\star(x'\mid x)$. Each model receives independent
examples through its own injective observation map. A learner trained
with the softmax contrastive (InfoNCE) loss at temperature $t$
\citep{oord2018representation} distinguishes the positive partner from
the negatives with a critic, a unit-diagonal positive semidefinite matrix
$K$ on $\Xcal$. Learner B of Theorem~\ref{thm:link} is instead
trained with the spectral contrastive loss of \citet{haochen2021provable}.
This finite Gibbs model is a special case of the
settings of \citet{zimmermann2021contrastive} and \citet{lyu2022finite},
and PRH's idealized contrastive learner \citep{huh2024platonic} is its
population limit. The masses are \emph{balanced} if $Z(x)$ takes one value
$Z$ for every $x$, and $P_+$ is then symmetric.

When both learned Gram matrices have unit diagonal and recover the
restriction of $K_\star$ to the probes $S\subseteq\Xcal$ within
$e^{\max}_A$ and $e^{\max}_B$ in the maximum norm,
\begin{equation}
 \norm{D_A-D_B}_{\max}\le2(e^{\max}_A+e^{\max}_B),
 \label{eq:recovery-route}
\end{equation}
so recovery bounds transfer to agreement. Two inaccurate models can
nevertheless agree on the same wrong relation, so agreement and recovery
remain different questions.

\subsection{Sample Size: Co-occurrence and Fixed-target Transitions}
\label{sec:data}

To show that agreement on the distances of rarely co-occurring pairs
needs many examples, we bound the Fisher information that one training
example carries about the entry $K(x,x')$ of a pair.

\begin{lemma}[Co-occurrence Information]
\label{lem:cooccurrence}
For distinct $x,x'\in\Xcal$, the Fisher information about the entry
$K(x,x')$ carried by one training example is at most
$P^{s}_+(x,x')/t^2$, where
\[
 P^{s}_+(x,x')=P(x)P_\star(x'\mid x)+P(x')P_\star(x\mid x')
 =P(x)P(x')\,e^{K_\star(x,x')/t}\Bigl(\frac1{Z(x)}+\frac1{Z(x')}\Bigr).
\]
Every estimator that is regular at $K_\star$ therefore has asymptotic
variance of $\sqrt N(\hat K(x,x')-K_\star(x,x'))$ at least
$t^2/P^{s}_+(x,x')$.
\end{lemma}

Since $Z(x)\ge P(x)e^{1/t}$ by the unit diagonal of $K_\star$, the
information about a pair is at most
$(P(x)+P(x'))e^{-D_\star(x,x')/(2t)}/t^2$, which decays exponentially in
$D_\star(x,x')/(2t)$ for the squared distance
$D_\star(x,x')=2-2K_\star(x,x')$. Reliable agreement within
$\varepsilon$ on the distances of distant pairs therefore needs a number
of examples of order at least $t^2/(\varepsilon^2P^s_+(x,x'))$ for the
least co-occurring pair, for instance under a tolerance $\varepsilon$ on
the maximum distance error. This delay of agreement on distant pairs
arises within one objective and one model class, and more data removes
the delay.
Theorem~\ref{thm:peer} gives the sample size at which two independent
runs come to agree on a discrete observable.

\begin{theorem}[Fixed-target Transitions between Independent Runs]
\label{thm:peer}
In the finite Gibbs model, let both learners use temperature $t$ and
return minimizers of the empirical InfoNCE risk over unit-diagonal
positive semidefinite matrices on $\Xcal$, with independent samples of
sizes $N$ and $\nu N$. Let $\theta$ collect the entries $K(x,x')$ over
unordered pairs $x\ne x'$, let $\theta_\star$ collect those of $K_\star$,
and let $I_J$ be the Hessian of the population risk at $\theta_\star$.
\begin{enumerate}
\item $\sqrt N(\hat\theta_N-\theta_\star)$ converges in distribution to
the centered Gaussian law with covariance $I_J^{-1}$.
\item Let $T$, viewed as a function of $\theta$, be constant on the open
cells of a finite arrangement of hyperplanes. If $\theta_\star$ lies in
an open cell, then $p_{(T,=)}\to1$. If $\theta_\star$ lies on exactly one
hyperplane of the arrangement and $T$ takes different values on its two
sides near $\theta_\star$, then $p_{(T,=)}\to1/2$.
\item If instead $\theta_\star=\theta_0+\mathbf v/\sqrt N$ for a fixed
$\mathbf v$ and a point $\theta_0$ on exactly one hyperplane of the
arrangement, with normal vector $\mathbf n$, near which $T$ takes
different values on its two sides, then
$p_{(T,=)}\to\Phi(z)\Phi(\sqrt\nu z)+\Phi(-z)\Phi(-\sqrt\nu z)$ with
$z=\mathbf n^\top\mathbf v/(\mathbf n^\top I_J^{-1}\mathbf n)^{1/2}$ and
$\Phi$ the standard normal distribution function.
\end{enumerate}
\end{theorem}

With equal sample sizes, part~3 gives the sample size at which exact
agreement on an observable becomes reliable. Let $\rho^{I_J}_T$ be the
smallest
Fisher-normalized margin
$|\mathbf n^\top(\theta_\star-\theta_{\mathbf n})|/(\mathbf n^\top I_J^{-1}\mathbf
n)^{1/2}$ over the hyperplanes $\theta_{\mathbf n}+\mathbf n^\perp$ that separate
cells with different values of $T$, which is the margin of $T$ at
$\theta_\star$ in the metric of $I_J$. Agreement on $T$ then becomes
reliable at a level $1-\delta$ with $0<\delta<1/2$ near
$N^*\approx z_\delta^2/(\rho^{I_J}_T)^2$, where $z_\delta$ solves
$\Phi(z)^2+\Phi(-z)^2=1-\delta$. Since $p_{Q_1}\ge p_{Q_2}$ whenever
$Q_1\preceq Q_2$, agreement under exact criteria becomes reliable in the
order of the lattice for every pair of learning conditions, and part~3
sets the sample size of each transition
(Definition~\ref{def:transition}). A coarser observable
changes value across a subset of the hyperplanes of a finer observable,
so its margin $\rho^{I_J}_T$ is at least as large. A tie in the target
keeps agreement on $T$ from becoming reliable, since part~2 gives
$p_{(T,=)}\to1/2$ when $\theta_\star$ lies on exactly one hyperplane of
the arrangement, across which $T$ changes value. The exact criterion
$(T,=)$ is then asymptotically obstructed in the sense of
Definition~\ref{def:reliable} whenever $\delta+\eta\le1/2$.

\begin{proof}[Proof of Lemma~\ref{lem:cooccurrence}]
The law of the negatives does not depend on $K$. For an example whose
anchor is $x$, the log-likelihood of its positive partner $x_+$ is
$\log P(x_+)+K(x,x_+)/t-\log Z_K(x)$, where $Z_K$ and $P_K$ denote the
normalizing constant and the partner law under $K$, and its derivative
with respect to the symmetric entry $K(x,x')$ equals
$(\indic\{x_+=x'\}-P_K(x'\mid x))/t$. An example anchored at $x'$
contributes $(\indic\{x_+=x\}-P_K(x\mid x'))/t$, and every other example
contributes zero. The Fisher information of one example is the expected
square of this score at $K_\star$,
\[
 \frac1{t^2}\Bigl(P(x)P_\star(x'\mid x)\bigl(1-P_\star(x'\mid x)\bigr)
 +P(x')P_\star(x\mid x')\bigl(1-P_\star(x\mid x')\bigr)\Bigr)
 \le\frac{P^{s}_+(x,x')}{t^2},
\]
since a Bernoulli variance is at most its mean. The InfoNCE loss is the
likelihood of the index of the positive given the set of candidates, a
statistic of the example, so its information is at most that of the full
example. The full-example information matrix dominates $I_J$, which is
positive definite by the proof of Theorem~\ref{thm:peer}, so it is
nonsingular. For a regular estimator, the asymptotic variance of a
coordinate is at least the corresponding diagonal entry of the inverse
of this matrix, which is at least the inverse of its diagonal entry.
\end{proof}

\begin{proof}[Proof of Theorem~\ref{thm:peer}]
The empirical InfoNCE risk is a sum of log-sum-exp functions of linear
forms in $K$ minus a linear term, hence convex in $\theta$, and the set
of unit-diagonal positive semidefinite matrices is convex. Since InfoNCE
is the correctly specified conditional likelihood of the index of the
positive among the $J+1$ candidates, the score covariance of the
population risk at $\theta_\star$ equals its Hessian $I_J$. A direction
in the coordinates of $\theta$ is a symmetric function on
$\Xcal\times\Xcal$ with zero diagonal. The quadratic form of $I_J$ at
this direction equals $t^{-2}$ times the expected variance of the
function at the anchor and a candidate drawn from the posterior of the
positive's index given the example. Every pair of elements of $\Xcal$
occurs among the candidates with positive probability when $J\ge1$, so
this variance vanishes only if the function is constant in its second
argument for every anchor, and the zero diagonal then makes the function
zero. Hence $I_J$ is positive definite. The convexity lemma of
\citet{hjort1993asymptotics} gives part~1 for the unconstrained
minimizer, and the positive semidefiniteness constraint is inactive with
probability tending to one because $K_\star$ is positive definite.

For part~2, the two estimates are independent and asymptotically Gaussian
around $\theta_\star$. If $\theta_\star$ lies in an open cell, both
estimates fall in that cell with probability tending to one. If
$\theta_\star$ lies on one hyperplane, the two estimates fall
independently on either side, each with probability tending to $1/2$,
and the two values of $T$ are equal exactly when the estimates fall on
the same side. For part~3, the model is a smooth
finite-dimensional family, so the local alternatives
$\theta_0+\mathbf v/\sqrt N$ are contiguous, and by Le Cam's third lemma
$\sqrt N\,\mathbf n^\top(\hat\theta-\theta_0)$ converges to the Gaussian
law with mean $\mathbf n^\top\mathbf v$ and variance
$\mathbf n^\top I_J^{-1}\mathbf n$ for learner A and variance
$\mathbf n^\top I_J^{-1}\mathbf n/\nu$ for learner B. Agreement occurs
when the two estimates fall on the same side, which gives the stated
limit.
\end{proof}

\subsection{Model Class: Rank}
\label{sec:capacity}

A model class unable to express the target leaves an approximation
floor, and the simplest floor comes from the rank of the representation.

\begin{proposition}[Rank Obstruction]
\label{prop:rank}
Let $\lambda_1(\bar K_\star)\ge\dots\ge\lambda_m(\bar K_\star)$ be the
eigenvalues of $\bar K_\star=HK_\star H$ and
$\chi_{n_h}(S)=(\sum_{j>n_h}\lambda_j(\bar K_\star)^2)^{1/2}/\sqrt{m(m-1)}$.
If $\mathrm{rank}(HK_BH)\le n_h$ and $K_A$, $K_B$, $K_\star$ have unit
diagonal, then
$\norm{D_A-D_B}_{\max}\ge2\bigl[\chi_{n_h}(S)-e^{\rm rms}_A\bigr]_+$,
where $e^{\rm rms}_A$ is the root mean square of the off-diagonal entries
of $K_A-K_\star$.
\end{proposition}

A model whose representation has dimension $n_h$ therefore keeps a
distance error of at least $2\chi_{n_h}(S)$ against an accurate model, for
every sample size, budget, and optimizer. Two models restricted in the
same way agree with each other on the same biased relation if their
population optimum is unique and their empirical risk minimizers are
consistent. The rank obstruction therefore concerns comparisons with a
richer model or with the target.

\begin{proof}[Proof of Proposition~\ref{prop:rank}]
Write $\bar K=HKH$. By the Eckart--Young theorem
\citep{eckart1936approximation}, every centered matrix of rank at most
$n_h$ differs from $\bar K_\star$ by at least
$(\sum_{j>n_h}\lambda_j(\bar K_\star)^2)^{1/2}$ in Frobenius norm, so
$\norm{\bar K_\star-\bar K_B}_F\ge\sqrt{m(m-1)}\,\chi_{n_h}(S)$. The
triangle inequality bounds the left side by
$\norm{\bar K_\star-\bar K_A}_F+\norm{\bar K_A-\bar K_B}_F$. The first
term is at most $\norm{K_\star-K_A}_F=\sqrt{m(m-1)}\,e^{\rm rms}_A$
because the diagonals agree, and the second equals
$\tfrac12\norm{H(D_A-D_B)H}_F\le\tfrac12\sqrt{m(m-1)}\,\norm{D_A-D_B}_{\max}$
because $D_A-D_B$ has zero diagonal. Rearranging gives the bound.
\end{proof}

\subsection{The Gauge of a Linear Head}
\label{app:learning-gauge}

\begin{formalof}{Theorem~\ref{thm:gauge}}[Formal Version]
Let the training inputs $x_1,\dots,x_N\in\R^{n_x}$ have positive definite
second-moment matrix $\hat\Sigma=N^{-1}\sum_ix_ix_i^\top$ with eigenvalues
$\hat\lambda_1\ge\dots\ge\hat\lambda_{n_x}$ and
$\hat\lambda_{n_h}>\hat\lambda_{n_h+1}$, let
$\hat U\in\R^{n_x\times n_h}$ have orthonormal columns with
$\hat\Sigma\hat U=\hat U\,\mathrm{diag}(\hat\lambda_1,\dots,\hat\lambda_{n_h})$, and
consider $\hat x=WEx$ with $E\in\R^{n_h\times n_x}$,
$W\in\R^{n_x\times n_h}$, $n_h\ge2$, the loss
$\mathcal L(E,W)=\tr((I-WE)\hat\Sigma(I-WE)^\top)$, and the probe Gram
matrix $K=\mathbf XE^\top E\mathbf X^\top$ for a probe matrix
$\mathbf X\in\R^{m\times n_x}$ with $H\mathbf X\hat U$ of full column
rank. The features $Ex$ are not normalized here, an exception to the
convention of Section~\ref{sec:criteria}. Every global minimizer of
$\mathcal L$ has $E=V\hat U^\top$ and $W=\hat UV^{-1}$ for an invertible
$V$, so $E^\top E=\hat U\Omega\hat U^\top$ with $\Omega=V^\top V$ positive
definite.
\begin{enumerate}
\item Gradient flow conserves $C=EE^\top-W^\top W$. If it converges to a
global minimizer, then $E^\top E=\hat U\Omega\hat U^\top$ with $\Omega$
positive definite and spectrum $\{f(\lambda):\lambda\in\mathrm{spec}(C_0)\}$, where
$C_0$ is the value of $C$ at initialization and
$f(\lambda)=(\lambda+\sqrt{\lambda^2+4})/2$. Since $f$ is $1$-Lipschitz with $f(0)=1$,
$\norm{\Omega-I}_{\rm op}\le\norm{C_0}_{\rm op}$.
\item Let two runs of gradient flow converge to global minimizers from
independent initializations under which $C_0$ has a density. If both runs
use the same sample, the probe Gram matrices at the limits of the two
flows satisfy $\CKA<1$ almost surely. If $H\mathbf X$ has full column
rank, then the same holds when the two runs use different samples that
both satisfy the hypotheses above, and, for the same or for different
samples, $\Prob(1-\CKA\le\varepsilon)$ is at most a bound determined by
$\mathbf X$ and the laws of $C_0$ in the two runs, which tends to zero as
$\varepsilon\to0$.
\item With weight decay $\lambda_{\mathrm{wd}}\in(0,\hat\lambda_{n_h})$, every critical
point of $\mathcal L+\lambda_{\mathrm{wd}}(\norm{E}_F^2+\norm{W}_F^2)$ satisfies
$EE^\top=W^\top W$, and every global minimizer has
$E^\top E=\hat U\,\mathrm{diag}(1-\lambda_{\mathrm{wd}}/\hat\lambda_1,\dots,
1-\lambda_{\mathrm{wd}}/\hat\lambda_{n_h})\,\hat U^\top$. With tied weights
$W=E^\top$ and no weight decay, every global minimizer has
$E^\top E=\hat U\hat U^\top$. In
both cases the probe Gram matrix is unique.
\item The orthogonal projectors onto the column spaces of $K$ and of
$\bar K=HKH$ are the same for every global minimizer in parts~1 and~3 on
one sample.
\item Let two independent samples each consist of $N$ independent draws
from a law with mean zero and finite second moments, and let $H\mathbf X$
have full column rank. Let the covariance of this law be positive
definite with its $n_h$th largest eigenvalue above the next eigenvalue
and above $\lambda_{\mathrm{wd}}$. As $N\to\infty$, both samples satisfy
the hypotheses above with probability tending to one, and on the two
samples each orthogonal projector of part~4 converges in probability to
one limit. The probe Gram matrices of global minimizers of the same loss
of part~3 on the two samples also converge in probability to one limit,
and their $\CKA$ converges in probability to one.
\end{enumerate}
\end{formalof}

In embedded deep linear networks, \citet{ziyin2025proof} prove that the
implicit gradient-norm penalty of finite-step stochastic gradient descent
selects solutions aligned up to rotation. This selection by the optimizer
differs from the selection by weight decay in part~3.

\begin{proof}[Proof of Theorem~\ref{thm:gauge}]
Part~1. Writing $\Gamma=(I-WE)\hat\Sigma$, gradient flow on $\mathcal L$
is $\dot E=2W^\top\Gamma$ and $\dot W=2\Gamma E^\top$, so the derivatives
of $EE^\top$ and $W^\top W$ along the flow are
$2(W^\top\Gamma E^\top+E\Gamma^\top W)$ and
$2(E\Gamma^\top W+W^\top\Gamma E^\top)$, whose difference is zero
\citep{du2018algorithmic,arora2018optimization}. At a global minimizer,
$WE=\hat U\hat U^\top$, the unique minimizer of
$\tr((I-WE)\hat\Sigma(I-WE)^\top)$ over products $WE$ of rank at most
$n_h$, because $\hat\Sigma$ is positive definite and
$\hat\lambda_{n_h}>\hat\lambda_{n_h+1}$. The row space of $E$ then equals
the column space of $\hat U$, so $E=V\hat U^\top$ and $W=\hat UV^{-1}$ for
an invertible $V$. Hence $EE^\top=VV^\top$ and
$W^\top W=(VV^\top)^{-1}$, so $VV^\top-(VV^\top)^{-1}=C_0$, whose unique
positive definite solution is $VV^\top=f(C_0)$, and
$E^\top E=\hat U(V^\top V)\hat U^\top$ with
$\mathrm{spec}(V^\top V)=\mathrm{spec}(VV^\top)$. The derivative
$f'(\lambda)=(1+\lambda/\sqrt{\lambda^2+4})/2$ lies in $(0,1)$, which gives the Lipschitz
bound.

Part~2. Let $\mathbf X_U=H\mathbf X\hat U$ and
$\Sigma_U=\mathbf X_U^\top\mathbf X_U$, which is positive definite. Then
$\bar K=\mathbf X_U\Omega\mathbf X_U^\top$, and $\CKA$ is the cosine
between $\Sigma_U^{1/2}\Omega_A\Sigma_U^{1/2}$ and
$\Sigma_U^{1/2}\Omega_B\Sigma_U^{1/2}$. With
$\mathbf f_a$ the vector of the values $f(\lambda)$,
$\lambda\in\mathrm{spec}(C_0^a)$, in decreasing order, this cosine is at
most
\[
 \Psi_{\Sigma_U}(\mathbf f_A,\mathbf f_B)=\max_{O_A,O_B\in O(n_h)}
 \cos\bigl(\Sigma_U^{1/2}O_A\,\mathrm{diag}(\mathbf f_A)O_A^\top\Sigma_U^{1/2},\;
 \Sigma_U^{1/2}O_B\,\mathrm{diag}(\mathbf f_B)O_B^\top\Sigma_U^{1/2}\bigr),
\]
which is attained on the compact set $O(n_h)^2$ and, since
$\Sigma_U^{1/2}$ is invertible, equals one only if $\mathbf f_A\propto\mathbf f_B$. Independent initializations
under which $C_0$ has a density make $\mathbf f_A\propto\mathbf f_B$ an
event of probability zero. Indeed, a density of $C_0$ on symmetric
matrices gives the sorted vector $\mathbf f_A$ a density on $\R^{n_h}$,
and for $n_h\ge2$ the
vectors proportional to a fixed $\mathbf f_B$ form a null set, so this
event has probability zero given $\mathbf f_B$. Hence $\CKA<1$ almost
surely. Now let $H\mathbf X$ have full column rank, and let the two runs use
samples with matrices $\hat U_A$ and $\hat U_B$. Writing
$\Omega_a=O_a\,\mathrm{diag}(\mathbf f_a)O_a^\top$ with $O_a\in O(n_h)$,
we have
$\bar K_a=H\mathbf XU_a\,\mathrm{diag}(\mathbf f_a)U_a^\top\mathbf X^\top H$
for $U_a=\hat U_aO_a$, which has orthonormal columns. Hence $\CKA$ is at
most the maximum of the cosine between
$H\mathbf XU_A\,\mathrm{diag}(\mathbf f_A)U_A^\top\mathbf X^\top H$ and
$H\mathbf XU_B\,\mathrm{diag}(\mathbf f_B)U_B^\top\mathbf X^\top H$ over
all $U_A,U_B\in\R^{n_x\times n_h}$ with orthonormal columns. These
matrices are nonzero and the set of such pairs $(U_A,U_B)$ is compact, so
this maximum is a continuous function of $(\mathbf f_A,\mathbf f_B)$
determined by $\mathbf X$. Since $H\mathbf X$ has a left inverse, the
cosine equals one only if
$U_A\,\mathrm{diag}(\mathbf f_A)U_A^\top\propto U_B\,\mathrm{diag}(\mathbf f_B)U_B^\top$,
which requires $\mathbf f_A\propto\mathbf f_B$. Hence
$\Prob(1-\CKA\le\varepsilon)$ is at most the probability that this
maximum is at least $1-\varepsilon$. As $\varepsilon\to0$, these events
decrease to the event that the maximum equals one, which lies in the event
$\mathbf f_A\propto\mathbf f_B$ of probability zero, so the bound tends to
zero. The same argument gives $\CKA<1$ almost surely for different
samples.

Part~3. A critical point of the regularized loss satisfies
$W^\top\Gamma=\lambda_{\mathrm{wd}} E$ and $\Gamma E^\top=\lambda_{\mathrm{wd}} W$. Multiplying the
first equation by $E^\top$ on the right and the second by $W^\top$ on the
left gives $\lambda_{\mathrm{wd}} EE^\top=W^\top\Gamma E^\top=\lambda_{\mathrm{wd}} W^\top W$. For
every $(E,W)$, $\norm E_F^2+\norm W_F^2\ge2\norm{WE}_*$, with equality if
and only if $E^\top E=((WE)^\top WE)^{1/2}$ and
$WW^\top=(WE(WE)^\top)^{1/2}$, and every matrix of rank at most $n_h$ is
a product $WE$ attaining equality. The problem therefore reduces to
minimizing $\tr((I-WE)\hat\Sigma(I-WE)^\top)+2\lambda_{\mathrm{wd}}\norm{WE}_*$ over the
product $WE$ of rank at most $n_h$, whose minimizer is
$WE=\hat U\,\mathrm{diag}(1-\lambda_{\mathrm{wd}}/\hat\lambda_1,\dots,
1-\lambda_{\mathrm{wd}}/\hat\lambda_{n_h})\,\hat U^\top$ \citep{kunin2019loss}. Each
direction saves $(\hat\lambda_i-\lambda_{\mathrm{wd}})^2/\hat\lambda_i$, which increases
with $\hat\lambda_i$, and $\lambda_{\mathrm{wd}}<\hat\lambda_{n_h}$ together with
$\hat\lambda_{n_h}>\hat\lambda_{n_h+1}$ makes the minimizer unique.
Equality then gives $E^\top E=((WE)^\top WE)^{1/2}=WE$. With tied weights
the loss is $\tr((I-E^\top E)\hat\Sigma(I-E^\top E))$, and by the
uniqueness in part~1 its minimum over positive semidefinite $E^\top E$ of
rank at most $n_h$ is attained only at $\hat U\hat U^\top$.

Part~4. In every case
$K=\mathbf X\hat U\Omega\hat U^\top\mathbf X^\top$ with $\Omega$ positive
definite, so the column space of $K$ is that of $\mathbf X\hat U$, and
the column space of
$\bar K=H\mathbf X\hat U\Omega\hat U^\top\mathbf X^\top H$ is that of
$H\mathbf X\hat U$.

Part~5. By the weak law of large numbers, the second-moment matrix of
each sample converges in probability to the covariance $\Sigma$ of the
law, since the law has mean zero. With probability tending to one, both
second-moment matrices are therefore positive definite, with their
$n_h$th largest eigenvalues above the next eigenvalues and above
$\lambda_{\mathrm{wd}}$. The gap of $\Sigma$ makes the orthogonal
projector $\hat U\hat U^\top$ onto the top-$n_h$ eigenspace of
$\hat\Sigma$ continuous in $\hat\Sigma$ at $\Sigma$, and $\hat\Sigma^{-1}$
is continuous at $\Sigma$. Under weight decay, part~3 gives
$E^\top E=\hat U\hat U^\top-\lambda_{\mathrm{wd}}\hat\Sigma^{-1}\hat U\hat U^\top$,
and with tied weights $E^\top E=\hat U\hat U^\top$. For global minimizers
of the same loss of part~3, the probe Gram matrices
$K=\mathbf XE^\top E\mathbf X^\top$ on the two samples thus converge in
probability to one limit. The limit of $E^\top E$ is positive
semidefinite of rank $n_h$, and $H\mathbf X$ has full column rank, so the
centered limit is nonzero and the $\CKA$ of the two probe Gram matrices
converges in probability to one. By part~4, the column spaces of $K$ and
$\bar K$ are those of $\mathbf X\hat U\hat U^\top$ and
$H\mathbf X\hat U\hat U^\top$, which converge in probability to
$\mathbf XUU^\top$ and $H\mathbf XUU^\top$ for the orthogonal projector
$UU^\top$ onto the top-$n_h$ eigenspace of $\Sigma$. All these matrices
have rank $n_h$, because $H\mathbf X$, and hence $\mathbf X$, has full
column rank. The orthogonal projector onto the column space is continuous
along matrices of constant rank, so on the two samples each projector of
part~4 converges in probability to one limit.
\end{proof}

\begin{formalof}{Corollary~\ref{cor:gauge-neighbors}}[Full Version]
In the setting of Theorem~\ref{thm:gauge}, let two global minimizers on
one sample have $E_a^\top E_a=\hat U\Omega_a\hat U^\top$, $a\in\{A,B\}$,
and let $\mathrm{cond}_{AB}$ be the condition number of
$\Omega_A^{-1/2}\Omega_B\Omega_A^{-1/2}$. Their probe Gram matrices share
the $k$-neighbor set of every probe $i$ with
$D_{A,i,(k+1)}>\mathrm{cond}_{AB}\,D_{A,i,(k)}>0$, where $D_{A,i,(k)}$ is
the $k$th smallest off-diagonal entry of row $i$ of $D_A$. If both are
global minimizers of the same loss of part~3 of Theorem~\ref{thm:gauge},
then $\mathrm{cond}_{AB}=1$. If both minimizers are
limits of gradient flow from initial values $C_0^A$ and $C_0^B$ of $C$,
then
$\mathrm{cond}_{AB}\le\prod_{a\in\{A,B\}}f(\lambda_{\max}(C_0^a))/f(\lambda_{\min}(C_0^a))$,
and this bound is at most $f(\lambda_0)^4$ if
$\norm{C_0^a}_{\rm op}\le\lambda_0$ for both runs.
\end{formalof}

\begin{proof}[Proof of Corollary~\ref{cor:gauge-neighbors}]
Both minimizers share $\hat U$, so
$D_a(i,j)=(x_i-x_j)^\top\hat U\Omega_a\hat U^\top(x_i-x_j)$ for the
probes $x_i$, the rows of $\mathbf X$. Let $\lambda_-$ and $\lambda_+$ be
the smallest and the largest eigenvalue of
$\Omega_A^{-1/2}\Omega_B\Omega_A^{-1/2}$, so that
$\mathrm{cond}_{AB}=\lambda_+/\lambda_-$. The ratio $D_B(i,j)/D_A(i,j)$ is
a generalized Rayleigh quotient, so
$\lambda_-D_A(i,j)\le D_B(i,j)\le\lambda_+D_A(i,j)$ for all $i,j$. For
$j\in\Nset_k^A(i)$ and $l\notin\Nset_k^A(i)\cup\{i\}$, the assumption
$D_{A,i,(k+1)}>\mathrm{cond}_{AB}\,D_{A,i,(k)}$ gives
$D_B(i,j)\le\lambda_+D_{A,i,(k)}<\lambda_-D_{A,i,(k+1)}\le D_B(i,l)$, so
the $k$ smallest entries of row $i$ of $D_B$ are those indexed by
$\Nset_k^A(i)$. Under weight decay or tied weights, part~3 of
Theorem~\ref{thm:gauge} gives $\Omega_A=\Omega_B$, hence
$\mathrm{cond}_{AB}=1$. In general
$\lambda_+\le\lambda_{\max}(\Omega_B)/\lambda_{\min}(\Omega_A)$ and
$\lambda_-\ge\lambda_{\min}(\Omega_B)/\lambda_{\max}(\Omega_A)$, so
$\mathrm{cond}_{AB}$ is at most the product of the condition numbers of
$\Omega_A$ and $\Omega_B$, a bound that holds without gradient flow.
Under gradient flow, part~1 of Theorem~\ref{thm:gauge} gives
$\mathrm{spec}(\Omega_a)=f(\mathrm{spec}\,C_0^a)$, and $f$ is increasing,
so the condition number of $\Omega_a$ is
$f(\lambda_{\max}(C_0^a))/f(\lambda_{\min}(C_0^a))$. If
$\norm{C_0^a}_{\rm op}\le\lambda_0$ for both runs, this ratio is at most
$f(\lambda_0)/f(-\lambda_0)=f(\lambda_0)^2$, because $f(\lambda)f(-\lambda)=1$, so
$\mathrm{cond}_{AB}\le f(\lambda_0)^4=e^{4\operatorname{arsinh}(\lambda_0/2)}$, since
$f(\lambda)=e^{\operatorname{arsinh}(\lambda/2)}$.
\end{proof}

\subsection{Objectives with Different Monotone Links}
\label{app:learning-link}

In the informal statement of Section~\ref{sec:objective}, part~2
combines parts~2 and~3 of the formal version below, and part~3 is
part~4 of the formal version.

\begin{formalof}{Theorem~\ref{thm:link}}[Formal Version]
In the balanced finite Gibbs model with positive definite $K_\star$, let
learner A minimize the empirical InfoNCE risk at temperature $t$ over
unit-diagonal positive semidefinite matrices on $\Xcal$, with a minimizer
$\hat K_A$, and let learner
B minimize the empirical spectral contrastive risk of
\citet{haochen2021provable} over features $h$ of dimension at least
$|\Xcal|$, whose Gram matrix normalized to unit diagonal is $\hat K_B$. This risk is
$-2\,\E_{(x,x')\sim P_+}h(x)^\top h(x')+\E(h(x)^\top h(x''))^2$, with
$x$ and $x''$ drawn independently from $P$ in the second term. Suppose $|S|\ge3$
and the off-diagonal distances of $K_\star$ on $S$ are pairwise distinct,
and let $D_a=D(\hat K_a)$, $D_\star=D(K_\star)$, and CKA be computed on $S$.
As both sample sizes grow:
\begin{enumerate}
\item $\hat K_A\to K_\star$ and $\hat K_B\to K_\star^{(t)}$ in
probability, where $K_\star^{(t)}=\exp\circ((K_\star-\ind\ind^\top)/t)$,
so that $D_B$ converges in probability to $D_B^\infty=g_t(D_\star)$ with $g_t(D)=2(1-e^{-D/(2t)})$ applied
entrywise;
\item $p_{(T,=)}\to1$ for every observable $T\preceq T_E$;
\item for every continuous, positive, positively homogeneous
normalization $\psi$ of distance matrices,
$\varepsilon_\psi=\norm{D_\star/\psi(D_\star)-g_t(D_\star)/\psi(g_t(D_\star))}_{\max}$
is positive, the discrepancy
$\norm{D_A/\psi(D_A)-D_B/\psi(D_B)}_{\max}$ converges in probability to
$\varepsilon_\psi$, the criterion with this discrepancy and tolerance
$\varepsilon<\varepsilon_\psi$ has $p_Q\to0$, and
$\CKA(\hat K_A,\hat K_B)\to\CKA(K_\star,K_\star^{(t)})<1$ in probability;
\item the criterion with discrepancy
$\inf_g\norm{g(D_A)-D_B}_{\max}$, the infimum over nondecreasing $g$
applied entrywise, has $p_Q\to1$ for every $\varepsilon>0$, and two
learners trained with the same objective have $p_Q\to1$ for every
criterion whose discrepancy is a continuous function of their two
distance matrices that vanishes when these are equal, at every tolerance
$\varepsilon>0$.
\end{enumerate}
\end{formalof}

The concave link $g_t$ compresses the large distances of learner B
relative to those of learner A. If one entrywise link relates the Gram
matrices learned with two objectives for every target on domains of
every size, this link belongs to a restricted class. A
continuous function applied entrywise preserves positive
semidefiniteness of correlation matrices of every size if and only if
the function is a power series with nonnegative coefficients
\citep{schoenberg1942positive}, and \citet{rudin1959positive} removed the
continuity assumption. The map $K_\star\mapsto K_\star^{(t)}$ of
Theorem~\ref{thm:link} is of this form. The fit of a monotone link ranges
over all nondecreasing maps, a class containing the monotone links of
this form. The share of the geometric discrepancy of two trained models
removed by a fitted monotone link therefore bounds from above the share
explained by a change of objective of this kind.

\begin{proof}[Proof of Theorem~\ref{thm:link}]
Population limits. For learner A, the InfoNCE risk is minimized by
critics of the form $K(x,x')/t=K_\star(x,x')/t+\upsilon(x)$. Symmetry of
$K$ forces $\upsilon$ to be constant, and the unit diagonal forces
$\upsilon=0$, so $K_\star$ is the unique population minimizer. For learner
B, write $K(x,x')=h(x)^\top h(x')$. The spectral risk equals
$\sum_{x,x'}P(x)P(x')\bigl(K(x,x')^2-2K(x,x')K_+(x,x')\bigr)$ with
$K_+(x,x')=P_+(x,x')/(P(x)P(x'))=e^{K_\star(x,x')/t}/Z$ under balance,
i.e., it is a weighted squared distance from $K$ to $K_+$ up to a
constant. The matrix $K_+$ is positive definite by the argument of
Proposition~\ref{prop:bernstein}, so it is the unique minimizer over
positive semidefinite $K$ and is realized by features of dimension
$|\Xcal|$. Its diagonal equals $e^{1/t}/Z$, and normalization gives
$K_\star^{(t)}$. Both empirical risks are convex in $K$ and converge
uniformly on compact sets, so both estimates are consistent. The diagonal
of the limit $K_+$ is positive, so the Gram matrix learned by B has a
positive diagonal with probability tending to one, and its normalization
$\hat K_B$ converges to $K_\star^{(t)}$, which is
part~1, and $D(K_\star^{(t)})=2-2e^{-D_\star/(2t)}=g_t(D_\star)$.

Part~2 follows because $g_t$ is strictly increasing, so the two limits
share every observable below $T_E$, and the distinct distances of
$K_\star$ place the limits in the interior of their common fibers. For
part~3, $g_t(D)/D$ is strictly decreasing for $D>0$, so $g_t(D_\star)$ is
proportional to $D_\star$ on the probes only if all their off-diagonal
distances are equal, which distinct distances exclude when $|S|\ge3$.
Hence $\varepsilon_\psi>0$. By part~1 and the continuity and positivity of
$\psi$, the discrepancy converges in probability to $\varepsilon_\psi$, and
the limits lie outside the closure of the acceptance set when
$\varepsilon<\varepsilon_\psi$. CKA equals one exactly when the centered
Gram matrices are proportional, and for hollow distance matrices $HDH$
determines $D$, so the limit of CKA is smaller than one. For part~4, the
two limits have the same order of entries, so near the limits a
nondecreasing map sends the entries of $D_A$ to those of $D_B^\infty$,
and the isotonic discrepancy is at most
$\norm{D_B-D_B^\infty}_{\max}\to0$. Two learners trained with the same
objective have equal limits, so for a discrepancy that is a continuous
function of the two distance matrices and vanishes when these are equal,
the pair of limits lies in the interior of the acceptance set at every
tolerance $\varepsilon>0$.
\end{proof}

\begin{formalof}{Remark~\ref{rem:temperature}}[Full Version]
\begin{enumerate}
\item Fix the training examples, and suppose that the empirical InfoNCE
risk at temperature $t_1$ has a unique minimizer $\hat K_{t_1}$ over
unit-diagonal positive semidefinite matrices on $\Xcal$. If $\hat K_{t_1}$
is positive definite and the matrix
\[
 \frac{t_2}{t_1}\hat K_{t_1}+\Bigl(1-\frac{t_2}{t_1}\Bigr)\ind\ind^\top
\]
is positive semidefinite, then this matrix is the unique minimizer at
temperature $t_2$ on the same examples, and its squared distances are
$t_2/t_1$ times those of $\hat K_{t_1}$. The matrix is positive
semidefinite for every $t_2\le t_1$.
\item In the finite Gibbs model, with probability tending to one, the
empirical risk at temperature $t$ has a unique minimizer, and this
minimizer is positive definite, so part~1 applies with $t_1=t$. An
InfoNCE learner at temperature $t_A$ in this model converges in
probability to $(t_A/t)K_\star+(1-t_A/t)\ind\ind^\top$ if this matrix is
positive semidefinite. This matrix is positive semidefinite for every
$t_A\le t$, and its squared distances are $t_A/t$ times those of
$K_\star$. Since $\psi$ is positively homogeneous, every criterion with
the discrepancy $\norm{D_A/\psi(D_A)-D_B/\psi(D_B)}_{\max}$ of part~3 of the formal
version of Theorem~\ref{thm:link} accepts the pair of this matrix and $K_\star$.
\item Fix the training examples. If every example anchored at an input
$x$ has the positive partner $x'\ne x$, every example anchored at $x'$ has
the positive partner $x$, and one of these examples has a candidate other
than its positive partner, then at every temperature the empirical risk
has no positive definite minimizer.
\end{enumerate}
\end{formalof}

\begin{proof}[Proof of Remark~\ref{rem:temperature}]
Part~1. Let $\mathcal U$ be the affine space of unit-diagonal symmetric
matrices on $\Xcal$, and let $\mathcal U_+\subset\mathcal U$ be the set of
its positive semidefinite matrices. The $J+1$ candidates of an example
with anchor $x$ are its positive partner $x_+$ and its $J$ negatives, and
at temperature $t_1$ the example has the loss
$\log\sum_{x''}e^{K(x,x'')/t_1}-K(x,x_+)/t_1$, where $x''$ runs over the
candidates. The empirical InfoNCE risks $\hat{\mathcal L}_{t_1}$ and
$\hat{\mathcal L}_{t_2}$ average the losses at temperatures $t_1$ and
$t_2$ over the examples and are convex on $\mathcal U$ by the proof of
Theorem~\ref{thm:peer}. For $K\in\mathcal U$, the matrix
$K'=(t_2/t_1)K+(1-t_2/t_1)\ind\ind^\top$ lies in $\mathcal U$ and
satisfies $K'/t_2=K/t_1+(1/t_2-1/t_1)\ind\ind^\top$. A constant added to
$K(x,x'')/t_1$ for every candidate $x''$ of an example leaves its loss
unchanged, so $\hat{\mathcal L}_{t_2}(K')=\hat{\mathcal L}_{t_1}(K)$. The
map $K\mapsto K'$ is a bijection of $\mathcal U$ onto itself, with inverse
$K'\mapsto(t_1/t_2)K'+(1-t_1/t_2)\ind\ind^\top$. Since $\hat K_{t_1}$ is
positive definite, a neighborhood of $\hat K_{t_1}$ in $\mathcal U$ lies in
$\mathcal U_+$, so $\hat K_{t_1}$ is a local minimizer of the convex
function $\hat{\mathcal L}_{t_1}$ on $\mathcal U$ and hence a global
minimizer on $\mathcal U$. It is the unique minimizer on $\mathcal U$,
because a second minimizer would make $\hat{\mathcal L}_{t_1}$ constant on
the segment joining the two minimizers, and the points of this segment
near $\hat K_{t_1}$ would be minimizers on $\mathcal U_+$ other than
$\hat K_{t_1}$. For $K'\in\mathcal U_+$, the preimage $K$ lies in
$\mathcal U$, so
$\hat{\mathcal L}_{t_2}(K')=\hat{\mathcal L}_{t_1}(K)
\ge\hat{\mathcal L}_{t_1}(\hat K_{t_1})$, with equality only if
$K=\hat K_{t_1}$. The matrix of part~1 is the image of $\hat K_{t_1}$ and
therefore the unique minimizer of $\hat{\mathcal L}_{t_2}$ on
$\mathcal U_+$ if it lies in $\mathcal U_+$. For $t_2\le t_1$ this matrix
is a convex combination of $\hat K_{t_1}$ and $\ind\ind^\top$, hence
positive semidefinite. Off the diagonal, its squared distances are
$2-2\bigl((t_2/t_1)\hat K_{t_1}+1-t_2/t_1\bigr)=(t_2/t_1)(2-2\hat K_{t_1})$.

Part~2. In the finite Gibbs model with $t_1=t$, the proof of
Theorem~\ref{thm:peer} gives a minimizer of $\hat{\mathcal L}_t$ on
$\mathcal U$ with probability tending to one, and this minimizer converges
in probability to $K_\star$. Since $K_\star$ is positive definite, the
minimizer is also positive definite with probability tending to one. An
example takes finitely many values, each with a smooth convex loss, and
$\hat{\mathcal L}_t$ weights the loss of each value by the frequency of
this value among the examples. These frequencies converge almost surely to
the probabilities of the values, so the Hessian of $\hat{\mathcal L}_t$ in
the coordinates $\theta$ of Theorem~\ref{thm:peer} converges to that of
the population risk uniformly on a neighborhood of $K_\star$. At this
minimizer, the empirical Hessian therefore converges in probability to
$I_J$, which is positive definite by the proof of Theorem~\ref{thm:peer}.
A convex function whose Hessian at a minimizer is positive definite has no
other minimizer, since the function would be constant on the segment
joining two minimizers and have a vanishing second derivative along this
segment. With probability tending to one, the minimizer on $\mathcal U$ is
thus unique and lies in $\mathcal U_+$, so it is also the unique minimizer
on $\mathcal U_+$. At every temperature, the population risk weights the
loss of each value of an example by the probability of this value instead
of its frequency. By the proof of Theorem~\ref{thm:link}, $K_\star$ is the
unique minimizer of the population risk at temperature $t$ on
$\mathcal U$, hence also on $\mathcal U_+$. Applied to the population risk
with $t_1=t$ and $t_2=t_A$, part~1 shows that
$(t_A/t)K_\star+(1-t_A/t)\ind\ind^\top$ is the unique population minimizer
on $\mathcal U_+$ at $t_A$ if this matrix is positive semidefinite.
Unit-diagonal positive semidefinite matrices have entries in $[-1,1]$, so
$\mathcal U_+$ is compact, and the empirical risks at $t_A$ converge to the
population risk uniformly on $\mathcal U_+$. Their minimizers therefore
converge in probability to this matrix, as in the proof of
Theorem~\ref{thm:link}. For $t_A\le t$, this matrix is a convex
combination of $K_\star$ and $\ind\ind^\top$, and the computation at the
end of part~1 with $K_\star$ in place of $\hat K_{t_1}$ gives its squared
distances. Multiplying a distance matrix $D$ by a positive number leaves
$D/\psi(D)$ unchanged, since $\psi$ is positively homogeneous, so the
discrepancy of every such criterion vanishes at this pair.

Part~3. The loss of an example depends on $K(x,x')$ only if the example
is anchored at $x$ or at $x'$. In an example anchored at $x$, the
derivative of the loss at temperature $t_1$ with respect to $K(x,x')$
equals $-t_1^{-1}$ times the total softmax weight of $K(x,\cdot)/t_1$ on
the candidates other than the positive partner $x'$. This derivative is
nonpositive, and negative if some candidate differs from $x'$. The same
holds with $x$ and $x'$ exchanged, so the derivative of
$\hat{\mathcal L}_{t_1}$ with respect to $K(x,x')$ is negative on all of
$\mathcal U$ at every temperature $t_1>0$. A positive definite minimizer
on $\mathcal U_+$ would be a stationary point of $\hat{\mathcal L}_{t_1}$
on $\mathcal U$, as in part~1, and $\hat{\mathcal L}_{t_1}$ has no
stationary point on $\mathcal U$.
\end{proof}

\section{Experimental Details and Full Results}
\label{app:exp}

This appendix gives the protocol, compute, and preregistration history
behind Section~\ref{sec:experiments}, and then the design, the results,
and every failed preregistered prediction of each experiment in that
section. Table~\ref{tab:exp-verdicts} lists every preregistered
prediction with its outcome. The complete experimental record at
\url{https://eethanshi.github.io/LRH/experiment_record.pdf} holds every
table and detail summarized here. The public repository
\url{https://github.com/EEthanShi/LRH} holds the code and the frozen
plans, predictions, and tolerances, except ten frozen files that are
withheld for now. A script
generated every table from the frozen result files, except entries
transcribed from frozen summary files and logs and checked against their
sources.

\begin{table}[tp]
\centering\scriptsize
\setlength{\tabcolsep}{3pt}
\renewcommand{\arraystretch}{1.0}
\begin{tabular}{@{}>{\raggedright\arraybackslash}p{.12\linewidth}p{.42\linewidth}p{.42\linewidth}@{}}
\toprule
Experiment, item & Preregistered rule or prediction & Outcome \\
\midrule
\multicolumn{3}{@{}l}{\emph{Pretrained models (fixed; randomness from permutations or input draws)}}\\
COCO, G1 & Crossing-law predictor beats two baseline predictors out of sample, once an estimator passes a preliminary check against these predictors on synthetic couplings & \textbf{Not run}, since every candidate estimator failed this check; predictor reported as unvalidated \\
COCO, secondary & Slopes ordered seed $\ge$ vision $\ge$ vision--language; elasticities in $k$ and $m$ cancel, without a tolerance & Ordering holds at $k=1,10,100$, mirroring the levels; summed elasticities have class medians $0.03$ to $0.06$, and the correlation $0.99$ of the elasticity in $k$ with minus the elasticity in $m$ is nearly guaranteed by the range of pairs \\
WIT, PE2 (Prop.~\ref{prop:mst}) & ``SOG-global $\approx$ 1-NN'' if forest share of shared MST edges $\ge0.70$ in $\ge90\%$ of pairs, $|s_{\Tset}-s_{N,1}|\le0.03$ in $\ge80\%$, and bridge trend absent & ``SOG-global $\approx$ 1-NN'' at both depths; the first two conditions hold nearly automatically, and the last-layer bridge trend $0.57$ (Holm $p=0.23$) is a power-limited non-rejection \\
WIT, diagnostics & A difference by a global monotone link refuted if a monotone fit explains $<0.10$ beyond an affine fit; anchors certified by Equation~\eqref{eq:neighbor-stability} expected near 0 & Refuted ($0.0004$, $0.0012$); 0 anchors at $k=10$ in 544 configurations \\
WIT, PE3 & ``Width artifact'' if equalizing width to 128 components raises the trend of CKA minus its permutation mean by $\ge0.3$; ``genuine global loss'' if by $\le0.1$ & Rule returns ``genuine global loss'' (change $0.00$, $-0.07$), a name \textbf{false as worded}, since CKA minus its permutation mean keeps its fixed-depth trend ($0.70$, $0.92$); the test shows that width does not drive this trend, not that the trend is lost \\
WIT, layer max & Layer-maximum protocol of ARH with decision rules fixed in a frozen plan & Pooled ``ARH pattern absent'' and BLOOMZ ``layer max loses the CKA trend''; with LLaMA-1 13B and 30B, pooled ``partial'' and LLaMA-1 ``intermediate'' (mKNN trend $0.38<0.60$); with OpenLLaMA 13B, pooled not decidable, since the interval $[-0.02,0.24]$ of the difference $0.10$ covers the rule boundary and 0 (``direction not established''), and OpenLLaMA ``layer max loses the CKA trend'' \\
WIT, LLaMA-1, \mbox{P-L1} to \mbox{P-L4} (Prop.~\ref{prop:baselines}, Cor.~\ref{cor:aggregation}) & Calibration lowers the layer-max CKA trend by $\ge0.10$ (P-L2); median threshold within $0.01$ of the predicted baseline, predicted order in $\ge34$ of 51 comparisons (P-L2a); mKNN trend moves by $<0.10$ (P-L1); debiased CKA trend drops by $<0.10$ and its pooled correlation by $\le0.05$ (P-L4); fixed-depth CKA keeps its sign (P-L3, marked at risk) & All hold; CKA $0.68\to-0.85$; median gap $0.0018$, order 51 of 51; mKNN $-0.03$; debiased CKA $0.00$, pooled $0.37\to0.36$; fixed depth $0.97$ and $1.00$; every $p\ge1/6$ \\
WIT, LLaMA-1, \mbox{P-L5} & Calibration of each layer pair before the maximum keeps the sign of the CKA trend & \textbf{Reversed} ($0.68\to-0.88$) \\
WIT, OpenLLaMA, \mbox{P-O1} to \mbox{P-O5} (Prop.~\ref{prop:baselines}, Cor.~\ref{cor:aggregation}) & All descriptive; bounds on the raw 13B score fixed from features forecast whether 13B ranks below a smaller model after calibration, 0 wrong among $\ge10$ decided (P-O2b); median threshold offset in $[0.0009,0.0029]$, order in $\ge22$ of 33 (P-O2a); P-O1, P-O4 as P-L1, P-L4; calibration before the maximum does not raise the CKA trend (P-O5); fixed-depth CKA without a prediction (P-O3) & All hold; 30 of 34 decided, all correct; offset $0.0019$, order 33 of 33; CKA $0.88\to0.18$; mKNN $0.00$; debiased CKA $+0.06$, pooled $0.37\to0.36$; calibration before the maximum $0.88\to0.21$; fixed depth keeps its sign; every $p\ge1/6$ \\
WIT, Pythia, H1 to H4, \mbox{P-P2} (Prop.~\ref{prop:baselines}, Cor.~\ref{cor:aggregation}) & Calibration lowers the layer-max CKA trend by $\ge0.10$ consistently across the seven sizes (H1, level $0.05$), and if H1 is confirmed, three secondary tests H2 to H4 under Holm; median threshold offset in the predicted band for $\ge6$ of 7 models, predicted order in $\ge277$ of 307 (P-P2a); 0 wrong among $\ge30$ decided forecasts of the order after calibration (P-P2b) & H1 \textbf{not confirmed} (CKA $0.93\to0.89$, drop $0.034$, $p_{\rm perm}=1612/5040$), so H2 to H4 were not tested; mKNN and debiased CKA move by $0.00$; P-P2a holds (7 of 7 in band, order 307 of 307); P-P2b holds (352 of 357 decided, all correct) \\
\midrule
\multicolumn{3}{@{}l}{\emph{Sentence-encoder follow-up on the COCO evaluation sets (pretrained models; randomness from input draws)}}\\
COCO, sentence encoders, \mbox{P-S1} (Conj.~\ref{conj:scale}) & Elasticities in $k$ and $m$ on $k\in\{10,100\}$ sum to a median within $\pm0.03$ and to within $\pm0.05$ in $\ge80\%$ of the pairs ($\ge8$) whose elasticity in $m$ has size $\ge0.1$; $k\ge10$ and tolerances set after the primary pairs & Holds; 15 of 15 pairs qualify and lie within $\pm0.05$, median $0.006$; holds in every ordering, in all 11 further arms, and for the primary pairs under centering \\
COCO, sentence encoders, \mbox{P-S2b}, \mbox{P-S2c} & (\mbox{P-S2b}) The slope at $k=10$ follows the line fitted on the primary pairs against the log level, not the language-model slope of its vision model; (\mbox{P-S2c}) within a pair, averaging one to five captions moves the slope with the log level at the rate of this line & Undecided, both labeled uninformative; predicted slopes differ by a median $0.092$, below the required $0.113$ (\mbox{P-S2b}), and five captions raise the level by $19\%$ over the first, below the required $35\%$ (\mbox{P-S2c}) \\
\midrule
\multicolumn{3}{@{}l}{\emph{Controlled training (randomness from initialization, data, order, and minibatches)}}\\
Linear AE, P1 & Gradient descent: neighbor sets, $G_s$, CKA unreliable at every $N$, mean at $N=2^{16}$ inside the $95\%$ interval at $N=\infty$ & \textbf{Falsified}; unreliable in 12 of 12 cells, plateau test fails for $G_s$ and CKA (paired differences about $10^{-4}$) \\
Linear AE, P2, P3 & Agreement on neighbor sets, $G_s$, and CKA reliable at $N=2^{16}$ under weight decay, tied weights, and small initialization; agreement under the projector criterion reliable under every training procedure & Confirmed \\
\bottomrule
\end{tabular}
\caption{Preregistered predictions and their outcomes for the pretrained
models and the linear autoencoder, with failed items in bold. The LLaMA-1, OpenLLaMA, and Pythia rows
concern language models added after the first preregistration, with
predictions frozen before scoring (Appendix~\ref{app:exp-protocol}), and
H1 to H4 were fixed before the download of any Pythia weight.}
\label{tab:exp-verdicts}
\end{table}

\begin{table}[tp]
\renewcommand{\theHtable}{\thetable.cont}
\centering\scriptsize
\setlength{\tabcolsep}{3pt}
\renewcommand{\arraystretch}{1.0}
\begin{tabular}{@{}>{\raggedright\arraybackslash}p{.12\linewidth}p{.42\linewidth}p{.42\linewidth}@{}}
\toprule
Experiment, item & Preregistered rule or prediction & Outcome \\
\midrule
\multicolumn{3}{@{}l}{\emph{First network experiment (CIFAR-5m networks with a linear head)}}\\
Network, P1 & Without weight decay, $u$ unreliable on $T_{N,10}$ and $G_s$ at both $N$ & \textbf{Failed}; $T_{N,10}$ 96 of 96 at $250$k, meeting the preregistered falsification condition; $G_s$ unreliable (65 of 96) under a tolerance set at $100$k (same-size test in row A3) \\
Network, P2, P4 & Without weight decay, $y$ and $\Delta_\Pi$ reliable at $250$k; trunk $s_N(10)$ equal across arms within $0.05$ & Passed (largest trunk difference $0.011$) \\
Network, P3 & Weight decay and small initialization make $u$ reliable at $250$k & Held under weight decay; \textbf{failed} under small initialization ($G_s$ 82 of 96) \\
Network, P5 & Fraction of agreeing pairs under small initialization within $0.1$ of that under weight decay & \textbf{Failed} ($0.146$ on $G_s$ at $250$k) \\
\midrule
\multicolumn{3}{@{}l}{\emph{Second network experiment (fresh runs; tolerances from pilot pairs of the same arm and size)}}\\
Second network, C3, C4a (Rem.~\ref{rem:optimizers}) & Momentum from $C_0=0$ gives a median relative imbalance above $0.05$ (lower $97.5\%$ bound), at least $0.02$ in $\ge90\%$ of runs; from the default initialization, a median change of $C$ above its initial norm and a median cosine of the final and the initial $C$ below $0.5$ (upper bound) & All hold at both $N$; $0.28$ and $0.33$, 64 of 64 runs; $9.3$ and $8.8$ times; $-0.017$ and $-0.005$ (upper bounds $0.011$, $0.012$) \\
Second network, I3 & From a balanced or small initialization, momentum adds a larger imbalance part than head flow on $G_s$ and $1-\CKA$ at both $N$; 8 tests, Holm at $0.05$ & 8 of 8 rejected; raw differences $0.0012$ to $0.0111$ \\
Second network, I4 (part~2 of Thm.~\ref{thm:gauge}) & Head flow from the default initialization, without a predicted direction: (i) mechanism of part~2 realized if the imbalance part on $G_s$ and its run-to-run part exceed 0 at both $N$; (ii) not detectable if the imbalance part is within $\pm0.5$ SD on $G_s$ and $1-\CKA$ at both $N$; simultaneous $98.75\%$ intervals & Outcome (i), $0.54$ and $0.59$ SD with run-to-run parts $0.29$ and $0.46$ SD; (ii) does not hold, upper bounds $0.59$ to $0.74$ SD \\
Second network, A1 (part~3 of Thm.~\ref{thm:gauge}) & Weight decay: $u$ reliable on $T_{N,10}$ and $G_s$ at both $N$ & Holds; 96 of 96 in all four cells, also under the pooled tolerance \\
Second network, A3 & Momentum without weight decay: $u$ and its common-gauge counterfactual unreliable on $G_s$ at both $N$ & Tolerance-sensitive, without a categorical claim; $u$ 0 and 5 of 96, counterfactual 1 and 21 of 96, and 95 of 96 at $250$k under the $7.8\%$ larger pooled tolerance \\
Second network, A7 & Balanced or small initialization under head flow: (a) $u$ reliable on $G_s$ at both $N$ and not tolerance-sensitive; (b) balanced part above $0.5$ SD at both $N$ (lower $95\%$ bound) & (a) \textbf{fails} in both arms ($u$ 20 and 35 of 96 balanced, 8 and 93 of 96 small); (b) holds, $2.7$ to $3.9$ SD; preregistered sentence on weight decay not added, since the balanced part under weight decay lies outside $\pm0.5$ SD ($-1.23$ SD at $25$k) \\
Second network, A2 & Where $y$ is reliable, $u$ not unreliable on $T_{N,10}$ and the $95\%$ interval of its mean excess within $\pm0.002$ & \textbf{Fails} in 4 of 10 cells, all at $250$k, with $u$ agreeing better than $y$ by $0.0017$ to $0.0023$; $u$ reliable on $T_{N,10}$ in every cell \\
Second network, A4, A8, A9 & Momentum without weight decay: common-gauge share of the $G_s$ excess at $250$k in $[0.5,0.9]$, intervals of the mismatch, balanced, and imbalance parts excluding 0 (A4); $G_s$ excess at $250$k over that at $25$k in $[0.5,1.5]$ (A8); $\Delta_E(u)<\Delta_E(y)$ in $\ge90\%$ of pairs (A9) & All hold; $0.71$; $0.83$; 96 of 96 at both $N$ \\
Second network, K4, K5, K6 (Cor.~\ref{cor:gauge-neighbors}) & At $k=10$, outputs of one run under the gauges of both runs: median bound over observed slot change in $[5,25]$ per cell, certified fraction $\le0.10$ (K4); balanced-head coordinates move $\ge0.01$ of the slots under momentum without weight decay (K5); cross-run excess over the function alone $\le0.0005$ (K6) & All hold; $10.9$ to $12.7$; at most $0.047$; $0.052$ and $0.070$; at most $0.0002$ \\
\midrule
\multicolumn{3}{@{}l}{\emph{Finite Gibbs model (certified minimizers inside the premises of Theorem~\ref{thm:link})}}\\
Gibbs, F0 (Thm.~\ref{thm:link}, parts~1 to~4 of the formal version, one model prediction) & Nine confirmatory tests under Holm: both learners converge from $2^{34}$ to $2^{40}$; exact agreement across objectives on the global order of eight probes and agreement under the isotonic discrepancy reliable at $2^{40}$; agreement on mean-normalized distances unreliable at $2^{34}$ and $2^{40}$; $1-\CKA$ at its limit; difference after undoing the link converging; cross-objective mean discrepancy within the sum of the within-objective means on order criteria, above it on geometric ones & All 9 confirmed, $p$ from $7.9\times10^{-68}$ to $5\times10^{-5}$; order 96 of 96, geometry 0 of 96, isotonic 96 of 96 (Appendix~\ref{app:exp-s23}) \\
Gibbs, secondary & Monte Carlo prediction of the ratio of the mean discrepancy across objectives to the mean of the two within-objective means at $2^{40}$, within a factor of $1.5$ (P5b); Gaussian limit and transition of Thm.~\ref{thm:peer} (P1b, P2b) & P5b \textbf{failed} for neighbor sets and spanning trees; P1b and P2b pass \\
\midrule
\multicolumn{3}{@{}l}{\emph{CIFAR-5m contrastive networks (outside the premises of Theorem~\ref{thm:link})}}\\
Contrastive, link & Link confirmed if items (i) to (v) hold; falsified if the warp share is $<0.2$ or an order gap $>0.1$ & \textbf{Falsified}; warp share $0.147$, order gaps $1.0$; (i), (iii), (iv), (v) fail, (ii) passes by its letter \\
Contrastive, secondary & Fitted link exponent in $[2.5,10]$; co-occurrence prediction derived from Lemma~\ref{lem:cooccurrence} & \textbf{Both failed} \\
\bottomrule
\end{tabular}
\addtocounter{table}{-1}
\caption{Preregistered predictions and their outcomes (continued) for the
network experiments and the finite Gibbs model, with failed items in bold. SD is the standard deviation of the
output discrepancy over the main pairs of a cell in the second network
experiment.}
\end{table}

\subsection{Protocol}
\label{app:exp-protocol}

\paragraph{Reporting Rules.}
A prediction is frozen before scoring if we computed it from single-model
features before any cross-model score of the new models. An analysis is
exploratory if it lacks a decision rule fixed in advance or if its plan
labels it exploratory, and post hoc if we formulated its tested statement
after seeing the result. We present both kinds as descriptions of the
measured results, not as tests.
We report effect sizes, directions, and intervals with the exact $p$ of
each family, and we label a joint $p$ over families as exploratory. A
clear trend that misses a preregistered threshold is reported with both
the outcome of its test and its size.

\paragraph{Models, Data, and Criteria.}
The language models compared with the 17 vision transformers of ARH are
BLOOMZ at five sizes from 0.56B to 7.1B parameters, OpenLLaMA 3B and 7B,
and LLaMA-1 7B. In Section~\ref{sec:exp-aggregation}, we add LLaMA-1 13B
and 30B, OpenLLaMA 13B, and seven Pythia models from 70M to 6.9B, and
the repository lists the Hugging Face identifier and pinned revision of
every pretrained model.
Features are preprocessed by the clip and $\ell_2$ normalization of PRH
(Appendix~\ref{app:measurement-definitions}), with plain $\ell_2$
normalization as secondary preprocessing.
Appendix~\ref{app:exp-prereg} reports the check of our code against the
published code of PRH, ARH, and SOG, and Appendix~\ref{app:exp-published}
the comparison of our measurements with the published results. The size trend
$\mathrm{Tr}$ of Section~\ref{sec:exp-protocol} has an exact one-sided
$p$ of at least $1/120$ along BLOOMZ, $1/6$ along LLaMA-1 and OpenLLaMA,
and $1/5040$ along Pythia. The fixed depths are relative depth $0.5$ (mid
depth) and $1$ (the last layer), and we compute the layer maximum of PRH,
ARH, and SOG under a preregistered plan.

\paragraph{Decisions in Controlled Training.}
In a learning experiment we train disjoint pairs of runs, count the pairs
whose discrepancy is at most a frozen tolerance $\varepsilon$, and decide
with a one-sided $95\%$ Clopper--Pearson bound. Agreement is reliable if
the lower bound exceeds $0.9$, unreliable if the upper bound is below
$0.8$, and inconclusive otherwise, which instantiates
Definitions~\ref{def:reliable} and~\ref{def:transition} with
$\delta=\eta=0.1$. With 96, 64, and 32 pairs, these rules return
``reliable'' with at least 92, 62, and 32 pairs within tolerance and
``unreliable'' with at most 69, 45, and 21. In
Table~\ref{tab:exp-verdicts}, an entry such as ``$u$ reliable on
$T_{N,10}$'' states that agreement of disjoint pairs of runs under this
criterion is reliable.

\paragraph{What Is Random.}
The pretrained models are fixed, and a score on them is random through the
draw of evaluation inputs (Lemma~\ref{lem:rows}) and the $n_\pi=500$ permutations of its calibration. The comparisons on WIT use
one draw of 1{,}024 inputs, and our intervals resample the 17 vision
models, so they cover the choice of vision models and the permutations
but not the draw of inputs. The COCO experiment varies this draw through five random nested orderings of its pool (Appendix~\ref{app:exp-s11}). In the linear autoencoder
the randomness is the initialization, the training sample at finite $N$,
and the minibatches, with common random numbers across training
procedures. In the finite Gibbs model the randomness is the training
sample of each run, drawn as exact multinomial counts from its own
seed. In the networks the randomness is the initialization, the training
subset, its order, and the augmentation, each drawn from its own seed. We evaluate the networks on a fixed set of 2048 held-out images. The two Gram matrices of a disjoint pair of runs therefore form a draw
from $\mu_{r_A}\otimes\mu_{r_B}$, with $r_A=r_B$ for two runs of one
training procedure and objective, and draws of different pairs are independent. The exception is one arm of the first network experiment, whose runs share training corpora (Appendix~\ref{app:exp-s22}). An arm is the set of
runs or measurements under one setting of an experiment, such as one
training procedure and step size or one objective and temperature.

\paragraph{Criteria and Discrepancies.}
\phantomsection\label{app:exp-criteria}
All criteria are computed in float64 on the squared distances $D$ of the
preprocessed features, with ties broken by input index, and in the
learning experiments we evaluate the criteria ordered by
Proposition~\ref{prop:lattice}.
The neighbor-set discrepancy $T_{N,k}$ is $1-s_N$ at $k$ neighbors, with
chance level $k/(m-1)$ for $s_N$. The order discrepancy $T_O$ is the
fraction of discordant pairs between the two rankings of the other inputs
by distance to an anchor, averaged over anchors, with ties counted as one
half. The discrepancy $T_E$ is $(1-\tau_b)/2$ for Kendall's $\tau_b$
between the off-diagonal entries of $D_A$ and $D_B$. The spanning-tree discrepancy
$T_{\mathrm{MST}}$ is $1-s_{\Tset}$, with the forest and bridge split of
Proposition~\ref{prop:mst} reported for the forest $F_A$, and SOG$_{0.1}$
is one minus the length-weighted 10-nearest-neighbor overlap of
Equation~\eqref{eq:sog-weight} at $\tau=0.1$. On the off-diagonal entries
of $D_A$ and $D_B$, $T_{G/\mathrm{Aff}}$ is one minus their Pearson
correlation, and $G_s$ on $T_{G/\R_+}$ is their root mean square
difference after division of each matrix by its mean off-diagonal entry.
We write $1-\CKA$ for one minus linear (biased) CKA. On the pretrained
models, debiased CKA is linear CKA with the unbiased estimator of the
Hilbert--Schmidt independence criterion in place of the biased
estimator. ARH calls this score unbiased CKA, and our implementation
matches the code of ARH. For the
contrastive networks, $\Delta_{\rm iso}$ is the cross-fitted root mean
square residual of the best global nondecreasing map from one distance
matrix to the other, in units of the target's mean distance and averaged over both directions. The warp share is the fraction of the affine
residual removed by this map. The projector distance
$\Delta_\Pi=\norm{\Pi_A-\Pi_B}_F^2/(\operatorname{rank}\Pi_A+\operatorname{rank}\Pi_B)$
compares the orthogonal projectors $\Pi_a$ onto the column spaces of the
two centered feature matrices. It equals one minus the mean squared canonical correlation when the ranks are equal (part~4 of
Theorem~\ref{thm:gauge}).

\paragraph{Hardware, Software, and Compute.}
\phantomsection\label{app:exp-compute}
Every experiment ran on one machine with an NVIDIA RTX PRO 6000
Blackwell Workstation Edition GPU (97{,}887\,MiB), an Intel Core Ultra 9
285K CPU with 24 cores, and 125\,GiB of memory, under Ubuntu 24.04.3.
The training code used PyTorch 2.13.0 with CUDA 13.2 and bf16 autocast,
and the metric code used float64 NumPy. GPU kernels are seed-deterministic
but not bitwise reproducible. The computations on the pretrained models took
under 4 GPU-hours for feature extraction, layer maxima, and their checks, and
about 210 CPU-hours, 159 of them for the neighbor computations over
evaluation sizes. With up to four runs sharing the GPU, the 1{,}392 runs of the first
network experiment took 4.91 hours of summed run time and the 496 runs of
the contrastive networks took 35.43 hours. Each count includes 48 excluded pilot runs. The second network experiment trained 4{,}352 runs,
1{,}984 of them pilots, in 20.92 hours of summed run time. The finite
Gibbs experiment took 12.71 GPU-hours and the linear autoencoder experiment about
1.5 CPU-hours. These times include contention between jobs sharing the
machine.

\subsection{Preregistration, Gates, and Deviations}
\label{app:exp-prereg}

\paragraph{Documents.}
Each preregistered experiment ran under a plan frozen before its main
computation, and an event log records the time of every
freeze. The Stage~1 preregistration (pretrained models and the linear
autoencoder) preceded the first experiment, and its amendment v1.1
preceded the opening of the COCO features. The plan of the
evaluation-size test followed the selection of its estimator and preceded
the opening of the COCO features. The Stage~2
preregistration (first network experiment and contrastive networks)
preceded any main run. Its one amendment changed the time criterion
of a pilot gate from per-run wall time to throughput before the
evaluation of that gate. The plans for LLaMA-1, OpenLLaMA, and Pythia
preceded the download of any new weights, and their predictions from
features preceded any cross-model score of the new models, an order
confirmed by file creation times.

\paragraph{Gates.}
Before any analysis of real data, the metric code passed four of its six
known-answer checks on synthetic couplings. A cross-fitted isotonic fit
explained $0.9978$ of the variance of an exponential-link coupling against a required $0.999$. The first coupling estimator missed the observed
$s_N$ at $k=1$ on Gaussian couplings by up to $0.057$ against a tolerance
of $0.02$. Amendment v1.1 based the isotonic check on the in-sample fit
and replaced the first-order approximation behind the estimator by the
exact expected discrepancy for $k\le3$.
We did not execute the rebuild of the exponential coupling specified by
this amendment, so the pass of the isotonic check rests on the existing
run, with an in-sample fraction of $1.000000$. The first coupling estimator also failed a real-data sanity check on held-out WIT items. Under the amendment, we re-selected among at most four pre-listed estimators, and every candidate failed the preliminary check of G1, so we did not run G1. Our
code agreed with the reference implementations of PRH and ARH to
$1.67\times10^{-8}$, within a tolerance of $10^{-6}$. On three pairs of
models at both fixed depths, it also agreed to $10^{-6}$ with the
published code of PRH and ARH, and with that of SOG under the
preprocessing of PRH. The exceptions come from two WIT inputs with
identical captions. The code of SOG gives these inputs a distance of
about $10^{-8}$, which changes its length-weighted overlaps by up to
$9.8\times10^{-4}$. The mKNN code shared by ARH and SOG breaks their tie
differently from ours in one of 60 values, by $3.3\times10^{-5}$. The first and second network experiments and the contrastive networks passed all of their pilot gates.

\paragraph{Deviations.}
\begin{enumerate}[nosep,leftmargin=*]
\item The power gate of the spanning-tree test asked for $m=3072$
inputs, more than were extracted, so we used the largest disjoint set
($m=2047$), with the same outcome as on $m=1024$.
\item In the linear autoencoder, we added secondary arms with step
$0.002$ before the confirmatory runs, since the preregistered step $0.05$ changes $C$ by a median $53\%$. The preregistration left the learning rates of SGD and Adam open, and we chose these rates on excluded pilots.
\item The unconstrained spectral loss diverged at every tested learning
rate, so the spectral arm of the contrastive networks uses the ball
projection of \citet{haochen2021provable}. This projection was active on
every evaluated input, and the arm therefore trains on the unit sphere.
\item The main run of the layer maximum started before the end of the
check against the code of ARH, and its outputs were opened after a pass of
this check.
\item In the first network experiment, an interim analysis on two runs per
cell preceded the tolerance freeze, and the analysis code was frozen only after the tolerances were fixed. The final analysis reused cached per-pair results after an independent recomputation matched all 252 decision cells. The analysis code computes $\Delta_\Pi$ on $u$ instead of $y$, and the two distances agree to $10^{-7}$.
\item The fixed-depth comparison preceded the recording of the amendment
for the two failed known-answer checks, although the preregistration
required these checks to pass before any analysis of real data. This
comparison uses neither failed check, and each cross-fitted isotonic
value carries a flag.
\item A script of the evaluation-size test was edited after its outputs
existed, without an event-log entry, and an independent script reproduces
every number of the plan. A verification run printed frozen predictions of
the first coupling estimator for three WIT pairs, outside every
preregistered endpoint.
\item In the LLaMA-1 experiment, the order of the prediction freeze and
the scores rests on file creation times, since the automated order check
inspected an unused directory. Deterministic validation jobs also exceeded
the plan's cap of six CPU threads for about 77 minutes. The LLaMA-1 7B
results of one implementation reproduce the frozen run bitwise.
\item In the OpenLLaMA experiment, we saw the exploratory joint statistic
in a log of one implementation before the agreement check of the two
implementations. The inputs, code, and rule of this statistic were
frozen beforehand, and read-only check scripts were added after the code
freeze. The plan's bound on the fixed-depth baseline fails
at the last layer under plain $\ell_2$, so under both preprocessings we
report P-O3 without a prediction.
\item In the Pythia experiment, the agreement check failed by its letter
on floating-point differences of at most $9.1\times10^{-14}$. We accepted this check post hoc, after seeing the category returned by the decision rule
(Appendix~\ref{app:exp-pythia}). A GPU smoke test was rewritten after the
code freeze, and the new version was frozen before its first run.
\item The finite Gibbs experiment ran under four amendments, two
of which fixed programming errors in the frozen analysis code
(Appendix~\ref{app:exp-s23}).
\item The sentence-encoder follow-up ran under one amendment,
which fixed a programming error in its file-writing step, and its review
preceded the independent reimplementation, against the planned order. Its
plan calls the pool caption random, whereas the pool holds the last
caption of every image (Appendix~\ref{app:exp-s11}).
\end{enumerate}

\paragraph{Design Errors in the Stage~2 Preregistration.}
We record seven errors in our own Stage~2 design, each of which leaves
every number and decision above unchanged. Prediction P1 of the first
network experiment extended the obstruction to the neighbor-set
criterion, outside part~2 of Theorem~\ref{thm:gauge}. Prediction P3 stated that agreement of $u$ under small initialization would be reliable, although the same
document expected the balance premise of this initialization to fail
under stochastic gradient descent. Both Stage~2 experiments set the
tolerance at a pilot size between the two tested sizes, and with a fixed budget the smaller size also meant more epochs. At $25$k, the first network
experiment therefore finds 0 of 96 pairs within tolerance for $u$ and $y$
on every criterion under each training procedure with independent
training subsets. A decision on $u$ at
the tolerance of $y$ depends on both the gauge shared by all runs and the
run-to-run mismatch, without separating these two parts. In the
contrastive experiment, item~(v) assumed that a change of temperature is a
pure rescaling, which Remark~\ref{rem:temperature} gives for a unique
positive definite minimizer, a minimizer absent from the empirical risk of
the networks. The falsification threshold of $0.2$ on the warp share
rules out the sharp link of the theorem but not a general monotone link. Finally, the fitted exponent $\hat\beta$ equals minus
the logarithm of the mean cosine of the spectral arm to within $0.008$,
which ties $\hat\beta$ to the spectral embedding independently of the
InfoNCE temperature (Appendix~\ref{app:exp-s21}).

\subsection{Evaluation Size}
\label{app:exp-s11}

The two preregistered experiments of Section~\ref{sec:exp-scale} test Conjecture~\ref{conj:scale} by measuring $s_N$ at a fixed $k$ and at a fixed $k/m$ on growing evaluation sets. Two exploratory follow-ups measure how the depth of the vision features changes the level of $s_N$.
All four use one pool of $16{,}384$ COCO validation images, each paired
with the caption stored by \texttt{coco-30-val-2014}, i.e., its last
caption in the COCO annotation file. Five fixed random orderings of the pool
define nested prefixes of sizes $m\in\{256,512,\dots,16{,}384\}$, and in
each ordering the anchors are those of the first 1024 images that lie in a fixed random half of the pool,
with neighbors searched in the whole prefix. The slope of a pair at
fixed $k$ is the least-squares slope of $\ln s_N$ on $\ln m$ over the
seven sizes, averaged over the five orderings. A slope is flat if its
absolute value is below $0.1$ and falling or rising otherwise. At $k=10$
the standard deviation of the $s_N$ of a pair over the five orderings is
at most $0.031$. The vision models are DINOv2-B, OpenCLIP ViT-B/16
(LAION-2B), MAE-B, AugReg ViT-B/16 (ImageNet-21k), and SigLIP-B/16, and
the language models are PolyPythia-410M seeds 1 to 9, BLOOMZ-1.7B, and
OpenLLaMA-3B, all at mid depth. The 29 primary pairs are 4 disjoint
pairs of PolyPythia-410M seeds, the 10 pairs of vision models, and the
15 pairs of a vision model with BLOOMZ-1.7B, OpenLLaMA-3B, or seed 1 of PolyPythia-410M. A secondary set holds all 36 pairs of the nine
seeds. The experimental record gives the values of each pair, the class medians of $s_N$ at every $k$ and $m$, the per-pair crossing-law predictions, and three sensitivity analyses.

\paragraph{Slope Classes and Stability at Fixed $k/m$.}
All 4 seed pairs have flat slopes at $k=1$ and $k=10$, as do all 36 seed
pairs of the secondary set. All 10 vision pairs have falling slopes at
$k=10$ and $k=100$, although after correction for chance the slope stays
falling for 7 of 10 at $k=100$ (exploratory). All 15 vision--language
pairs have falling slopes at every $k$, and at $k=10$ their class median
of $s_N$ falls from $0.091$ at $m=256$ to $0.0085$ at $m=16{,}384$. There, the $s_N$ of every such pair exceeds its chance level by a
factor of at least $8.75$. At $k=10$ the class median slopes are
$-0.01$, $-0.21$, and $-0.57$ for seed, vision, and vision--language
pairs, against $-1.00$ for the chance level $k/(m-1)$. At all three values of $k$, the slopes follow the preregistered secondary order, seed above vision above vision--language, but this order repeats that of the levels of $s_N$ rather than adding a second observation. The rising
class of Conjecture~\ref{conj:scale} appears in 0 of the 87 combinations
of a primary pair and a value of $k$. The preregistered slope rule also
classifies as flat all 9 units of a synthetic conformal coupling used to select the coupling estimator. At these levels of $s_N$ the rule fails to detect the rising class, and the conformal case of the conjecture stays untested. In a post hoc analysis
of the slope of $\ln(1-s_N)$ on $\ln m$, the steepest seed pair reaches
$-0.026$ at $k=1$, against about $-0.05$ for a conformal coupling at the
intrinsic dimension of the seed models. The separation between the seed pairs and the conformal case is moderate. At a shared $k/m$, the
class medians for $k=100$ stay within $4\%$ of those for $k=10$, whereas
those for $k=1$ lie $13$ to $19\%$ below for seed and vision pairs and
within $11\%$ for vision--language pairs
(Figure~\ref{fig:exp-neighbors}b). The elasticities in $k$ and $m$ of a
pair are the two coefficients of a least-squares fit of $\ln s_N$ on
$\ln k$ and $\ln m$. Their per-pair sums have class medians of
$0.03$, $0.06$, and $0.03$ for seed, vision, and vision--language pairs, against $-0.02$,
$-0.19$, and $-0.56$ for the elasticity in $m$ alone. The preregistered secondary prediction of a zero sum of these elasticities lacks a tolerance, so we report
these sums as its outcome. The elasticity in $k$ and minus the
elasticity in $m$ have correlation $0.99$ across the 29 pairs. The spread of $s_N$ nearly guarantees this value, since identical representations
give zero for both elasticities and pairs at chance level give one and
about minus one. For vision--language pairs the stability at large $k/m$ is partly trivial, since at $k=100$ and $m=256$ their class median
is $1.10$ times chance.

\paragraph{The Crossing-law Estimate.}
The preregistered out-of-sample test G1 required a coupling estimator
that passed a real-data sanity check and a preliminary check on
synthetic couplings against two baseline predictors. These are the
no-change null N0, which predicts the value of the pair at $m=1024$ for
every $m$, and the affine null N1, an affine fit of the coupling with Gaussian noise.
Every candidate estimator failed this preliminary check. In three of its
four comparisons, the null generates the synthetic coupling exactly or
approximately. The failure means that each estimator lost to a null on
the coupling generated by this null, not that the crossing law is wrong.
G1 was therefore not run, and the predictor is reported as unvalidated,
following the preregistered outcome rule. In a descriptive comparison,
the predictor recovers the slope class of all 29 pairs at $k=10$,
against 4 for N0 and 25 for N1. Like N1, it predicts the sign of the slope correctly in all 69 non-flat combinations of a pair and a value of $k$. On vision--language pairs at $k=10$, the predictor gives a
median slope of $-0.24$ against $-0.57$ observed (Figure~\ref{fig:exp-scale}), too shallow for all 15
pairs by $0.29$ to $0.40$. N1 instead predicts too steep a slope at
$k=10$ for all 29 pairs. We take the median of $|\ln(\text{predicted}/\text{observed})|$ over the cells of all three values of $k$ with $m\ne1024$ and an observed $s_N$ above twice chance. On vision--language pairs, this median corresponds to a factor of $2.86$ for the crossing law, larger than the $2.26$ of N0. The slopes are measured over a 64-fold range of $m$, and at $k=1$ the class median of the vision
pairs levels off at $0.082$ to $0.088$ for $m\ge2048$. The falling class of these pairs at larger $k$ may therefore be pre-asymptotic.

\begin{figure}[htbp]
\centering
\includegraphics[width=5.5in]{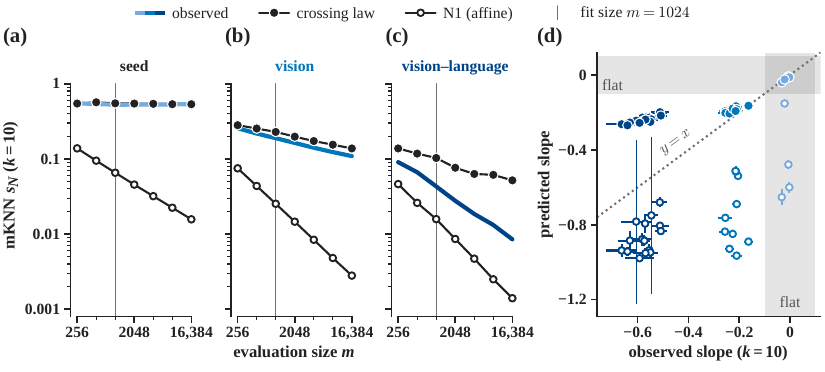}
\caption{\textbf{Descriptively, at $k=10$ the crossing law recovers the slope class of every pair but predicts too shallow vision--language slopes, whereas the affine null N1 predicts every slope too steep (Theorem~\ref{thm:crossing}).} (a--c) Class medians of mKNN at $k=10$ on the growing COCO evaluation sets, observed and as predicted from the measurements at the fit size $m=1024$. (d) Predicted against observed slope of $\ln s_N$ on $\ln m$ per pair, $\pm1$ s.d.\ over the five orderings of the images. Predicted and observed slope classes coincide where the grey bands cross or in the white lower left. The two long thin bars come from orderings where N1 is undefined.}
\label{fig:exp-scale}
\end{figure}

\paragraph{Sentence Encoders.}
We call the second preregistered experiment of
Section~\ref{sec:exp-scale} the sentence-encoder follow-up. This
follow-up pairs the five vision models with \mbox{all-mpnet-base-v2}
(MPNet), Contriever, and Qwen3-Embedding-0.6B at pinned revisions, and
each encoder embeds a caption with the pooling of its model card. Its
main arm keeps the evaluation sets, orderings, anchors, caption, and
clip of the 29 primary pairs, with a sentence encoder in place of the
language model of each vision--language pair. Eleven further arms
average up to five or all captions of each image, combine this averaging
with centering, center the features in place of the clip, use plain
$\ell_2$ normalization, or drop duplicate captions. We froze the plan
and then the code of the follow-up before the download of any
sentence-encoder weight. Before the first cross-model score, one
amendment let the frozen file-writing step reuse four tables already
written with identical content, and every rule and statistic of the plan
stayed unchanged. A checking script and the creation times of all 87
output files of the new pairs confirm this order. An independent reimplementation reproduced the
overlap counts in all $9{,}450$ cells of the main arm and of the five
arms with the first one to five captions, and all 41 labels. In the main
arm, $s_N$ falls with $m$ for all 15 pairs at $k=1$, $10$, and $100$,
and at $k=10$ their class median falls from $0.105$ at $m=256$ to
$0.013$ at $m=16{,}384$ (Figure~\ref{fig:app-e0}a). At $k=10$ and $m=1024$, the median $s_N$ of
these pairs is $0.059$, between the $0.043$ of vision--language pairs
with language models and the $0.189$ of vision pairs. The $s_N$ of each vision model with every sentence encoder exceeds its $s_N$ with any
language model.

\begin{figure}[htbp]
\centering
\includegraphics[width=5.5in]{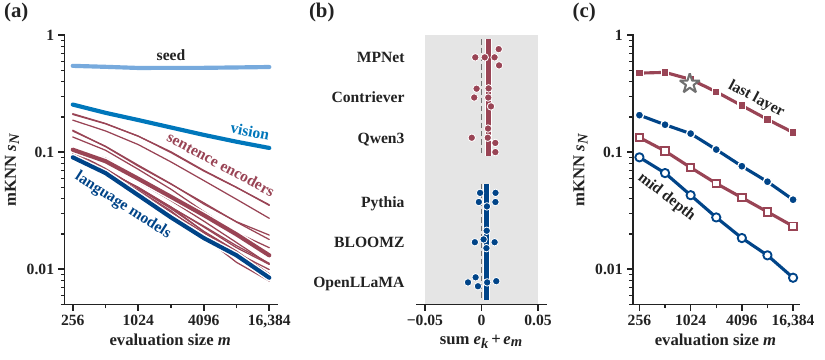}
\caption{\textbf{Neighbor overlap between vision models and sentence encoders falls with the evaluation size but stays close to stable at a fixed $k/m$ (Conjecture~\ref{conj:scale}), and in exploratory follow-ups vision features from the last layer in place of mid depth raise this overlap.} (a)~Class medians at $k=10$ as in Figure~\ref{fig:exp-neighbors}a, with sentence-encoder pairs as thin lines. (b)~Sums of the elasticities in $k$ and $m$ on $k\in\{10,100\}$, with group medians as bars and the preregistered tolerance in grey. (c)~Medians at $k=10$ with vision features from mid depth (open) or from the last layer (filled, exploratory), for sentence encoders with centering and five averaged captions (squares) and for language models (circles). The star marks the value of \citet{schnaus2026rosetta} for DINOv2-B and MPNet at equal $k/m$.}
\label{fig:app-e0}
\end{figure}

\paragraph{Stability with Sentence Encoders.}
The preregistered prediction of stability at a shared $k/m$ (\mbox{P-S1}
in Table~\ref{tab:exp-verdicts}) required at least 8 pairs whose
elasticity in $m$ has absolute value at least $0.1$. Among these pairs,
the median sum of the elasticities in $k$ and $m$ on $k\in\{10,100\}$
had to lie within $\pm0.03$, with at least $80\%$ of the sums within
$\pm0.05$. We fixed the range $k\ge10$ and these tolerances after
scoring the primary pairs and before any sentence-encoder score. The range leaves out $k=1$, where the $s_N$ of the primary pairs departs from
a function of $k/m$. Since the tolerances come from the sums of the
primary pairs, the primary vision--language pairs lie inside the band of
Figure~\ref{fig:app-e0}b by construction. All 15 new pairs qualify, with
elasticities in $m$ from $-0.60$ to $-0.39$, and their sums lie between
$-0.009$ and $0.016$ around a median of $0.006$. \mbox{P-S1} holds in
the main arm, in each of its five orderings, in all 11 further arms, and
for the primary pairs under centering. With $k=1$ included, the median
sum of the new pairs rises to $0.058$, against $0.030$ for the primary
vision--language pairs, and 5 of the 15 new sums stay within $\pm0.05$,
against 12 of 15 (exploratory). At $k=100$ and $m=256$ the $s_N$ of the
new pairs sits near chance, and on $m\ge2048$, where these values lie at
least $2.25$ times above chance, the median sum is $0.005$
(exploratory). The training data of MPNet include COCO caption pairs,
probably from the 2014 training images outside our validation pool
(unverified), and the pairs of MPNet have the largest median sum,
$0.012$. Without them, the sums keep a median of $0.006$ and all lie
within $\pm0.05$, so \mbox{P-S1} also holds (post hoc).

\paragraph{Undecided Slope Predictions.}
Two further preregistered predictions concern the slope at $k=10$
(\mbox{P-S2b} and \mbox{P-S2c} in Table~\ref{tab:exp-verdicts}).
\mbox{P-S2b} states that the slope of each new pair follows the line fitted on the primary pairs against the logarithm of the level, rather than the slope of the same vision model with the language
models. \mbox{P-S2c} states that averaging one to five captions moves the
slope of a pair with the logarithm of its level at the rate $0.228$ of this line. The other alternative is the rate $0.080$ seen within the classes of the primary pairs. The frozen plan counted each prediction as
decidable if its two alternatives lay far enough apart, a condition
called the guard. The guards required a median gap of at least $0.113$
between the two predicted slopes of \mbox{P-S2b} and a median rise of
the level by at least $35\%$ from the first caption to the average of
five for \mbox{P-S2c}. The gap reached $0.092$ and the rise $19\%$, so
both predictions carry the frozen label ``uninformative''. In the frozen operating characteristics of \mbox{P-S2b}, this label occurred in $84$ to $88\%$
of simulated runs at a median level near $0.06$. Whatever the label, the plan fixed two values for the report: the weight $0.44$ of the level line
in a mixture of the two predictions of \mbox{P-S2b} and the median rate
$0.435$ of \mbox{P-S2c}. The unmet guards leave both values without an
interpretation. Of the 12 secondary cells of \mbox{P-S2b}, 9 are
uninformative and one is inconclusive. The other two pass the guard by
less than $0.01$ and give ``holds'' with a weight of $0.96$. This weight comes from averaging slopes flatter than the line for three vision models with
steeper slopes for SigLIP-B/16.

\paragraph{Absolute Level and Depth of the Vision Features.}
\citet{schnaus2026rosetta} center each side and average the captions of
an image before they compare two models. At $k=10$ and $m=1024$,
centering raises the level of the new pairs by a median $9\%$ and the
average of five captions by $17\%$. With both changes, the $s_N$ of
DINOv2-B and MPNet reaches $0.077$ at $k=100$ and $m=10{,}000$, against
about $0.38$ read from the figure of \citet{schnaus2026rosetta} for the
same two models at $k=100$ on $10{,}000$ COCO images. We interpolate
every value at $m=10{,}000$ between our two largest evaluation sets, of
$8{,}192$ and $16{,}384$ images. In two exploratory follow-ups, each
with decision rules fixed before its first score, vision
features from the last layer in place of mid depth close the gap between
$0.077$ and about $0.38$. With the CLS token of DINOv2-B from its last
layer and every other setting kept, its $s_N$ with MPNet rises to
$0.429$. At $k=100$ and $m=10{,}000$, vision features from the last
layer raise the $s_N$ of all 15 pairs of a vision model and a sentence
encoder by a median factor of $4.3$ (range $2.3$ to $7.4$). We wrote the
rules of the second follow-up after the first follow-up had produced the
scores of the three last-layer pairs of DINOv2-B with the sentence
encoders. These rules include the three pairs, and the $s_N$ of the
other 12 sentence-encoder pairs also rises at the last layer. With
language models as the text side, last-layer vision features raise the
$s_N$ of all 15 pairs at $k=10$ and $m=1024$ by a median factor of
$2.8$, from $0.043$ to $0.144$. Across layers of the language model,
this level changes by at most a factor of $1.4$. At the last layer,
$s_N$ still falls
with $m$ at $k=10$ for all 30 pairs, with median slopes of $-0.42$ with language models and $-0.30$ with sentence encoders, against $-0.57$ and $-0.43$ at mid depth (Figure~\ref{fig:app-e0}c). Here the sentence-encoder pairs use centering and five averaged captions. The elasticities of these 30 pairs on
$k\in\{10,100\}$ sum to at most $0.02$ in absolute value. In a post hoc comparison at a shared $k/m$, the values for $k=100$ lie within $9\%$ of those for $k=10$. Neighbor overlap therefore stays close to stable at a fixed $k/m$ also at these higher levels. In a post hoc check, the line fitted on the primary
pairs and its counterpart for centered features predict too shallow slopes at the last layer. The predicted slopes are $-0.31$ against $-0.42$ observed with language
models and $-0.05$ against $-0.30$ with sentence encoders. Independent
reimplementations reproduced every overlap count and label of both
follow-ups. These follow-ups rest on COCO, five ViT-B models, and
language models of at most 3B parameters, and we did not score the
vision pairs at the last layer.

\subsection{Pretrained Models at Fixed Depth}
\label{app:exp-s12}

On WIT at fixed depth, we classify MST overlap with a preregistered spanning-tree test and split the squared-distance discrepancy by increasingly flexible fits as a preregistered diagnostic. A preregistered width test then checks whether model width drives the CKA trend, and all three analyses follow the rules of Table~\ref{tab:exp-verdicts}. Unless
stated otherwise, the trend $\mathrm{Tr}$ in this subsection is computed
on $s_{\rm obs}$ minus its permutation mean, not on $s_{\rm cal}$ of
Section~\ref{sec:protocol}. Its $p$-values are exact over the 120
relabelings of the BLOOMZ sizes, and where stated they are Holm-corrected
across four primary tests, i.e., the spanning-tree and width tests at the
two depths.

\paragraph{Spanning Trees.}
The primary evaluation set has $m=2047$ inputs, since the power gate of
the test required $m=3072$, larger than the set of extracted inputs
(Appendix~\ref{app:exp-prereg}). On this set, the forest share of the
shared MST edges exceeds the threshold of $0.70$ in every pair. With the other two conditions also met, the preregistered rule returns
``SOG-global $\approx$ 1-NN'' at both depths, i.e., it classifies MST
overlap as mostly one-nearest-neighbor overlap
(Figure~\ref{fig:exp-mst}). The forest share has medians $0.875$ and
$0.889$, against $0.850$ and $0.852$ expected if the shared edges were a
random subset of the language-model tree, and its minimum over pairs is
$0.76$ and $0.78$. MST overlap exceeds one-nearest-neighbor overlap by a
median factor of $1.20$ and $1.32$ at the two depths, in 129 and 136 of
the 136 pairs (Figure~\ref{fig:exp-neighbors}c). The bound of
Proposition~\ref{prop:mst} holds in all 816 configurations. Its half term $m\,s_{N,1}/2$ attains the maximum of the bound in every
configuration, so these pretrained pairs leave the branch
$m\,s_{N,1}-n^{\rm mut}_A$ of the bound untested. Bridge overlap exceeds
its permutation null in $93\%$ and $100\%$ of pairs, with medians
$0.016$ and $0.035$ against $0.001$. The bridge trend within BLOOMZ is $\mathrm{Tr}=-0.16$ at mid depth and $0.57$ at the last layer
($p=0.058$, Holm-adjusted $0.23$). At the last layer, the rule counts
the bridge trend as absent through a failure to reject. The point estimate $0.57$ exceeds the preregistered bar of $0.5$, but the Holm
correction across the four primary tests requires $p\le0.0125$. On the $m=1024$ inputs of the other
fixed-depth tests, the Holm-corrected outcome is also ``SOG-global $\approx$ 1-NN'' at both depths.
Without the correction, the last-layer bridge trend of $0.72$ ($p=0.017$) on these inputs would change
this outcome to ``genuine global convergence''. Two of
the three conditions of the rule hold nearly automatically. The forest
makes up a median $84$ to $85\%$ of each language-model tree and at least $82\%$ in every pair. This fraction is the forest share expected among the
shared edges under random relabeling, so a random subset of tree edges
would pass the threshold of $0.70$. With $s_{\Tset}$ and $s_{N,1}$
between $0.02$ and $0.05$, their difference stays within $0.03$ almost by
construction. The data add a modest enrichment over the forest fraction,
with forest edges shared at $1.25$ and $1.40$ times the bridge rate
(exploratory). The mean number of shared language-model bridges per
pair is $5.32$ at mid depth, below the power target of 10, and $10.72$ at the last layer. At this layer, the rule returns ``mixed'' when the forest of
the vision model replaces the primary forest of the language model.

\paragraph{Warps and Margins.}
In the preregistered diagnostics, we split the squared-distance
discrepancy into the shares explained by four increasingly flexible
fits, a scale fit, an affine fit, a global monotone fit, and one
monotone fit per row. Each share counts what its fit adds beyond the
previous fits. The shares are in-sample, since each fit is estimated and
scored on the same $m=1024$ inputs. Their medians are $0.681$, $0.190$,
$0.0004$, and $0.029$ at mid depth and $0.683$, $0.148$, $0.0012$, and
$0.034$ at the last layer, with residual shares $0.102$ and $0.123$. The share of the global monotone fit stays below the preregistered level
of $0.10$ in all 136 pairs at both depths. The bound of Equation~\eqref{eq:neighbor-stability} certifies 0
anchors at $k=10$ in all 544 configurations under every variant of the
bound. This count follows from the neighbor gap, which is a median
$0.0017$ and $0.0026$ times twice the largest discrepancy in its row. On these
pretrained pairs, neighbor overlap therefore depends on the typical
crossings of Theorem~\ref{thm:crossing} rather than on the worst-case
margin.

\paragraph{Preprocessing and Families.}
With the clip of PRH and after subtraction of the permutation mean, mKNN
and CKA rise with size at the last layer along BLOOMZ and LLaMA-1,
whereas the scores with OpenLLaMA 7B exceed those with OpenLLaMA 3B for
$12\%$ of the vision models on mKNN and $29\%$ on CKA (exploratory).
With 13B added, the OpenLLaMA trends of $s_{\rm cal}$ at the last layer
are $0.59$ for mKNN and $0.56$ for CKA
(Table~\ref{tab:app-layermax-openllama}), so this difference between
families rests on the step from 3B to 7B. At mid depth, the BLOOMZ
trends with the clip are $0.70$ for CKA (exact one-sided $p=0.117$) and
$0.84$ for mKNN ($p=0.025$). Under plain $\ell_2$ they are $-1.00$
(two-sided $p=0.017$) and $-0.92$ ($p=0.033$), and CKA falls strictly
with size for 17 of the 17 vision models
(Figure~\ref{fig:exp-criteria}b,c). These $p$-values are uncorrected for multiplicity. Each line of Figure~\ref{fig:exp-criteria}b,c is divided by its value at 0.56B. Before this division, the means over vision models range from $0.113$ to $0.130$ for CKA and from $0.034$ to $0.041$ for mKNN with the clip. Under plain $\ell_2$ normalization the ranges are $0.007$ to $0.014$ and $0.018$ to $0.024$. On these models, a size trend thus depends
on the model family, the layer, and the preprocessing together.

\paragraph{Width.}
In the preregistered width test, we project every model to 128 principal
components. With the clip of PRH, this projection changes the
fixed-depth trend of CKA minus its permutation mean along BLOOMZ by
$0.00$ at mid depth and by $-0.07$ at the last layer, from $0.92$
($p=0.017$) to $0.85$ ($p=0.025$). Under plain $\ell_2$ normalization,
the projection leaves both trends unchanged. Width therefore does not drive this trend, in line with
Proposition~\ref{prop:baselines}, which ties the CKA baseline to
effective dimension rather than width. The preregistered rule
returns ``genuine global loss'' at both depths and under both
preprocessings. With the clip of PRH, the claim of this name is false
as worded, since CKA minus its permutation
mean keeps its fixed-depth trend at $0.70$ ($p=0.117$) and $0.92$
($p=0.017$).

\begin{figure}[htbp]
\centering
\includegraphics[width=5.5in]{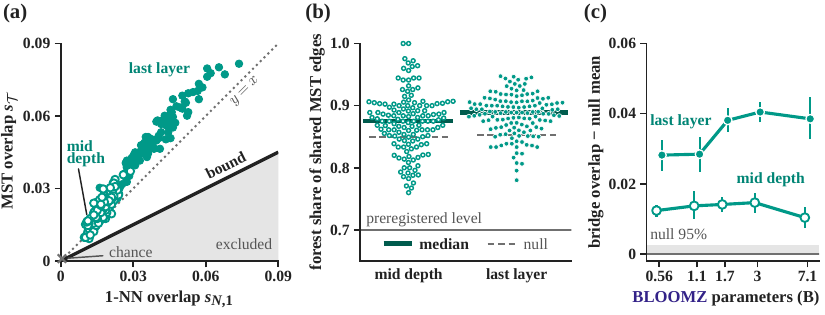}
\caption{\textbf{Spanning-tree overlap mostly depends on the one-nearest-neighbor scale, and bridge overlap exceeds its permutation null without an established size trend (Proposition~\ref{prop:mst}).} Vision--language pairs of WIT at mid depth (open) and the last layer (filled). (a) MST overlap against one-nearest-neighbor overlap for each pair, with chance as a cross. (b) Per pair, the share of the shared MST edges that lie in the nearest-neighbor forest of the language model, with its value for a random subset of tree edges (null). (c) Bridge overlap minus its permutation null mean along BLOOMZ, as means over vision models with $95\%$ intervals. The grey band ends at the median $95\%$ permutation quantile. The experimental record gives the values.}
\label{fig:exp-mst}
\end{figure}

\subsection{Layer Aggregation}
\label{app:exp-layermax}

This subsection gives the details of the layer-maximum experiments of
Section~\ref{sec:exp-aggregation} along BLOOMZ, LLaMA-1, OpenLLaMA, and
Pythia, each family compared with the 17 vision models of ARH on the
1{,}024 WIT pairs. For each vision--language pair we evaluate every layer
pair with $n_\pi=500$ permutations and recompute the maximum inside each permutation, which applies the same relabeling to every layer. The
preregistered experiment applied this protocol of ARH to the 136 pairs
of the vision models with the eight language models of
Appendix~\ref{app:exp-protocol}, in $80{,}960$ layer pairs and under decision rules fixed in a
frozen plan. We derived Corollary~\ref{cor:aggregation} after seeing the
BLOOMZ result, so the threshold analyses on BLOOMZ are post hoc. For
LLaMA-1 13B and 30B, OpenLLaMA 13B, and seven Pythia models, a plan for
each family fixed predictions from single-model features before any
score involving the new models. We compare the LLaMA-1 outcomes with
their predictions without a confirmatory test, and the OpenLLaMA plan
declared every outcome descriptive. Before the download of any Pythia
weight, the Pythia plan fixed a confirmatory test (H1) of a drop of at
least $0.10$ in the layer-max CKA trend under calibration. In each family two independent implementations agree on
every decision, with floating-point differences of at most $10^{-14}$ along BLOOMZ,
LLaMA-1, and OpenLLaMA. When the second implementation
reproduces the frozen LLaMA-1 7B results, the relative deviation
exceeds the strict tolerance of $10^{-12}$ in one value, a debiased-CKA permutation mean with expectation zero. Its absolute deviation of
$6.7\times10^{-18}$ lies within the mixed tolerance of the plan. The
permutation means of biased CKA match the baseline of part~2 of
Proposition~\ref{prop:baselines} with median relative errors between
$0.07\%$ and $0.09\%$ in every family. On the new LLaMA-1 and OpenLLaMA layer pairs, the permutation
means of debiased CKA are consistent with the zero baseline of part~3.
In Figures~\ref{fig:exp-aggregation}d--i
and~\ref{fig:app-aggregation}d--i, every line and marker gives a mean
over the 17 vision models, and the baseline of a language-model layer is its largest baseline over the layers of a vision model. The measured baselines in Figures~\ref{fig:exp-aggregation}g--i and~\ref{fig:app-aggregation}g--i agree with the predicted baselines to a relative error of at most $4.4\times10^{-4}$. Table~\ref{tab:app-layermax} gives the size trends of the four families.

\begin{figure}[htbp]
\centering
\includegraphics[width=5.5in]{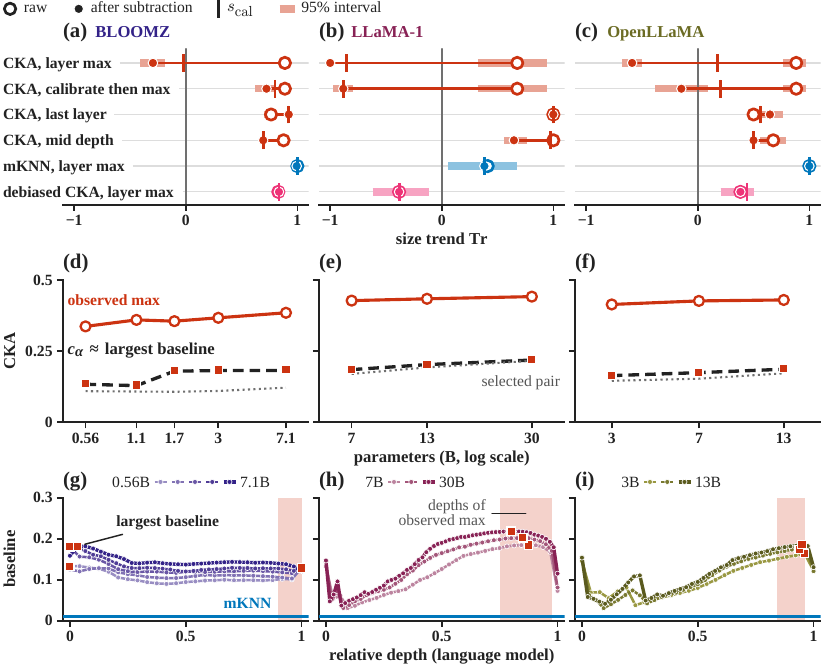}
\caption{\textbf{After a maximum over layers, the calibration threshold follows the largest random-pairing baseline along BLOOMZ, LLaMA-1, and OpenLLaMA (Proposition~\ref{prop:baselines}, Corollary~\ref{cor:aggregation}).} The trends of debiased CKA and the BLOOMZ trends of calibration before the maximum are exploratory. (a--c)~Size trend of each score in its raw form, after subtraction of the permutation mean, and under $s_{\rm cal}$. (d--f) Squares give the threshold $c_\alpha$, dashed lines the largest baseline predicted from features, and dotted lines the threshold of the layer pair attaining the maximum (``selected pair''). (g--i) CKA baseline of each language-model layer, predicted (lines) and measured (dots). Table~\ref{tab:app-layermax} gives the trends.}
\label{fig:app-aggregation}
\end{figure}

\paragraph{BLOOMZ.}
Along BLOOMZ, the layer-max CKA trend of $0.89$ ($p=0.042$) falls to
$-0.29$ after subtraction of the permutation mean ($p=0.68$; $95\%$
interval over vision models $[-0.41,-0.18]$) and to $-0.02$ under
$s_{\rm cal}$ ($[-0.19,0.12]$). After subtraction, the layer-max trends
of mKNN and debiased CKA are $0.99$ ($p=0.008$) and $0.84$ ($p=0.033$,
exploratory), the second with a zero baseline at every layer pair by
part~3 of Proposition~\ref{prop:baselines}. At fixed depth the threshold
rises with size at mid depth and falls at the last layer
($\mathrm{Tr}=0.99$ and $-0.91$). Calibration moves the CKA trend
opposite to the trend of the threshold without a reversal. In post hoc
analyses, the threshold $c_\alpha$ of the maximum lies $0.0017$ to
$0.0028$ above the largest predicted baseline
$\max_{\ell,\ell'}\sqrt{\PR_A\PR_B}/(m-1)$ over the 85 BLOOMZ pairs, as assumed in Corollary~\ref{cor:aggregation}. For every size except 1.1B this
baseline sits in the first two language-model layers, where the
participation ratio jumps from 134 at 0.56B to between 242 and 249 at
1.7B and above. In contrast, 112 of the 136 observed maxima of the
preregistered pairs fall at relative depth $0.9$ or above, so on BLOOMZ the
maximum and its threshold mostly come from different layer pairs. From 0.56B to
7.1B the mean threshold rises from $0.135$ to $0.184$, a gain of $0.048$
equal to that of the observed maximum ($0.336$ to $0.384$) and close to
the $0.049$ of the largest baseline. The threshold of the layer pair attaining the maximum stays between $0.107$ and $0.121$. The
calibrated score therefore loses its upward trend, the outcome allowed
by the corollary when the baselines grow as fast as the score. Two
further analyses support the corollary on BLOOMZ, the first fixed in the
frozen plan and the second post hoc. A subtraction of the threshold of the
layer pair attaining the maximum gives $\mathrm{Tr}=0.76$. A restriction of the maximum to language-model layers at relative depth
$0.5$ or more restores the trend after subtraction to $0.89$
($p=0.017$).

\begin{table}[htbp]
\centering\small
\setlength{\tabcolsep}{3.5pt}
\begin{tabular}{@{}llrrrrrrrr@{}}
\toprule
 & & \multicolumn{3}{c}{layer maximum} & \multicolumn{2}{c}{mid depth} & \multicolumn{2}{c}{last layer} & \multicolumn{1}{c}{calibrate} \\
\cmidrule(lr){3-5}\cmidrule(lr){6-7}\cmidrule(lr){8-9}
Family & Metric & raw & $s_{\rm cal}$ & $s_{\rm obs}-\overline{s}_{\pi}$ & raw & $s_{\rm cal}$ & raw & $s_{\rm cal}$ & then max \\
\midrule
BLOOMZ & CKA & 0.89 & $-$0.02 & $-$0.29 & 0.88 & 0.70 & 0.76 & 0.92 & 0.80 \\
 & mKNN & 1.00 & 1.00 & 0.99 & 0.85 & 0.84 & 0.99 & 0.98 & -- \\
 & debiased CKA & 0.83 & 0.84 & 0.84 & 0.70 & 0.70 & 0.90 & 0.91 & -- \\
\addlinespace[2pt]
LLaMA-1 & CKA & 0.68 & $-$0.85 & $-$1.00 & 1.00 & 0.97 & 1.00 & 1.00 & $-$0.88 \\
 & mKNN & 0.41 & 0.38 & 0.38 & 0.82 & 0.82 & 0.91 & 0.91 & 0.38 \\
 & debiased CKA & $-$0.38 & $-$0.38 & $-$0.38 & 1.00 & 1.00 & 1.00 & 1.00 & $-$0.38 \\
\addlinespace[2pt]
OpenLLaMA & CKA & 0.88 & 0.18 & $-$0.59 & 0.68 & 0.50 & 0.50 & 0.56 & 0.21 \\
 & mKNN & 1.00 & 1.00 & 1.00 & 0.68 & 0.68 & 0.56 & 0.59 & 1.00 \\
 & debiased CKA & 0.38 & 0.44 & 0.38 & 0.50 & 0.50 & 0.59 & 0.53 & 0.38 \\
\addlinespace[2pt]
Pythia & CKA & 0.93 & 0.89 & 0.89 & 0.98 & 0.96 & 0.67 & 0.80 & 0.89 \\
 & mKNN & 1.00 & 1.00 & 1.00 & 0.99 & 0.99 & 0.94 & 0.95 & 1.00 \\
 & debiased CKA & 0.90 & 0.90 & 0.90 & 0.97 & 0.96 & 0.84 & 0.84 & 0.90 \\
\bottomrule
\end{tabular}
\caption{Size trends $\mathrm{Tr}$ along the four families under the clip
and $\ell_2$ normalization of PRH, for the layer maximum, the two fixed depths, and $s_{\rm cal}$ under calibration of each layer pair before the maximum. The
column $s_{\rm obs}-\overline{s}_{\pi}$ gives the trend of $s_{\rm obs}$
minus the mean of its 500 permutation values, and the fixed-depth trends
use the permutations drawn for the layer maximum. The experimental record
gives every exact one-sided $p$. Along Pythia, the predicted drop of at least
$0.10$ in the layer-max CKA trend under calibration (H1) was not
confirmed, and every other Pythia entry and every OpenLLaMA entry are
descriptive. The debiased CKA rows and the BLOOMZ
value of calibration before the maximum are exploratory. A dash marks a trend not reported.}
\label{tab:app-layermax}
\label{tab:app-layermax-llama}
\label{tab:app-layermax-openllama}
\label{tab:app-layermax-pythia}
\end{table}

\paragraph{LLaMA-1.}
Along LLaMA-1, the 7B, 13B, and 30B models give 51 pairs with
$47{,}520$ layer pairs. Computed from single-model participation ratios
before any cross-model score, the largest predicted baselines have means
$0.1847$, $0.2023$, and $0.2182$ and sit at relative depth $0.80$ to
$0.88$, where the observed maximum also lies. The thresholds lie a median $0.0018$ above these baselines
(range $0.0015$ to $0.0021$) and follow their order in 51 of 51
comparisons. From 7B to 30B the threshold of the maximum gains $0.033$
against $0.014$ for the observed maximum. Calibration turns the layer-max CKA trend of $0.68$ into $-0.85$ under $s_{\rm cal}$ ($95\%$
interval $[-1.00,-0.65]$) and into $-1.00$ after subtraction, while
mKNN moves from $0.41$ to $0.38$. The threshold of the layer pair
attaining the maximum lies $0.004$ below the threshold of the maximum at
30B, so the two analyses that restored the trend on BLOOMZ leave the
LLaMA-1 trend reversed. A subtraction of this threshold gives
$\mathrm{Tr}=-0.91$, and a restriction to layers at relative depth $0.5$
or more leaves $s_{\rm cal}$ at $-0.85$. At
fixed depth the calibrated CKA trend keeps its sign, as predicted, at
$0.97$ at mid depth and $1.00$ at the last layer. We had marked this
prediction as at risk, since the mid-depth baselines also rise with size
(means $0.050$, $0.065$, and $0.075$). The raw scores rise faster, from $0.184$ to $0.227$ between 7B and 13B, against a threshold rise from
$0.052$ to $0.067$. Even so, subtraction of the permutation mean lowers
the mid-depth trend from $1.00$ to $0.65$. Calibration of each layer pair before the maximum
reverses the trend to $-0.88$, although we had predicted in advance that
the trend would keep its sign under this protocol. The layer pairs
selected under this protocol stay on the late block, where their own
thresholds rise from $0.146$ to $0.167$ with size. On BLOOMZ, where the
largest baseline sits away from the layers of the maximum, this protocol
keeps the trend, at $0.80$ for $s_{\rm cal}$ and $0.72$ after subtraction. This result is exploratory, since we defined the protocol in the
LLaMA-1 plan after seeing the BLOOMZ result.

\paragraph{OpenLLaMA.}
\phantomsection\label{app:exp-openllama}%
Along OpenLLaMA, the 3B, 7B, and 13B models give 51 pairs, with
$14{,}432$ layer pairs for the new 13B model. The largest participation
ratio over layers grows from 202 at 3B through 229 at 7B to 261 at 13B,
and the largest baseline grows with it for each of the 17 vision models. This baseline has means $0.1636$, $0.1741$, and $0.1860$ and sits on late layers at relative depth $0.94$ to $0.96$. The offsets of the thresholds above this
baseline have a median of $0.0019$ (range $0.0016$ to $0.0023$) and lie
inside the frozen band of $0.0009$ to $0.0029$ for 17 of 17 vision
models. From 3B to 13B the threshold and the largest baseline gain
$0.022$, against $0.016$ for the observed maximum. The CKA trend of
$0.88$ falls to $0.18$ under $s_{\rm cal}$ and to $-0.59$ after
subtraction ($95\%$ interval $[-0.68,-0.50]$). Although the lines of
Figure~\ref{fig:exp-aggregation}f are nearly parallel, the trend after
subtraction is strongly negative, because $\mathrm{Tr}$ averages
per-vision-model rank correlations over three sizes and the small excess
of baseline growth recurs in all 17 vision models. From 7B to 13B the
raw gain falls below the threshold gain for 17 of 17 vision models, with
medians of $0.0018$ against $0.0130$, an outcome outside the predictions
fixed in advance. Given the raw scores, the thresholds forecast from features
yield a predicted order after calibration in 30 of the 34 comparisons involving
OpenLLaMA 13B, all correct. These 30 cover 16 of the 17 orders
reversed by calibration, 10 of the 13 kept orders, and 4 comparisons in
which 13B already scores lower before calibration. As predicted, the mKNN trend stays at $1.00$ under
calibration and the debiased CKA trend moves by $0.06$ or less.
Calibration before the maximum lowers the CKA trend from $0.88$ to
$0.21$, in line with the prediction that this protocol does not raise the trend. At both fixed depths the CKA trend keeps its sign under
calibration (descriptive). With an uncensored mKNN trend of $1.00$, the
decision rule of the plan assigns OpenLLaMA the category ``layer max
loses the CKA trend'', whereas the mKNN trend of $0.38$ left LLaMA-1 in
the category ``intermediate''.

\paragraph{Pythia.}
\phantomsection\label{app:exp-pythia}%
Along the seven Pythia models from 70M to 6.9B, calibration kept the
layer-max CKA trend, from $0.93$ ($p=0.0020$) to $0.89$ ($p=0.0042$)
under $s_{\rm cal}$. The primary preregistered prediction of a drop of at least $0.10$ consistent
across the sizes was not confirmed. The trend changes by $-0.034$ under
calibration ($95\%$ interval over vision models $[-0.046,-0.023]$). The permutation $p$ is $1612/5040$ for the consistency of the drop across sizes, and the bootstrap $p$ is $1.0$ for a drop of at least $0.10$. Since
the plan made the secondary tests H2 to H4 conditional on a confirmed
primary prediction, we did not carry out these tests. To forecast the offset of each threshold above its
predicted baseline, the plan fixes for model $i$ the band
$f(\PR^*_i)\pm0.0008$, where $\PR^*_i$ is the largest participation ratio
over the layers of the language model. The function $f(\PR^*)=0.001069+0.01292/\sqrt{\PR^*}$ is fitted on the 187 known pairs
of the 11 language models of the other families. The values of $\PR^*$
lie between $91.5$ at 70M and $166.6$ at 6.9B, so the bands of 70M and
160M extrapolate below the fitted range, which starts at $124.5$. The
median offsets lie between $0.0020$ and $0.0021$, inside the band of
every model, and the threshold order matches the predicted order in 307
of the 307 comparisons with a material predicted gap. Given the raw
scores, the bands yield a predicted order after calibration in 352 of
the 357 comparisons, all correct, including 11 of the 14 orders reversed
by calibration. The raw score of the layer maximum gains a median
$0.024$ per doubling of parameters against $0.0063$ for the threshold.
With one size or one of the four vision-model families left out, or
without the censoring of $s_{\rm cal}$, the change of the trend stays
between $-0.054$ and $-0.017$. The five sizes from 410M to 6.9B give the
largest drop, $-0.065$ ($p=0.35$), so every subset stays short of the
required $0.10$. The
category returned by the decision rule stays the same for every required
drop from $0.06$ to $0.20$ and changes at $0.03$. The two
implementations agree on every integer, $p$-value, decision, and
category and differ in 163 pair-level and 3 descriptive floating-point
values by at most $9.1\times10^{-14}$, so the preregistered agreement
check failed by its letter. After seeing the category, we accepted this
check with an added absolute tolerance of $10^{-12}$, a post hoc ruling,
and a third implementation reproduces the results of both with 0
mismatches. The plan fixed the wording of
Section~\ref{sec:exp-aggregation} for each outcome before any Pythia
weight was downloaded, and that section reports a subset of these
numbers in a different order, a post hoc choice of wording.

\paragraph{Scope.}
With three sizes, every exact one-sided $p$ of a size trend along
LLaMA-1 or OpenLLaMA is at least $1/6$, so each outcome along these
families is a directional result on one family. Their intervals cover
the choice of vision models but not of language models. On both families
the thresholds follow the predicted order in every comparison, but this
order is the same for each of the 17 vision models, since the vision
model enters the predicted baseline as a common factor. Once the thresholds track the largest baseline, two predictions hold almost by construction, because $s_{\rm cal}$ does not increase with the threshold. These are the OpenLLaMA forecast of the order after calibration and part of the LLaMA-1 prediction that calibration lowers the CKA trend. A failure of these predictions would come
from the raw gain, which fell below the threshold gain for 17 of 17
vision models at each pair of LLaMA-1 sizes (medians $0.006$ against
$0.019$ from 7B to 13B). Joint permutation tests over families are
exploratory. Under calibration, the mean of the BLOOMZ and LLaMA-1
trends falls by $1.22$ for $s_{\rm cal}$ and by $1.43$ after
subtraction ($p=6/720$ and $2/720$). The test on the weakest drop along BLOOMZ, LLaMA-1, and OpenLLaMA gives $p=48/4320$ and $2/4320$. The BLOOMZ result generated the tested hypothesis, and the trend of the largest BLOOMZ baseline fails the materiality rule of the plan. The BLOOMZ and LLaMA-1 results were known before the OpenLLaMA plan. With the BLOOMZ and LLaMA-1 drops held at
their observed values, a rule rejecting at $p\le0.05$ rejects under 2 of the 6
relabelings of the OpenLLaMA sizes, so the three-family value carries
little evidence about OpenLLaMA itself. Depth and width grow with
size along LLaMA-1, whose 30B model was also trained on more tokens
\citep{touvron2023llama}, and along OpenLLaMA, whose three models share
one corpus of 1T tokens \citep{openlm2023openllama}. The seven
Pythia models share their training data and its order
\citep{biderman2023pythia}, and their width grows from 512 to 4096,
whereas their depth does not grow monotonically (16 blocks at 1B against
24 at 410M).

\paragraph{Secondary Preprocessing.}
Under plain $\ell_2$ normalization, our secondary preprocessing, the
largest baseline sits at the embedding layer for all three sizes of
LLaMA-1 and of OpenLLaMA. Before any score, both plans declared this preprocessing not material, since its predicted baselines differ too little between sizes to support a directional prediction. They differ by $0.0023$ between LLaMA-1 13B and 30B, and the participation ratio of $178.45$ at OpenLLaMA 13B lies below the $180.02$ of 7B. The analysis of this
preprocessing is therefore descriptive, as in the Pythia plan. With plain $\ell_2$ normalization,
calibration moves the layer-max CKA trend from $-0.50$ to $-0.56$ on
LLaMA-1, from $-0.26$ to $-0.62$ on OpenLLaMA, and from $0.68$ to
$-0.06$ on Pythia ($p_{\rm perm}=47/5040$). On BLOOMZ the uncensored trend is $-0.09$ against $0.43$ before calibration (post hoc). Along Pythia the raw score gains $0.0015$ per doubling against $0.0049$ for the threshold, in line with the condition of part~2 of Corollary~\ref{cor:aggregation}. The drop rests on 6.9B, without which $p_{\rm perm}$ is $0.12$. On LLaMA-1,
the prediction that the calibrated CKA trend keeps its sign at fixed
depth holds at the last layer and does not apply at mid depth,
where the raw trend is $0.03$. Over all 187 pairs of the 11 language models, the
difference between the drops of CKA and mKNN in the pooled statistic of
ARH is $0.081$ under this preprocessing, with interval $[0.006,0.165]$.
The OpenLLaMA plan excluded this secondary result from the comparison
with ARH (Appendix~\ref{app:exp-published}).

\subsection{Comparison with the Published Results}
\label{app:exp-published}
\label{sec:exp-related}

On the models shared with PRH, ARH, and BPC, our measurements agree with
a central observation of each study. With the preprocessing of PRH,
layer-max mKNN rises with BLOOMZ size ($\mathrm{Tr}=1.00$ before
calibration), in line with the cross-modal trend of
\citet{huh2024platonic}. The loss and reversal of the calibrated
layer-max CKA trend along BLOOMZ, LLaMA-1, and OpenLLaMA are in line with the loss reported by ARH. ARH plots this loss for RBF kernel alignment
in its main figure and reports the same loss for linear CKA, the version
of CKA in our measurements. Along Pythia, whose raw scores outgrow their
baselines, the trend stays at $0.89$ (Appendix~\ref{app:exp-pythia}).
BPC reports that mKNN at $k=10$ for its
primary vision--language pair falls from $0.135$ on $1{,}024$ WIT pairs
to $0.008$ on fifteen million LAION pairs, with the layer pair fixed.
On our COCO evaluation sets, the median vision--language mKNN at $k=10$
falls in the same direction, from $0.091$ to $0.0085$ over a 64-fold
range of $m$. SOG
reports that spanning-tree overlap tracks mKNN with a Spearman
correlation of $0.968$ over 204 pairs \citep{you2026structure}. By
Proposition~\ref{prop:mst}, the MST contains the nearest-neighbor forest
for pairwise distinct distances, and on our models most shared MST edges
are shared nearest-neighbor edges, so this correlation comes at least
partly from the nearest-neighbor forest. On our models, MST overlap and
one-nearest-neighbor overlap have a correlation of $0.96$ over both
depths (exploratory), computed from uncalibrated scores at fixed depth
rather than from the calibrated layer-max scores of SOG.

Our results differ from those of these studies in three respects. First, on our models the pooled statistic of ARH changes less under calibration than on the models of ARH. There, this Pearson correlation between score and language-model rank moves from $0.86$ to $0.45$ for CKA and from
$0.85$ to $0.84$ for mKNN. Over our 136 preregistered pairs it moves
from $0.49$ to $0.43$ for CKA and stays at $0.52$ for mKNN, and the
preregistered decision rule returns ``ARH pattern absent''. With LLaMA-1
13B and 30B added, the difference between the two drops is $0.12$
($95\%$ interval over vision models $[0.003,0.26]$), in the category
``partial''. With OpenLLaMA 13B also added, the 187 pairs give $0.50$ to
$0.40$ and $0.51$ to $0.51$, and the difference of $0.10$ has the interval $[-0.02,0.24]$. This interval contains zero and covers the
boundary $0.10$ of the decision rule, so the direction of the loss reported by ARH is not established on our models. With 187 pairs the category of the rule is not decidable. The corresponding difference
of $0.40$ in ARH lies above our interval, and our models do not include
LLaMA-65B. ARH also reports a loss for unbiased CKA, from $0.65$ to
$-0.01$, whereas our debiased CKA uses the same estimator
(Appendix~\ref{app:exp-criteria}) and moves from $0.37$ to $0.36$ over
the 187 pairs. Its zero baseline by part~3 of
Proposition~\ref{prop:baselines} places debiased CKA under part~1 of Corollary~\ref{cor:aggregation}, by which calibration does not reverse the order of model pairs with well-separated maxima (Section~\ref{sec:calibration}). The source of the difference from ARH remains open. In a secondary analysis along vision-model sizes at the last layer, CKA minus its permutation
mean has $\mathrm{Tr}=-0.04$ ($p=0.55$) against $0.71$ ($p=0.005$) for
mKNN, so the contrast reported by ARH appears there at fixed depth.
Such a loss can come
from baselines of the fixed vision layers that grow with size faster than
the score, as in part~2 of Corollary~\ref{cor:aggregation}, and we leave this explanation untested. Second, ARH attributes the
inflation of layer-max scores to the number of layer pairs
\citep{groger2026revisiting}. In a post hoc analysis of biased CKA, the permutation mean of the maximum exceeds the largest single-pair mean by $6\times10^{-5}$ on average over our 136 pairs and by at most $10^{-5}$ on LLaMA-1. There, the trend of a maximum over 17 layers of each model
stays reversed ($-0.91$ under $s_{\rm cal}$). Over the 136 pairs, the threshold of the maximum exceeds that of the layer pair attaining the maximum by $0.040$, so the
level of the highest-baseline layer pair sets the threshold, rather than
the number of layer pairs (exploratory). Third, the BLOOMZ trends at
mid depth reverse under plain $\ell_2$ normalization
(Appendix~\ref{app:exp-s12}), so their sign depends on the clip of PRH,
and preprocessing belongs to the criterion.

Our controlled runs add two results to this comparison. In the CIFAR-5m
networks with a linear head, the linear map of the head adds an excess
to the geometric discrepancy without lowering agreement on neighbor sets. Weight decay on the head removes this excess
(Appendix~\ref{app:exp-s22}). These results suggest one candidate
mechanism for why representations agree on neighbor sets but not on
global geometry, a question left open by ARH. Under momentum without weight decay, most
of this excess on $G_s$ at $250{,}000$ images comes from a gauge shared
by all runs trained this way (preregistered,
Appendix~\ref{app:exp-s22-v11}). Our contrastive networks
trained with the two objectives also disagree on neighbor sets
(Appendix~\ref{app:exp-s21}), so these networks are not described by the
idealized contrastive learner of PRH, which recovers one kernel up to an
offset.

\subsection{Linear Autoencoder} \label{app:exp-s13a}

In the linear-autoencoder experiment of Section~\ref{sec:exp-gauge}, the inputs $x\in\R^{16}$ are Gaussian with mean zero and covariance $U\Lambda U^\top$, $\Lambda=\mathrm{diag}(4,3,2,1.5,0.2,\dots,0.2)$ with twelve eigenvalues $0.2$, and $U$ a fixed random orthogonal matrix. The latent dimension is $n_h=4$, and $m=256$ probes drawn independently of every training stream form the rows
of $\mathbf X$ and define the probe Gram matrix
$K=\mathbf XE^\top E\mathbf X^\top$ of Theorem~\ref{thm:gauge}. Every
criterion uses the unnormalized features $Ex$ of the probes, without the
unit normalization of Section~\ref{sec:criteria}, and the training loss
uses the second-moment matrix $N^{-1}\sum_ix_ix_i^\top$ of the run's
sample without centering, or $U\Lambda U^\top$ at $N=\infty$. Runs start from $\mathcal N(0,1/\text{fan-in})$, or from entries of scale $0.01$ under small initialization, and train until the relative gap to the closed-form optimum or the relative gradient falls below $10^{-8}$. The weight-decay arms use $\lambda_{\mathrm{wd}}=0.075$. Runs under stochastic gradient descent and Adam use minibatches of 16 for $4\times10^5$ steps at a constant rate and $10^5$ steps of decay and then end with full-batch gradient steps under the same stopping rule. For each training procedure we trained 48 disjoint pairs of runs at each $N\in\{2^8,2^{12},2^{16}\}$ and 500 pairs at $N=\infty$, the population limit. Run $r$ keeps its initialization at every $N$ and draws its own training sample and minibatches at each $N$. The tolerances are $0.05$ for $1-s_N$ at $k=10$, $0.02$ for $G_s$, and $0.01$ for $1-\CKA$ and for the projector distance, and a pair with a diverged run fails every criterion. The experimental record gives the decision table behind Figure~\ref{fig:exp-linae}.

\paragraph{Gradient Flow and Other Optimizers.} The preregistered gradient-descent arm uses a step of $0.05$, which changes the imbalance $C$ of part~1 of Theorem~\ref{thm:gauge} by a median $53\%$. A secondary arm declared before the confirmatory runs uses a step of $0.002$ and conserves $C$ to $0.18\%$, so this arm stands in for gradient flow. In that arm the spectrum of the learned $EE^\top$ equals $f(\mathrm{spec}\,C_0)$ of part~1 to a median relative error of $0.1\%$. The median of $\mathrm{cond}(EE^\top)$ is $4.66$ against $4.67$ predicted by part~1, with 10th and 90th percentiles of $2.98$ and $7.43$ across initializations. At both steps, 0 of 48 pairs of gradient-descent runs at every finite $N$ and 0 of 500 at $N=\infty$ lie within tolerance on 10-nearest-neighbor sets, $G_s$, and CKA. Under weight decay, tied weights, small initialization, and weight decay at the secondary step $0.002$, at least 47 of 48 pairs at $N=2^{12}$ and $2^{16}$ and 488 of 500 at $N=\infty$ lie within tolerance on all four criteria. At $N=2^8$ the count on neighbor sets and $G_s$ is 0 of 48 under every training procedure. Under weight decay, the median condition number of the learned $EE^\top$ is $1.0330$, against $1.0329$ predicted by part~3 of the formal version, a comparison made for Figure~\ref{fig:exp-linae}e outside the preregistered tests. Tied weights, small initialization, and weight decay use the preregistered step of $0.05$, since part~3 holds at any step and small initialization keeps the imbalance $C$ of part~1 near zero. Minibatch stochastic gradient descent and Adam change $C$ during training, outside the premises of the theorem, and the learned $EE^\top$ has intermediate median condition numbers of $1.19$ and $1.54$. Agreement between their runs on neighbor sets, $G_s$, and CKA stays unreliable at every $N$. At $N=\infty$, 164, 12, and 322 of 500 pairs lie within tolerance under stochastic gradient descent and 13, 0, and 172 under Adam. Agreement under the projector criterion is reliable for $N\ge2^{12}$ under every training procedure. Each pair failing at these sizes under weight decay or on the projector criterion contains a run that diverged at step $0.05$. The exceptions are two pairs of runs trained with stochastic gradient descent at $N=\infty$, whose runs diverged in the minibatch phase.

\begin{figure}[htbp]
\centering
\includegraphics[width=5.5in]{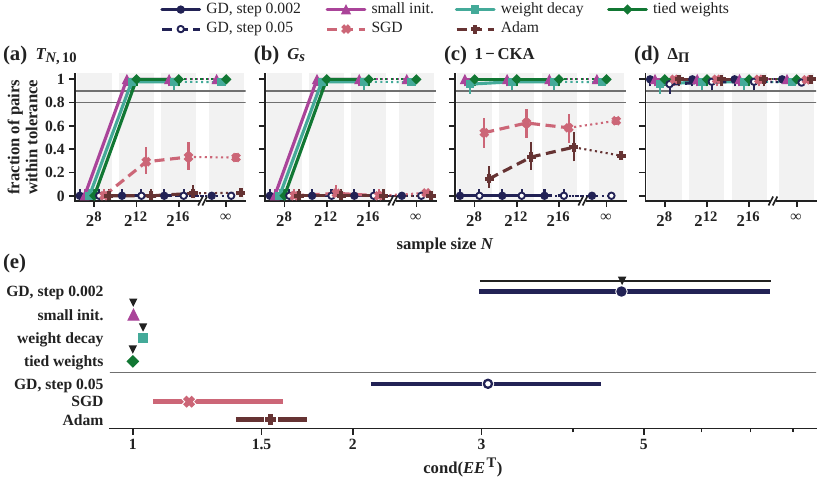}
\caption{\textbf{On the criteria whose outcomes depend on the gauge, agreement between runs trained by gradient descent without a regularizer stays unreliable at every sample size, whereas agreement between runs trained with weight decay, tied weights, or small initialization becomes reliable as $N$ grows (Theorem~\ref{thm:gauge}).} (a--d)~Fraction of disjoint pairs of linear-autoencoder runs within tolerance against $N$, with one-sided $95\%$ Clopper--Pearson bounds and grey decision levels. Dashed lines mark training procedures outside the premises of the theorem, and dotted segments join $N=2^{16}$ to $N=\infty$. The projector criterion~(d) is a control whose outcome is invariant to the gauge. (e)~Condition number of the learned $EE^\top$ at $N=\infty$ as the median and the 10th to 90th percentile of runs, with the predictions of parts~1 and~3 for the training procedures covered by the theorem in black.}
\label{fig:exp-linae}
\end{figure}

\paragraph{Failed Prediction P1.} For gradient descent at step $0.05$, P1 required unreliable agreement on neighbor sets, $G_s$, and CKA at every $N$, and a mean discrepancy at $N=2^{16}$ inside the two-sided $95\%$ interval of the mean at $N=\infty$. The first part holds in 12 of 12 cells. The second part holds for neighbor sets ($0.2787$ against $[0.2699,0.2806]$) and fails for $G_s$ ($0.4287$ against $[0.4072,0.4286]$) and CKA ($0.1496$ against $[0.1319,0.1447]$). The paired differences between the same initializations at $N=2^{16}$ and $N=\infty$ are $8.6\times10^{-4}$, $1.3\times10^{-4}$, and $-1.0\times10^{-4}$, each with a $95\%$ interval containing zero. The plan stated before the runs that this literal test passes with probability about $0.46$ when the two means are equal, because the 48-pair mean is far noisier than the 500-pair interval. For the secondary arm at step $0.002$, the second part fails for all three criteria. P1 remains falsified, and our conclusion that more data leaves agreement between gradient-descent runs unreliable rests on the first part.

\subsection{CIFAR-5m Networks with a Linear Head} \label{app:exp-s22}

In the first network experiment of Section~\ref{sec:exp-gauge}, the trunk is the airbench94 network \citep{jordan2024airbench} on $32\times32$ CIFAR-5m images \citep{nakkiran2021deep} with trunk features $h\in\R^{256}$, followed by the bias-free head $u=Mh$, $y=Wu$ with $u,y\in\R^9$, and logits equal to $y/9$ mapped by a fixed $10\times9$ matrix with orthonormal columns. Each run follows the airbench94 recipe (Nesterov momentum $0.85$, label smoothing $0.2$, flips and translations) for 488 steps of 1024 images. Three arms differ in the head alone, without weight decay on $M$ and $W$ (noWD), with weight decay on both (WD), or with a tenfold smaller initialization of both (small init). Each arm was trained on $N=25{,}000$ and $250{,}000$ images at a fixed budget of $500{,}000$ images seen, in 192 runs per arm and size, which gives 96 disjoint pairs. Each run draws its own training subset from five million images, so two subsets overlap by $0.5\%$ at $25{,}000$ and $5\%$ at $250{,}000$ in expectation. The two runs of a pair differ in initialization, training subset, data order, and augmentation. The tolerance of each criterion is the largest discrepancy between the outputs $y$ over 24 pairs of pilot runs at $N=100{,}000$, eight per arm. We judge $u$ against the tolerance of $y$ and compare features without normalization. The experimental record gives the tolerances and the full tables of this experiment.

\paragraph{Preregistered Decisions and Failed Predictions.} At $N=250{,}000$, weight decay makes $u$ agree in 96 of 96 pairs under all nine criteria, as predicted by part~3 of Theorem~\ref{thm:gauge}. At this size, $y$ agrees in 96 of 96 pairs under every criterion in every arm. Without weight decay, $u$ agrees on 10-nearest-neighbor sets in 96 of 96 pairs, and this result meets the preregistered falsification condition of prediction P1, which stated that this agreement would be unreliable at both sizes. The neighbor criterion is loose here, since the largest discrepancy of the pilot outputs sets its tolerance at $0.84$ on $1-s_N$. Agreement of $u$ fails instead on the scale-normalized criteria, where it is unreliable under $G_s$ at 65 of 96 pairs (Clopper--Pearson bounds $0.590$ and $0.756$) and inconclusive under $1-\CKA$ at 73 of 96. The mean $G_s$ discrepancy of $u$ is $0.2050$, just below the tolerance $\varepsilon=0.2073$ set at $N=100{,}000$, and a tolerance $4.7\%$ larger would make the cell reliable (Figure~\ref{fig:exp-gauge}b). The preregistered secondary tolerance is the $0.9$ quantile of the pilot discrepancies ($0.2030$ on $G_s$ and $0.0810$ on $1-\CKA$). At $250{,}000$, 39 and 45 of 96 pairs of $u$ lie within this tolerance on $G_s$ and $1-\CKA$ without weight decay, 96 and 96 with weight decay, and 58 and 79 under small initialization. In the second network experiment, with tolerances from the same arm and size, 0 and 5 of 96 pairs of $u$ lie within tolerance on $G_s$ under momentum without weight decay. We report prediction A3 on this agreement as tolerance-sensitive, without a categorical claim (Appendix~\ref{app:exp-s22-v11}).

Under a tenfold smaller initialization, 82 of 96 pairs of $u$ lie within tolerance on $G_s$ (inconclusive, lower bound $0.781$) and 86 of 96 on $1-\CKA$ (inconclusive), so prediction P3 failed in this arm. Prediction P5 stated that the fractions of agreeing pairs under small initialization and under weight decay differ by at most $0.1$, and it failed on $G_s$ at $250{,}000$, where these fractions differ by $0.146$. The other three comparisons of P5 are 0 against 0 or 96 against 96 pairs. At $N=25{,}000$, 0 of 96 pairs lie within tolerance for $u$ and $y$ on every criterion under each of the three training procedures with independent training subsets, since the tolerance came from pilot runs at $N=100{,}000$. The $25{,}000$ parts of P1, P3, and P5 hold or fail vacuously, whereas with same-size tolerances the second network experiment finds agreement of $u$ under weight decay reliable at both sizes.

\paragraph{Optimizer, Projector, and Trunk.} These runs fall in the case of Remark~\ref{rem:optimizers}, in which finite steps and momentum change $C$ during training. Stochastic gradient descent with momentum raised the median imbalance $\norm{MM^\top-W^\top W}_F$ of the head about ninefold at $N=250{,}000$, from $0.98$ to $8.59$, and about seven-hundredfold from the small initialization, from $0.0098$ to $6.92$. Under weight decay this imbalance fell to $0.43$. We divide the final imbalance by $\norm{MM^\top}_F+\norm{W^\top W}_F$ to obtain the relative imbalance, whose median at $250{,}000$ is $0.31$ without weight decay, $0.040$ with weight decay, and $0.28$ under small initialization. Under momentum the imbalance arises during training rather than from the initialization, since in the second network experiment the median cosine between the final and the initial $C$ is $-0.017$ and $-0.005$ (Appendix~\ref{app:exp-s22-v11}). Agreement of $u$ under the projector criterion is reliable in every arm at $N=250{,}000$, because $\Pi_u=\Pi_y$ (largest deviation $4.9\times10^{-8}$ over 1{,}344 runs) makes this criterion compare the column spaces of $y$. The mean $s_N(10)$ of the trunk features differs across arms by at most $0.011$, as predicted by P4. At $250{,}000$, agreement of $h$ and of its $\ell_2$-normalized version is reliable under every criterion in every arm except for $h$ on $1-\CKA$ under small initialization (91 of 96, inconclusive). At $25{,}000$ with independent training subsets, agreement of $h$ is unreliable on every criterion except $G_s$ under weight decay, whose reliable decision comes from a tolerance set by one pilot pair trained with small initialization. In an exploratory cell whose runs share four training corpora, with 24 dependent pairs per corpus, every mean discrepancy at $25{,}000$ fell relative to independent subsets, consistent with the distinction between independent and shared corpora in Appendix~\ref{app:criteria-supp}.

\begin{figure}[htbp] \centering \includegraphics[width=5.5in]{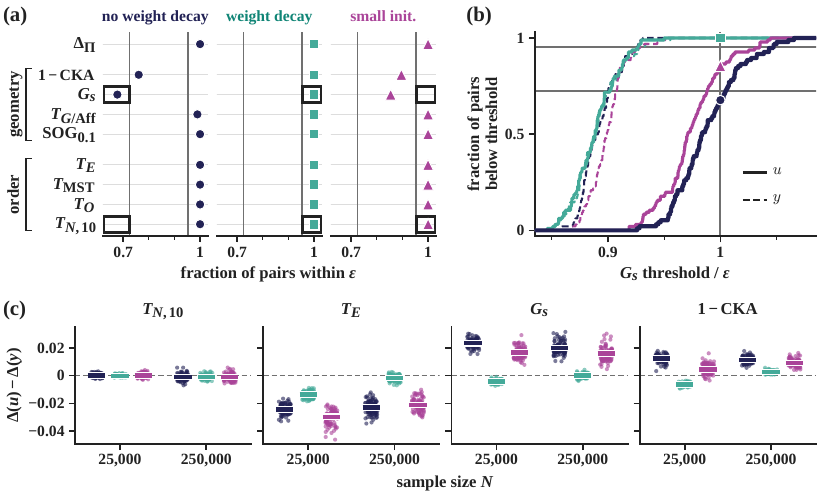} \caption{\textbf{Without weight decay, the gauge of a linear head raises the discrepancy of the head coordinates $u$ above that of the outputs $y$ on scale-normalized geometry but not on neighbor sets, and weight decay makes agreement of $u$ reliable under every criterion at $N=250{,}000$ (Theorem~\ref{thm:gauge}, Remark~\ref{rem:optimizers}).} (a)~Fraction of disjoint pairs of CIFAR-5m runs at $N=250{,}000$ whose discrepancy in the head coordinates $u$ lies within the tolerance $\varepsilon$ fixed in advance. Black boxes mark the ranges predicted in the preregistration, so a marker outside its box shows a failed prediction. (b)~Fraction of pairs whose $G_s$ discrepancy lies below a threshold, against this threshold over $\varepsilon$, for $u$ (solid) and the outputs $y$ (dashed). Grey lines at the fractions $69/96$ and $92/96$ in (a) and (b) mark the decision levels, and the vertical line in (b) marks $\varepsilon$. (c)~Excess $\Delta(u)-\Delta(y)$ of the discrepancy of $u$ over that of $y$ per pair of runs at both sample sizes, with median bars and the training procedures ordered and colored as in (a). The fractions in (b) and the excesses in (c) are exploratory.} \label{fig:exp-gauge} \end{figure}

\paragraph{Exploratory Analyses.} We chose the comparison of $u$ with $y$ within each pair and its split into a common gauge and a mismatch after seeing the preregistered decisions. Without weight decay, the $G_s$ discrepancy of $u$ exceeds that of $y$ in 96 of 96 pairs at both sizes, by a median $0.023$ at $25{,}000$ and $0.020$ at $250{,}000$ (Figure~\ref{fig:exp-gauge}c). Weight decay removes the $G_s$ excess ($-0.0001$ at $250{,}000$), whereas small initialization keeps most of this excess ($0.016$). On $T_{N,10}$ the median excess lies between $-0.0011$ and $0.0003$ in every cell, pilot pairs included. When the gauge of one run is applied to the outputs of both runs of a pair, this common gauge produces about three quarters of the $G_s$ excess without weight decay at $250{,}000$ ($0.0148$ of $0.0200$). The output discrepancy plus the remaining mismatch part lies within the tolerance of this experiment in 96 of 96 pairs in every cell at $250{,}000$. The preregistered decision on $u$ therefore depends both on a gauge shared by all runs and on the run-to-run mismatch addressed by Theorem~\ref{thm:gauge}, and the second network experiment separates these two parts (Appendix~\ref{app:exp-s22-v11}). Order criteria stronger than the neighbor sets can detect a gauge. On $T_E$ the pairs of $u$ agree more closely than the pairs of $y$ (median excess $-0.023$ at $250{,}000$, 96 of 96 pairs), a difference predicted by neither Theorem~\ref{thm:gauge} nor Corollary~\ref{cor:gauge-neighbors}.

\paragraph{Neighbor Sets under the Gauge.} We formulated Corollary~\ref{cor:gauge-neighbors} after the results of the first network experiment, and the second network experiment tests this corollary prospectively (Appendix~\ref{app:exp-s22-v11}). Since the corollary is proved, we count the anchors certified by its ratio condition in the stored runs to check our code and to measure the coverage of the certificate, not to confirm a prediction (exploratory). In the linear autoencoder at $N=\infty$, where both runs of a pair minimize the same population loss, every certified anchor kept its neighbor set at $k=1$ and $k=10$ in 3{,}972 pairs from eight training procedures. At $k=30$, two violations come from one anchor of one tied-weights pair, whose saved encoders leave the span of $\hat U$ by a relative $6\times10^{-5}$. Both disappear when the distances are computed from the projected minimizer. At finite $N$, where the two runs of a pair see different samples and lie outside the premises of the corollary, certified anchors change, for instance 6{,}877 of 12{,}288 for tied weights at $N=2^8$. The certificate covers at least $94\%$ of the anchors at $k=10$ under weight decay, tied weights, and small initialization, where $\mathrm{cond}_{AB}-1<0.008$. Under gradient descent at step $0.002$, where the median $\mathrm{cond}_{AB}$ is $8.9$, the certificate covers 0 anchors at $k=10$, and $99\%$ of the 10-nearest-neighbor sets differ between the two runs of a pair. In the linear autoencoder, the gauge prevents agreement on neighbor sets together with agreement under CKA, and 0 of the 3{,}972 pairs met the neighbor tolerance while missing the CKA tolerance. The two bounds on $\mathrm{cond}_{AB}$ through the initial imbalance in the proof of the corollary hold in 500 of 500 pairs of the step-$0.002$ arm. Outside gradient flow, both bounds can fail, as permitted by Remark~\ref{rem:optimizers}. Under gradient descent from a small initialization they fail in 281 and 90 of 500 pairs, by at most $0.007$. From the default initialization at step $0.05$, the first bound fails in 1 of 486 pairs.

In the CIFAR-5m networks, which lie outside the premises of Theorem~\ref{thm:gauge}, the median $\mathrm{cond}_{AB}$ lies between $1.14$ and $1.45$ across cells. The certificate covers at most $5\%$ of the anchors at $k=10$, with 0 violations, the count implied by construction. An exchange of the head gauge alone moves 2 to $6\%$ of the 10-nearest-neighbor slots of one run, but without weight decay at $250{,}000$ the outputs of two runs share a mean of $0.19$ of their 10 nearest neighbors. Added to this difference between the learned functions, the gauge changes the mean cross-run discrepancy on neighbor sets by at most $0.0002$ in every cell. The decisions of this network on neighbor sets therefore reflect the learned functions and are not evidence for the corollary, and the same holds for the failed prediction P1.

\subsubsection{Second Experiment: Head Flow and Same-size Tolerances} \label{app:exp-s22-v11}

\paragraph{Design.} The second network experiment trains fresh runs of the network above, with new seeds for every source of randomness and the head computed in float64. It separates a gauge shared by the runs of an arm from the run-to-run gauge. Five core arms have 96 main pairs per size and 96 pilot pairs (64 for FLOW). MOM trains the head with the Nesterov momentum of the trunk, as in the arm without weight decay above, and WD adds weight decay on the head. FLOW, BAL-FLOW, and SMALL-FLOW remove the head from the optimizer, while the trunk keeps its momentum. After each trunk step, these arms hold the trunk features fixed and integrate gradient flow of the head on the minibatch loss over the time of one momentum-equivalent step (the learning rate of the trunk divided by $1-0.85$). The integration uses an adaptive Heun--Euler integrator of relative tolerance $10^{-4}$. These three arms start from the default, a balanced ($C_0=0$), and a tenfold smaller initialization. Three secondary arms with 32 main and 16 pilot pairs per size train the head with momentum from the balanced and the small initialization (BAL-MOM, SMALL-MOM) or with one plain Euler step per minibatch (SGD-K). At $N=250{,}000$, MOM-BF16 repeats MOM in 32 runs with the bfloat16 head of the first experiment, and KS-4 and KS-16 split the Euler step into 4 and 16 substeps in 16 runs each. We estimate the drift of $C$ from these three cells. The balanced initialization gives $W$ a condition number near $1.34$ against about $26$ for the default, and SMALL-FLOW controls for this confound by keeping the default condition number at one tenth of the scale.

\paragraph{Tolerances and Decompositions.} Each tolerance is the largest discrepancy of $y$ over the pilot pairs of the same arm and size, and $u$ is judged against this tolerance. As a sensitivity check, a pooled tolerance takes the largest discrepancy over the pilot pairs of all arms that pass the accuracy gate at that size. A decision entry that is reliable under one of the two tolerances and unreliable under the other is reported as tolerance-sensitive, without a categorical claim. A decision on $u$ requires reliable agreement of $y$ in the same cell and is otherwise inconclusive. Per pair and criterion, the excess $\Delta(u)-\Delta(y)$ splits in two ways. The common-gauge part is the excess of a counterfactual that gives every run of the arm one common gauge, and the mismatch part is the rest. This common gauge is the log-Euclidean mean of the determinant-normalized metrics $(WW^\top)^{-1}$ of the pilot runs. The balanced part is the excess of the coordinates $(WMM^\top W^\top)^{-1/4}y$ of the balanced head computing the same outputs, and the imbalance part is the rest. Its run-to-run component equals the mismatch part minus the mismatch part of the balanced coordinates. The distortion caused by the imbalance is the value of $\mathrm{cond}_{AB}$ between the coordinates $u$ of a run and those of its balanced head. Every part is given in standard deviations (SD) of $\Delta(y)$ over the main pairs of the cell.

\paragraph{Order, Gates, and Checks.} The preregistration and all training, gate, and analysis code were frozen before any engineering, pilot, or main run. The tolerances were frozen from pilot pairs alone during the training of the main runs, before the first main measurement. All 4{,}352 runs completed, with 0 runs flagged as failed or diverged and 0 retries. We divide $\norm{C_{\mathrm{end}}-C_0}_F$ by $\norm{MM^\top}_F+\norm{W^\top W}_F$ at the end of training to obtain the relative drift of $C$. The pilot runs passed every gate, including the gate on the median relative imbalance ($0.273$ and $0.314$ under MOM, $0.030$ and $0.040$ under WD). The main runs passed every check, with a largest relative drift of $C$ under head flow of $2.75\times10^{-5}$ against a limit of $2\times10^{-4}$. Two preregistered bounds on the distortion held in the 2{,}337 of 2{,}368 runs where the first bound is non-vacuous and in all 1{,}152 runs of the flow arms. A third bound held in every pair, and the certificate of Corollary~\ref{cor:gauge-neighbors} had 0 violations, the count implied by construction. An independent recomputation reproduced every tolerance, all 704 decision entries, and all 880 means of the decompositions. We wrote a second recomputation without reading the analysis code or the first, and it reproduced every preregistered decision and every estimate up to Monte Carlo error.

\paragraph{Estimates.} Under FLOW, whose $C$ is conserved, the final gauge lies closer to the run's own initial gauge than to that of another run in 95 and 102 of 192 runs, and both $95\%$ intervals contain $0.5$. The first experiment's conclusion about this memory rested on the same statistic ($46\%$ to $52\%$). As this statistic stays at chance even when $C$ is conserved, this conclusion now rests on the cosine between the final and the initial $C$ under momentum (C4a in Table~\ref{tab:exp-verdicts}). Under SGD-K, a plain Euler step of the same size changes $C$ slightly more than momentum, with a median relative drift of $0.343$ and $0.376$ against $0.276$ and $0.314$ under MOM and $2.4\times10^{-5}$ and $2.5\times10^{-5}$ under FLOW. At $250{,}000$ the drift under momentum is $0.76$ times the drift of a plain step extrapolated from KS-4 and KS-16 ($0.137$ and $0.028$). This ratio lies below the preregistered factor of 3 for naming momentum the main cause, so Remark~\ref{rem:optimizers} names both finite steps and momentum as causes of the change of $C$. The bfloat16 head gives a median relative imbalance of $0.3121$ against $0.3119$ in float64, which rules out bfloat16 precision as the cause of the ninefold growth of the imbalance in the first experiment.
\begin{figure}[htbp]
\centering
\includegraphics[width=5.5in]{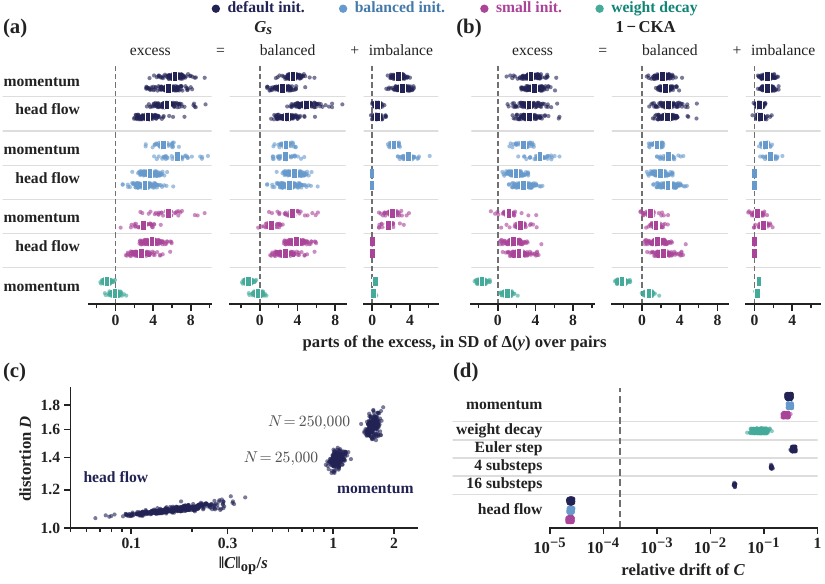}
\caption{\textbf{The initial imbalance kept by head flow contributes a smaller part of the excess discrepancy of the head coordinates than the imbalance created by momentum, and weight decay removes the excess on $G_s$ (Theorem~\ref{thm:gauge}, Remark~\ref{rem:optimizers}).} (a, b)~Excess $\Delta(u)-\Delta(y)$ of each pair of runs in the second network experiment and its balanced and imbalance parts. Dots are pairs, bars means, and each row shows $N=25{,}000$ above $N=250{,}000$. (c)~Distortion $D$ caused by the imbalance in each run against the imbalance relative to the smallest singular value $s$ of the head, initial under head flow and final under momentum (exploratory). (d)~Relative drift of $C$ per run against the limit set in advance for head flow (dashed). Appendix~\ref{app:exp-s22-v11} gives the tests.}
\label{fig:app-s22v11}
\end{figure}

Relative to the balanced start, the default start raises the imbalance part more under head flow than under momentum, with every difference-in-differences interval above zero. The distortion caused by the imbalance rises with the imbalance relative to the smallest singular value of the head (exploratory, Figure~\ref{fig:app-s22v11}c). Its Spearman correlation with this ratio is $0.89$ and $0.73$ under FLOW for the initial imbalance and $0.33$ and $0.32$ under MOM for the final imbalance (unadjusted $p<10^{-5}$). Under head flow the imbalance part is $0.19$ and $0.18$ times that under momentum on $G_s$ at the two sizes and $0.39$ and $0.48$ times on $1-\CKA$. Weight decay removes the excess on $G_s$, at $-0.91$ and $-0.04$ SD, whereas on $1-\CKA$ the excess under weight decay is $-1.65$ SD at $25{,}000$ and $+1.02$ SD at $250{,}000$, with all 96 pairs above zero (Figure~\ref{fig:app-s22v11}). Under MOM at $250{,}000$, the common-gauge part makes up $0.71$ of the excess on $G_s$. In this cell, agreement under the output discrepancy plus the mismatch part is inconclusive at 78 of 96 pairs, with a mismatch part of $1.63$ SD ($95\%$ interval $[1.46,1.80]$). Agreement of $u$ under the tolerance of the common-gauge counterfactual is inconclusive at 90 of 96. On $T_{N,10}$, agreement of $u$ is reliable in every cell, but prediction A2 fails in 4 of 10 cells, all at $250{,}000$. In these cells the mean discrepancy of $u$ lies $0.0017$ to $0.0023$ below that of $y$, and the $95\%$ interval of this difference leaves the band of $\pm0.002$ required by A2.

\paragraph{Sensitivity to the Pooled Tolerance and a Collapsed Pilot Run.} Under MOM with same-size tolerances, 0 and 5 of 96 pairs of $u$ lie within tolerance on $G_s$, and 1 and 21 of 96 pairs of its common-gauge counterfactual lie within this tolerance. Under the pooled tolerance, which is $7.8\%$ larger, 95 of 96 pairs of this counterfactual lie within tolerance at $250{,}000$, so we report the outcome of prediction A3 as tolerance-sensitive. At $250{,}000$, the tolerance-sensitive outcomes of A3 and A7 rest on one pilot pair of SMALL-FLOW, which sets the pooled tolerance on $G_s$ ($0.2098$), $1-\CKA$, and $T_{N,10}$. Its $G_s$ discrepancy lies $5.7$ SD above the main pairs of the arm, although both runs are healthy (accuracy about $0.897$). At the next pooled value on $G_s$ ($0.2017$), 79 of 96 pairs of the common-gauge counterfactual of MOM and 72 of 96 pairs of $u$ under BAL-FLOW lie within tolerance, and both decisions are inconclusive. The same pair sets the tolerance of SMALL-FLOW at $250{,}000$, where 93 of 96 pairs of $u$ lie within tolerance on $G_s$. This count falls to 52 of 96 at the second-largest pilot value of the arm and to 22 of 96 under the $0.9$-quantile tolerance. We therefore do not read this cell as a selection of the gauge by small initialization, and prediction A7(a) fails in both arms (Table~\ref{tab:exp-verdicts}). One SGD-K pilot run at $250{,}000$ collapsed without triggering the divergence flag (accuracy $0.089$, condition number of $W$ near $2.1\times10^5$). This run alone makes the tolerances of its cell vacuous ($1.17$ on $G_s$ against $0.20$ at the $0.9$ quantile) and removes the cell from the pooled tolerance through the accuracy flag. The primary endpoints exclude this cell, and the estimates of the drift of $C$ use main runs, whose final condition number of $W$ is at most $2.5$.

\subsection{Contrastive Objectives in the Finite Gibbs Model}
\label{app:exp-s23}

Inside the premises of Theorem~\ref{thm:link}, all nine confirmatory
predictions held after Holm correction
(Table~\ref{tab:exp-verdicts}), and by the preregistered rule the
pattern of the theorem is present on the order criteria.

\paragraph{Design.}
The model has $|\Xcal|=256$ inputs with uniform masses, in 16 clusters
of 16 points grouped into 4 superclusters, and the temperature is
$t=0.3$. Its $K_\star$ is built from a Gaussian kernel on fixed
coordinates in $\R^8$ and adjusted so that the masses are balanced, with
smallest eigenvalue $0.129$ and pairwise distinct off-diagonal
distances. Learners A and B are those of the formal version of Theorem~\ref{thm:link} (Appendix~\ref{app:learning-link}). Each run
draws its own $N$ examples as exact multinomial counts, and each example
has one negative ($J=1$). Both learners return minimizers with numerical certificates of optimality, and a
run whose certificate fails after one re-solve is kept and counted
against every test that uses the run. We run each learner 192 times at
each of the sample sizes $2^{20}$, $2^{22}$, $2^{25}$, $2^{28}$,
$2^{31}$, $2^{34}$, $2^{37}$, and $2^{40}$, in 96 disjoint blocks of two
runs of each learner. Exact criteria compare the order of all distances
among eight probe inputs, one in each of eight clusters, and the
criteria of Appendix~\ref{app:exp-criteria} compare all 256 inputs. The tolerance of the largest difference of mean-normalized distances and of the isotonic discrepancy $\inf_g\norm{g(D_A)-D_B/\psi(D_B)}_{\max}$, which is that of part~4 of the formal version of Theorem~\ref{thm:link} in units of $\psi(D_B)$, is half the limit $\varepsilon_\psi=0.3477$ of part~3 of the formal version, with $\psi$ the mean off-diagonal distance. These criteria, CKA, and $\varepsilon_\psi$ take all 256 inputs as the probes of Theorem~\ref{thm:link}, and the exact criteria take the eight probe inputs. Before the first training sample, the
frozen code computed the Gaussian limits of both learners, with that of
InfoNCE given by part~1 of Theorem~\ref{thm:peer}. From these limits, a
Monte Carlo simulation predicted that the share of pairs in exact
agreement across objectives on the order of the probes reaches $90\%$ at
$N^*=2^{33.75}$ examples, so the grid is centered on $2^{34}$. A secondary arm pairs
learner A with InfoNCE at temperature $t/2$ at $2^{20}$ and $2^{22}$
examples, and positive controls and two planted breaks of the order
check the pipeline.

\paragraph{Results.}
Exact agreement across objectives on the order of the probes held in 0
of 96 pairs up to $2^{28}$ examples, 7 of 96 at $2^{31}$, 94 of 96 at
$2^{34}$, and 96 of 96 at $2^{37}$ and $2^{40}$. The predicted shares
were $0.073$ at $2^{31}$ and $0.944$ at $2^{34}$, and at $2^{40}$ the
lower confidence bound of the share of pairs in exact agreement is
$0.969$.
From $2^{34}$ to $2^{40}$, the recovery errors of learners A and B fell
by ratios of $0.124$ $[0.122,0.126]$ and $0.126$ $[0.121,0.130]$, below
the preregistered bound of $0.35$ and close to $0.125$ at the rate
$N^{-1/2}$.
Across objectives the largest difference of mean-normalized distances
fell from a median $0.519$ at $2^{20}$ to $0.348$ at $2^{40}$, and 0 of
96 pairs came within the tolerance of $0.174$ at any size. At $2^{40}$,
all 96 pairs within each objective met this tolerance and agreed exactly
on the order of the probes. With the link
undone on the spectral run, the median difference fell from $0.606$ at
$2^{20}$ to $0.00497$ at $2^{34}$ and $0.000618$ at $2^{40}$, against
$0.00058$ and $0.00043$ within the two objectives. Its ratio from $2^{34}$ to $2^{40}$ is $0.124$ $[0.121,0.127]$. Intervals are $95\%$
block-bootstrap intervals over the 96 blocks. Under the isotonic
discrepancy, every pair across objectives came within the tolerance at
$2^{40}$. At this size, on neighbor sets, per-anchor rankings, and the global
order, the mean discrepancy across objectives stayed below the sum of the
two mean discrepancies within objectives, as predicted from the finite
model. Its ratio to the mean of these two was $0.97$ to $1.00$ on these
criteria, against $1.11$ on spanning trees (secondary). On $G_s$ and
$1-\CKA$, the mean discrepancy across objectives exceeded that sum by
$0.113$ and $0.141$, and $1-\CKA$ across objectives on all 256 inputs reached $0.14135$, within $0.01$ of its limit $0.1414$. On the positive
controls, a run and the link of the theorem applied to its own Gram
matrix agreed on every order criterion, as did a run and its pure rescaling. The pipeline detected both planted breaks of the order. At $2^{20}$, the mean
$s_N(10)$ was $0.806$ and $0.735$ within the two objectives and $0.763$
across them, against $0.715$ within InfoNCE in the networks of
Appendix~\ref{app:exp-s21}. At this size the mean discrepancy across
objectives on every order criterion also stayed below the sum of the two
mean discrepancies within objectives, as predicted from the finite model
(descriptive). Its ratio to the mean of these two was $1.01$ to $1.03$ on
these criteria and $1.14$ on $\Delta_{\rm iso}$, against $3.2$ in the
networks. The
secondary arm compares the temperatures $t$ and $t/2$ at $2^{20}$ and
$2^{22}$, where the InfoNCE minimizers lie on the boundary of the
positive semidefinite cone, outside the rescaling case of Remark~\ref{rem:temperature}.
Its mean discrepancy across temperatures exceeded that within
each temperature by at most $6.1\%$ on every criterion (descriptive). In
Figure~\ref{fig:exp-link}a and~c, the grey band runs between the values
within InfoNCE and within the spectral loss.
Figure~\ref{fig:exp-link}b shares the axes and color scale of Figure~\ref{fig:app-s21-net}b. It shows the pair at $2^{40}$ examples ranking 48th of 96 by increasing difference after undoing the link, a
choice fixed before plotting.

\paragraph{Secondary Predictions and Certificates.}
Two secondary predictions held, the Gaussian limit of part~1 of
Theorem~\ref{thm:peer} in 192 of 192 checks (P1b) and the transition of
the agreement within InfoNCE (P2b). At $2^{40}$, the Monte Carlo simulation
predicted the ratio of mean discrepancies to be $0.557$ on neighbor sets and $0$ on spanning trees, against the observed
$0.973$ and $1.111$ (P5b). Both predictions fall outside the factor of $1.5$ required
by the plan, so P5b failed. We attribute the failure on neighbor sets to
the Monte Carlo error of 48 simulated blocks. Of the 3{,}456 main solves, 44 narrowly missed
the tolerance $10^{-6}$ of their optimality certificate, and we kept and
flagged all 44. Thirty-nine are InfoNCE runs at $2^{25}$, a size without
a test, and five belong to the secondary arm at $2^{22}$, where counting
them against its descriptive conclusion leaves the conclusion unchanged.

\paragraph{Amendments and Checks.}
The preregistration and the 37 code files of this experiment were frozen
before any computation, and the experiment ran under four
amendments.
Under amendment A1, we resumed the experiment on a free GPU after
another job had slowed the pilot solves about tenfold, and we excluded
the slowed timings from the frozen projection of GPU time. Amendment A2 removed the
limits of $11.5$ and 12 GPU-hours that we had set ourselves in the
preregistration, after the main phase stopped at the first limit with
156 runs of the secondary arm left. Amendments A3 and A4 fixed two
programming errors in the frozen analysis code, a cache that lost its
newest matrix on overflow and a division by a Monte Carlo prediction
equal to zero. We wrote and froze both amendments during the analysis
run, and the preregistration allows such code fixes if they precede the
opening of the affected data. The file metadata do not
show whether the 418 pair caches written before the A4 freeze were
opened before this freeze. Every decision is independent of this
question, since A3 leaves every computed value unchanged and A4 defines
one secondary boolean outside every decision. A recomputation written without importing
any code of the experiment reproduced every decision and every frozen
value other than the bootstrap intervals to within $1.2\times10^{-7}$.
The creation times of 53 files confirmed the order of the experiment.

\paragraph{Compute and Scope.}
The InfoNCE minimizers lie on the boundary of the positive semidefinite
cone up to $2^{25}$ examples and in its interior at every larger size.
At $2^{25}$ the frozen
projection therefore underestimated the GPU time of the solves about
fivefold. The experiment used 12.71 GPU-hours, 2.71 of them for the
Gaussian limits, against an estimate of about 8, while learner B and the
analysis ran on the CPU. One preregistered gate used the rescaling
identity of Remark~\ref{rem:temperature} as a check of the solver rather
than as a test of the remark. Its residual was at most $1.1\times10^{-16}$ on the 32 samples at $2^{34}$ and $2^{40}$ that satisfy the premises of the remark. The preregistration also fixes a second stage that removes one
premise of the theorem at a time, and we have not run this stage.

\subsection{Contrastive Networks}
\label{app:exp-s21}

In our preregistration, we extended the prediction of
Theorem~\ref{thm:link} to CIFAR-5m networks, where learners trained with
InfoNCE and with the spectral contrastive loss would learn one relation
through different increasing links. Their representations would then
agree on every order criterion and disagree on scale-normalized
geometry. The trunk is
the airbench94 network \citep{jordan2024airbench} with features
$h\in\R^{256}$, followed by a projection head to a 128-dimensional
embedding. InfoNCE uses a cosine critic at temperature $0.2$ (arm A) and $0.5$ (arm A$'$). The spectral contrastive loss (arm B) uses the
ball projection of \citet{haochen2021provable} with unit parameter, a
deviation recorded in Appendix~\ref{app:exp-prereg}. All arms use
stochastic gradient descent with Nesterov momentum, weight decay
$5\times10^{-4}$, and batches of 512 images with two views each. We
train 64 runs of each arm on $N=50{,}000$ and $500{,}000$ images at a
budget of five million images seen, so $50{,}000$ images mean 100 epochs
and $500{,}000$ images 10. Embeddings are compared with the cosine
kernel on 2048 held-out images. The tolerance of each criterion is the
largest discrepancy over 24 pooled within-arm pilot pairs at
$N=150{,}000$, and pairs within an arm (32 per arm) and across arms (64)
are disjoint. Table~\ref{tab:app-s21-full} gives every decision, and
Figure~\ref{fig:app-s21-net}b shows the pair across objectives ranking
32nd of 64 by increasing warp share, chosen before plotting.
Grey boxes in Figure~\ref{fig:app-s21-net}a span all within-arm pairs.

\paragraph{Outcome.}
These networks lie outside the premises of the theorem, since the
spectral arm trains on the unit sphere and the embedding dimension falls
far below the size of the input domain. Their empirical InfoNCE risk also has no positive definite minimizer at any temperature (Remark~\ref{rem:temperature}), since two views of each image are each other's only positive partners. All 64 pairs of InfoNCE
and spectral runs fall outside the tolerance on neighbor sets,
per-anchor rankings, the global order, and spanning trees at both sizes.
At $500{,}000$ images their median discrepancies on these criteria are
$1.34$ to $2.51$ times the tolerance, whereas pairs with a common
objective meet every criterion (Figure~\ref{fig:app-s21-net}a). The
mean $s_N(10)$ at this size is $0.313$ across objectives against $0.715$
and $0.640$ within the two objectives. On each
of the ten criteria of Table~\ref{tab:app-s21-full}, the discrepancy of
every pair across objectives exceeds that of every pair within an
objective, at both sizes and for both temperatures. The preregistered
requirement of reliable agreement within each objective fails at
$50{,}000$ images, on four criteria for InfoNCE at $0.2$ and on six for
the spectral loss. In a preregistered descriptive cell, we trained 32
runs of each of arms A and B on $50{,}000$ images for the 10 epochs of
the main runs at $500{,}000$ images. All 32 pairs
across objectives lie outside the tolerance on every criterion, and
agreement within each arm is unreliable on six or seven of the ten
criteria and inconclusive on the others.

\paragraph{Items of the Link Prediction.}
The median warp share across objectives lies below $0.2$ and the order gaps exceed $0.1$, so both conditions for falsification in Table~\ref{tab:exp-verdicts} hold and the link prediction is falsified for these networks. Item~(i) required the share of pairs across objectives
within tolerance under each order criterion to come within $0.1$ of the
within-objective ceiling of Remark~\ref{rem:ceiling} at $500{,}000$
images, and their mean discrepancy to fall with $N$. Every gap is $1.0$,
and the fall from $50{,}000$ to $500{,}000$ images is significant after
Holm correction for $T_{N,10}$ ($-0.0125$, $95\%$ interval
$[-0.0145,-0.0105]$) and $T_{\mathrm{MST}}$ ($-0.0119$) but not for $T_O$
($p=0.037$ against $0.025$) or $T_E$ ($p=0.25$). Item~(ii) required
agreement across objectives under the geometric criteria to be
unreliable at both $N$, with reliable agreement within each objective at
$500{,}000$ images. The item passes by its letter, but agreement across
objectives is unreliable under every criterion, so the item does not
single out geometry. Item~(iii) required the isotonic discrepancy across
objectives to stay within $1.5$ times its value within an objective, and
the ratio is $3.198$. Item~(iv) required a median warp share of at least
$0.5$ across objectives, and the median is $0.147$ at $500{,}000$ images
(Figure~\ref{fig:app-s21-net}c) and $0.142$ at $50{,}000$, against
$0.012$ and $0.002$ within the two objectives. Applied without a fit,
the link of the theorem explains a share of $-0.04$ of the discrepancy
of the median pair (Figure~\ref{fig:app-s21-net}b). A warp share below $0.2$ is compatible with a general monotone link (Appendix~\ref{app:exp-prereg}). Our conclusion that these networks trained with the two objectives learn representations differing by more than a monotone link therefore rests on the order criteria and on
$\Delta_{\rm iso}$.
Theorem~\ref{thm:link} remains a statement about the finite Gibbs model,
where its pattern appears at $2^{34}$, $2^{37}$, and $2^{40}$ examples.

\paragraph{Link Checks.}
As a check of the pipeline, A runs are paired with the exponential link
$\exp(5(K-1))$ of their own Gram matrix (self, 64 pairs) and of the Gram
matrix of another A run (peer, 32 pairs). Every self pair meets every order criterion and fails every geometric criterion, with a
median warp share of $1.00$ (diamonds in
Figure~\ref{fig:app-s21-net}a,c), so the pipeline detects an exact link.
In peer pairs, the link leaves the discrepancy under every order criterion at its value for the two InfoNCE runs. It moves the discrepancy under every geometric criterion outside the tolerance, with median warp shares of $0.66$ and $0.75$ at the two sizes.

\paragraph{Temperatures.}
At each size, InfoNCE runs agree on neighbor sets in 32 of 32 pairs
within each temperature. All 64 pairs of InfoNCE runs across the two
temperatures lie within the tolerance on $\Delta_\Pi$ at both sizes and
outside the tolerance on every other criterion. Between these
temperatures, a fitted monotone link removes a median $11\%$ of the
residual left by an affine map. Item~(v) required
InfoNCE runs at the two temperatures to agree on $G_s$, CKA, and
SOG$_{0.1}$, and every one of the 64 pairs fails all three, while its
clause on raw distances passes. The two InfoNCE arms therefore lie
outside the rescaling case of Remark~\ref{rem:temperature}, which
remains a statement about exact minimizers over unit-diagonal positive
semidefinite matrices.

\begin{table}[htbp]
\centering
\resizebox{\textwidth}{!}{%
\begin{tabular}{@{}llrrrrrrrrrr@{}}
\toprule
$N$ & Pair set & $T_{N,10}$ & $T_O$ & $T_E$ & $T_{\mathrm{MST}}$ & SOG$_{0.1}$ & $T_{G/\mathrm{Aff}}$ & $G_s$ & $1-\mathrm{CKA}$ & $\Delta_{\rm iso}$ & $\Delta_\Pi$ \\
\midrule
50k & InfoNCE at 0.2, within & 32/32\,\textsc{r} & 0/32\,\textsc{u} & 0/32\,\textsc{u} & 32/32\,\textsc{r} & 32/32\,\textsc{r} & 0/32\,\textsc{u} & 32/32\,\textsc{r} & 0/32\,\textsc{u} & 32/32\,\textsc{r} & 32/32\,\textsc{r} \\
50k & InfoNCE at 0.5, within & 32/32\,\textsc{r} & 32/32\,\textsc{r} & 32/32\,\textsc{r} & 32/32\,\textsc{r} & 32/32\,\textsc{r} & 32/32\,\textsc{r} & 32/32\,\textsc{r} & 32/32\,\textsc{r} & 32/32\,\textsc{r} & 32/32\,\textsc{r} \\
50k & spectral, within & 1/32\,\textsc{u} & 32/32\,\textsc{r} & 32/32\,\textsc{r} & 7/32\,\textsc{u} & 2/32\,\textsc{u} & 32/32\,\textsc{r} & 19/32\,\textsc{u} & 32/32\,\textsc{r} & 19/32\,\textsc{u} & 0/32\,\textsc{u} \\
500k & each arm, within & 32/32\,\textsc{r} & 32/32\,\textsc{r} & 32/32\,\textsc{r} & 32/32\,\textsc{r} & 32/32\,\textsc{r} & 32/32\,\textsc{r} & 32/32\,\textsc{r} & 32/32\,\textsc{r} & 32/32\,\textsc{r} & 32/32\,\textsc{r} \\
both & InfoNCE at 0.2 or 0.5 vs spectral & 0/64\,\textsc{u} & 0/64\,\textsc{u} & 0/64\,\textsc{u} & 0/64\,\textsc{u} & 0/64\,\textsc{u} & 0/64\,\textsc{u} & 0/64\,\textsc{u} & 0/64\,\textsc{u} & 0/64\,\textsc{u} & 0/64\,\textsc{u} \\
both & InfoNCE at 0.2 vs 0.5 & 0/64\,\textsc{u} & 0/64\,\textsc{u} & 0/64\,\textsc{u} & 0/64\,\textsc{u} & 0/64\,\textsc{u} & 0/64\,\textsc{u} & 0/64\,\textsc{u} & 0/64\,\textsc{u} & 0/64\,\textsc{u} & 64/64\,\textsc{r} \\
both & link check, self & 64/64\,\textsc{r} & 64/64\,\textsc{r} & 64/64\,\textsc{r} & 64/64\,\textsc{r} & 0/64\,\textsc{u} & 0/64\,\textsc{u} & 0/64\,\textsc{u} & 0/64\,\textsc{u} & 64/64\,\textsc{r} & n/a \\
50k & link check, peer & 32/32\,\textsc{r} & 0/32\,\textsc{u} & 0/32\,\textsc{u} & 32/32\,\textsc{r} & 0/32\,\textsc{u} & 0/32\,\textsc{u} & 0/32\,\textsc{u} & 0/32\,\textsc{u} & 32/32\,\textsc{r} & n/a \\
500k & link check, peer & 32/32\,\textsc{r} & 32/32\,\textsc{r} & 32/32\,\textsc{r} & 32/32\,\textsc{r} & 0/32\,\textsc{u} & 0/32\,\textsc{u} & 0/32\,\textsc{u} & 0/32\,\textsc{u} & 32/32\,\textsc{r} & n/a \\
\bottomrule
\end{tabular}}
\caption{Decisions of the contrastive networks on the embedding with the
cosine kernel, as pairs within tolerance and the decision of
Appendix~\ref{app:exp-protocol} (\textsc{r} reliable, \textsc{u}
unreliable). A row marked ``both'' applies to $N=50{,}000$ and to $500{,}000$, and the row ``each arm'' applies to each of the three arms.}
\label{tab:app-s21-full}
\end{table}

\begin{figure}[htbp]
\centering
\includegraphics[width=5.5in]{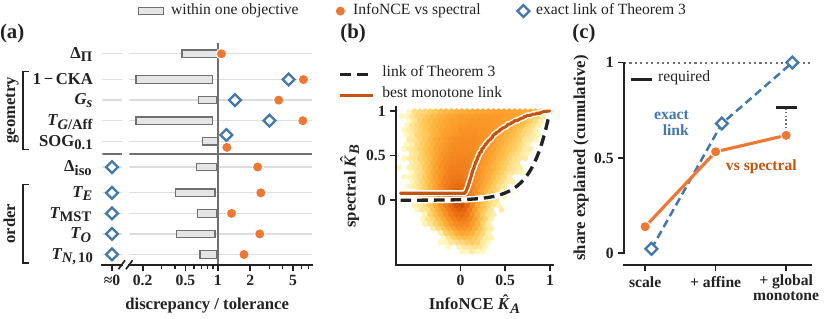}
\caption{\textbf{In CIFAR-5m networks, InfoNCE and spectral learners disagree on the order of distances as well as on geometry, and a monotone link explains little beyond an affine map.} (a)~Discrepancy over the tolerance of each criterion. Theorem~\ref{thm:link} predicts agreement on the criteria below the horizontal line, and diamonds mark pairs of an InfoNCE run and the exact link of the theorem applied to its Gram matrix. (b)~Density of kernel entries of one pair of runs. (c)~Share of the squared-distance discrepancy explained by successive maps, with the share required by the preregistered prediction. All panels are at $N=500{,}000$, and marks in (a) and~(c) are medians over pairs. Each pair across objectives joins a run of InfoNCE at temperature $0.2$ with a spectral run.}
\label{fig:app-s21-net}
\end{figure}

\paragraph{Secondary Predictions.}
The fitted link $K_B\approx\exp(\beta(K_A-1))$ has median exponent
$\hat\beta=2.04$ across objectives, outside the predicted interval
$[2.5,10]$. The ratio of $\hat\beta$ between the two temperatures is
$1.006$, outside the predicted $[1.5,3.5]$, and we trace this failure to
an error of our design (Appendix~\ref{app:exp-prereg}). The
co-occurrence prediction
derived from Lemma~\ref{lem:cooccurrence} stated that more data would
shrink the log-distance error of near pairs faster than that of far
pairs in both InfoNCE arms. In all three arms, more data shrinks the
error of near pairs less than that of far pairs, so this prediction
fails, while the predicted ordering between the two temperatures holds.

\paragraph{Exploratory Analyses.}
At $500{,}000$ images, the participation ratio of the centered Gram
matrix is $67.7$ and $27.9$ for InfoNCE at the two temperatures and
$11.7$ for the spectral arm. On four pairs of runs, an exponential link
of the InfoNCE Gram matrix raises this ratio above $300$ and spreads the
variance over many directions, whereas the spectral embedding
concentrates its variance on few directions. These runs leave two
explanations open, a change of rank or spectrum as in the model-class
row of Table~\ref{tab:learning}, and a common path of temperatures for
the three arms. A heuristic second-order expansion makes the spectral
loss of unit-norm embeddings resemble InfoNCE at temperature $1$ without its penalty on the mean embedding. The three arms may therefore lie at $0.2$,
$0.5$, and about $1$ on one path.

\section{Limitations and Open Problems}
\label{app:limitations}
\label{sec:limitations}

This appendix states the limitations of the results of
Sections~\ref{sec:reads} to~\ref{sec:experiments}, one topic per
paragraph, and then the open problems that these limitations leave.

\subsection{Limitations}
\label{app:limitations-list}

\paragraph{The Crossing Law on Pretrained Models.}
The predictor of mKNN built from the crossing law of
Theorem~\ref{thm:crossing} with an estimated coupling is not validated.
Every candidate estimator of the coupling failed the
preliminary check against two baseline predictors, so the preregistered
out-of-sample test was not run (Section~\ref{sec:exp-scale}). On the 29
pairs without sentence encoders, this predictor recovers the sign of every
non-flat slope, as does an affine fit of the coupling with Gaussian noise.
On the vision--language pairs among them, the predicted fall of mKNN with
$m$ is too slow, and the predictor errs more in the level of mKNN than a
no-change prediction from $m=1{,}024$ (Appendix~\ref{app:exp-s11}).

\paragraph{Levels of mKNN between Vision and Language Models.}
The mKNN scores of our vision--language pairs depend on the depth of the
vision features, the text model, and the protocol. With last-layer instead
of mid-depth vision features, mKNN at $k/m=0.01$ rises by a median factor
of $2.8$ with language models and $4.3$ with sentence encoders
(exploratory, on COCO with ViT-B models). At the last layer, DINOv2-B and
MPNet reach an mKNN of $0.43$ with centered features and five averaged
captions, slightly above the mKNN that \citet{schnaus2026rosetta} report
for these two models. At this depth, mKNN still falls with $m$ and stays
close to stable at a fixed $k/m$ (Appendix~\ref{app:exp-s11}).

\paragraph{Linear Heads in Trained Networks.}
In CIFAR-5m networks with a linear head, outside the premises of
Theorem~\ref{thm:gauge}, our results contradicted several preregistered
predictions (Section~\ref{sec:exp-gauge}). Without weight
decay, the coordinates $u$ agreed on neighbor sets in all 96 pairs
at $250{,}000$ images, whereas our extension of part~2 of this theorem to
neighbor sets had predicted unreliable agreement at both sizes.
From a small initialization, runs under momentum or gradient flow of the
head agreed less reliably on the scale-normalized geometry of $u$
than runs with weight decay under momentum.

The disagreement between runs also has sources outside
Theorem~\ref{thm:gauge}, which concerns the mismatch between the gauges of
two runs. Under gradient flow of the head, the coordinates of the balanced
heads of two networks still differed more in geometry than their outputs
$y$. Under momentum without weight decay, a gauge shared by all networks
trained this way produced most of the excess of $u$ over $y$ in the $G_s$
discrepancy at $250{,}000$ images (preregistered,
Appendix~\ref{app:exp-s22-v11}).

\paragraph{Contrastive Objectives in Trained Networks.}
In CIFAR-5m networks outside the premises of Theorem~\ref{thm:link}, the
representations trained with InfoNCE and with the spectral contrastive
loss differed even in their neighbor sets, against the preregistered
prediction (Section~\ref{sec:exp-link}). InfoNCE networks at two
temperatures lie outside the rescaling case of
Remark~\ref{rem:temperature}. In these networks, more data did not shrink
the log-distance error of near pairs of images faster than that of far
pairs, against the prediction derived from Lemma~\ref{lem:cooccurrence}
(Appendix~\ref{app:exp-s21}).

\paragraph{Statistical Scope of the Size Trends.}
Each size trend in Section~\ref{sec:experiments} rests on one draw of
evaluation inputs. The loss of the CKA trend under calibration rests on
BLOOMZ, LLaMA-1, and OpenLLaMA, with exact one-sided $p$-values of at
least $1/120$ along BLOOMZ and $1/6$ along the other two families. On
Pythia, the one family with a confirmatory test, the raw score outgrew the
threshold and calibration kept the trend. This loss therefore still awaits
a confirmatory test on a family that meets the condition of part~2 of
Corollary~\ref{cor:aggregation}.

\paragraph{Scope of the Lattice.}
A criterion of Definition~\ref{def:criterion} compares two representations
on paired inputs, whereas alignment without pairs
\citep{schnaus2026rosetta} also depends on whether the correspondence can
be identified from each space alone, a property untested by the criteria
of the lattice.

\subsection{Open Problems}
\label{app:open-theory}

\paragraph{Scales and Baselines.}
A prediction of mKNN on a larger evaluation set needs the dispersion of
the coupling at shorter distances, i.e., the spread of the distances of
model B at a given short distance of model A. The main open problem on the
measurement side is therefore a law linking this dispersion across scales,
together with the local Poisson limit needed to prove
Conjecture~\ref{conj:scale}. Before an agreement is read as convergence,
calibration also needs baselines that keep structure shared for reasons
unrelated to convergence, such as a common dependence on caption length.

\paragraph{Gauge Selection by Practical Optimizers.}
On the learning side, the selection of the gauge under momentum and finite
steps remains open, since both change the imbalance $C$ conserved by
gradient flow in Theorem~\ref{thm:gauge} (Remark~\ref{rem:optimizers}). A
law for the change of $C$ under such training and under adaptive
optimizers such as Adam would predict the criteria met reliably by two
runs of the linear autoencoder. In networks, this law would predict the
imbalance part of the disagreement between two runs.

\paragraph{Objectives and the Two Branches.}
In the finite Gibbs model, learners trained with two contrastive
objectives learn one relation through different increasing links as both
samples grow (Theorem~\ref{thm:link}). A model in which a change of
objective alters the learned relation itself is still lacking, and the
results of our contrastive networks (Appendix~\ref{app:limitations-list})
motivate such a model. This model could also give a learning condition
under which agreement becomes reliable on neighbor sets and not on
spanning trees, or the reverse, for a test of part~2 of
Hypothesis~\ref{hyp:lrh} on these two branches.

\section{Extended Potential Applications}
\label{app:open-problems}

For each area of Section~\ref{sec:discussion}, this appendix describes
the methods of the cited works, the implications of our results for these
works, and our proposed tests. These tests are left to future work, and
the statements marked as expectations or conjectures are not proved in
this paper.

\subsection{Pretraining, Scaling, and Multimodal Encoders}
\label{app:op-pretraining}

\paragraph{Size Trends at One Scale.}
PRH and later studies of multimodal alignment relate model size and
performance to mKNN between vision and language models, computed with
$k=10$ on $1{,}024$ image--caption pairs and a maximum over layer pairs
\citep{huh2024platonic,tjandrasuwita2025understanding}. A neighbor
criterion compares two models at the scale $k/m$
(Theorem~\ref{thm:crossing}), and size trends reported at several values
of $k/m$ would show whether larger models agree more on near neighbors or
on coarse structure alone. For a calibrated maximum of CKA over layer
pairs, as in ARH, the size trend also depends on whether the largest
random-pairing baseline grows faster than the raw maximum
(Corollary~\ref{cor:aggregation}). \citet{cherti2023reproducible} trained
CLIP models of several architectures on $80$ million to $2$ billion
image--text pairs. These models would allow a preregistered test of this
corollary in which the training data grow at fixed depth and width, unlike
the language-model families of Section~\ref{sec:exp-aggregation}, along
which size, depth, and width grow together.

\paragraph{Width and Superposition.}
Besides the baseline of Proposition~\ref{prop:baselines}, the raw CKA in
these size trends can depend on width. Under superposition, i.e., with
more features than dimensions, CKA between two networks storing the same
features depends on the overlaps between the feature directions of each
network \citep{liu2026similarity,shi2026superposition}. For uncorrelated
features of equal variance and directions drawn independently at random,
this CKA is approximately the ratio of the width to the number of features
\citep{liu2026similarity}. Raw CKA can then rise with width while the
stored features stay the same.

\paragraph{A Conjecture on the Losses of CLIP and SigLIP.}
Part~2 of Hypothesis~\ref{hyp:lrh} also bears on pretraining runs that
differ in their objective, such as the softmax loss of CLIP
\citep{radford2021learning} and the sigmoid loss of SigLIP
\citep{zhai2023sigmoid}. SigLIP scores each pair by a logit with a learned
scale and bias, and treats each pair as a separate binary classification
\citep{zhai2023sigmoid}. In the finite Gibbs model of
Appendix~\ref{app:gibbs}, we apply this loss with the logit
$t'K(x,x')+b$ of a critic $K$ to the positive pair and to the $J$
negatives of each example. We state the resulting limit without proof as
Conjecture~\ref{conj:sigmoid}.

\begin{conjecture}[Sigmoid Loss in the Finite Gibbs Model]
\label{conj:sigmoid}
In the finite Gibbs model with balanced masses, let a learner minimize the
empirical sigmoid risk with logit $t'K(x,x')+b$ over unit-diagonal
positive semidefinite critics $K$, scales $t'>0$, and biases $b$, and let
it be able to represent its population minimizers. As the sample grows,
the distance of its critic on the probes to the set
$\{aK_\star+(1-a)\ind\ind^\top:a>0\}$ tends to zero in probability.
\end{conjecture}

The conjecture rests on the fact that the logit minimizing a logistic loss
is the logarithm of the ratio between the rates at which a pair occurs as
a positive and as a negative. With $J$ negatives for each positive pair
and balanced masses, this logarithm equals $K_\star/t-\log(JZ)$, and the
unit diagonal of $K$ then gives $K=aK_\star+(1-a)\ind\ind^\top$ with
$a=1/(tt')$. By part~1 of Theorem~\ref{thm:link} and
Remark~\ref{rem:temperature}, such a limit differs from the InfoNCE limit
$K_\star$ by a rescaling of all squared distances, ignored by every
criterion on scale-normalized geometry. If the conjecture holds, learners
trained with the two losses meet every criterion on scale-normalized
geometry in the limit, unlike
the InfoNCE and spectral learners of Theorem~\ref{thm:link}. A comparison
of image encoders trained on the same data with the losses of CLIP and
SigLIP would test this prediction outside the model, where our InfoNCE and
spectral networks differed even in their neighbor sets
(Section~\ref{sec:exp-link}).

\paragraph{Pretraining Runs from Different Seeds.}
PolyPythias \citep{vanderwal2025polypythias} provides ten pretraining runs
per model size with checkpoints along training, and we compare four pairs
of these runs at 410M parameters by mKNN in
Section~\ref{sec:exp-scale}. By the rule of
Section~\ref{sec:exp-protocol}, the five disjoint pairs of runs of one size at a checkpoint can show
unreliable agreement under a criterion, whereas reliable agreement needs
at least $29$ agreeing disjoint pairs. Estimates of the probability $p_Q$
that two runs meet a criterion $Q$ would thus show where in the lattice
agreement across pretraining runs is unreliable during training.

\paragraph{Connectors of Multimodal Models.}
In LLaVA \citep{liu2023visual}, a trained linear layer maps the features
of a frozen CLIP encoder into the word-embedding space of a language
model. This connector can undo any invertible linear map $V$ of the vision
features, i.e., a change of gauge as in Section~\ref{sec:gauge}. Recent
work selects encoders by CKA \citep{maniparambil2025harnessing} or by a
Gromov--Wasserstein distance \citep{li2026rethinking} before the training
of connectors, and both scores change under such maps. With an $\ell_2$
penalty on the connector in its loss, the function learned by the
connector depends on $V$ through $V^\top V$ alone, as the probe Gram
matrix of Theorem~\ref{thm:gauge} depends on its gauge through
$\Omega=V^\top V$. A score for choosing encoders then needs invariance to
orthogonal maps under this penalty. Without regularization, every
invertible map leaves the set of minimizing maps from the original
features unchanged, but gradient training can select among several
minimizers according to the map, as in Theorem~\ref{thm:gauge}. An
experiment applying maps of growing condition number to the vision
features before the training of the connector would test whether the
trained multimodal model depends on these maps through $V^\top V$ alone.

\subsection{Post-training, Fine-tuning, and Reinforcement Learning}
\label{app:op-posttraining}

\paragraph{Low-rank Adaptation.}
Low-rank adaptation (LoRA) adds the product $BA$ of two trainable factors
to a frozen weight matrix \citep{hu2022lora}, and the factors $VA$ and
$BV^{-1}$ give the same update for every invertible $V$, a gauge as in
Section~\ref{sec:gauge}. As in Theorem~\ref{thm:gauge}, gradient flow
keeps the imbalance $AA^\top-B^\top B$ fixed \citep{du2018algorithmic},
whereas an $\ell_2$ penalty on both factors gives $AA^\top=B^\top B$ at
every critical point and restricts the gauge to orthogonal matrices.
\citet{hu2022lora} judge the rank needed by the update through a
comparison of the top right singular directions of $A$ between two random
seeds. These directions coincide in order with the top right singular
directions of $BA$ when the imbalance is a multiple of the identity. LoRA
starts from $B=0$ and a random $A$ with many more columns than rows, which
places the initial imbalance near such a multiple. The AdamW optimizer of
these runs can change this imbalance through its per-coordinate steps and
its momentum (Remark~\ref{rem:optimizers}), whereas its decoupled weight
decay alone multiplies the imbalance by a constant factor. The singular
directions of $BA$ are the same in every gauge, and a comparison of seeds
through these directions would separate a difference between updates from
a difference between gauges.

\paragraph{Model Merging.}
LoraHub composes adapters trained on different tasks by multiplying
weighted sums of their factors \citep{huang2024lorahub}, and the composed
update changes with the gauge of each adapter. We propose to test whether
this composition improves for adapters trained with an $\ell_2$ penalty
added to the loss and aligned by an orthogonal Procrustes rotation of
their factors. For two adapters that learn one update, factors balanced
and aligned in this way compose into this update, whereas for adapters of
different tasks we state this improvement as a hypothesis.

\paragraph{Reinforcement Learning and Supervised Fine-tuning.}
Studies of post-training find that reinforcement learning (RL) changes the
representation of the base model less than supervised fine-tuning (SFT)
\citep{shenfeld2026rl,huan2026math}. In an appendix analysis,
\citet{shenfeld2026rl} measure this change by CKNNA
\citep{huh2024platonic}, a CKA restricted to pairs of samples that are $k$
nearest neighbors in both representations. \citet{huan2026math} measure a
shift of mean hidden states along two principal components. A nonlinear
change of all distances of one model that keeps their order leaves the
pairs selected by CKNNA unchanged and can still change the value of CKNNA,
as for CKA (Section~\ref{sec:geometry}). To locate the difference between
RL and SFT in the lattice, we propose to compare each fine-tuned model
with its base model by a rank correlation of all distances and by the
residual around an isotonic link. A gap between RL and SFT that persists
in the rank correlation would indicate that SFT reorders the distances of
the base model. A gap that vanishes in the residual around the isotonic
link would instead indicate an order-preserving change of geometry.

\paragraph{Reward Models.}
Recent work on reward models argues that a reward model learned up to an
increasing transformation of the true reward suffices for sampling-based
optimization \citep{sun2025rethinking}. As an alternative to the
Bradley--Terry model, this work trains a classifier that scores each
response by the logit of its probability of winning a comparison. Under
the Bradley--Terry model, this logit is an increasing and generally
nonlinear function of the reward within each prompt. In the population
limit, the scores of the classifier and of a Bradley--Terry reward model
differ by an increasing link within each prompt, like the limits of the
two learners of Theorem~\ref{thm:link}. Best-of-$N$ sampling selects
responses by the order of their scores within a prompt, so the two reward
models select the same best-of-$N$ responses. The KL-regularized optimal
policy \citep{rafailov2023direct} and the group-normalized advantages of
group relative policy optimization (GRPO) \citep{shao2024deepseekmath} can
instead change with the link. The GRPO advantages stay unchanged when the
rewards in a group take two values, since every increasing link is affine
on two values.

We propose to train policies with GRPO from both reward models, and again
after the scores of the classifier pass through an isotonic map fitted
within each prompt to the Bradley--Terry rewards. The part of the gap
between the policies trained from the two reward models closed by this
isotonic map would measure the effect of the link.

\paragraph{Agreement across Runs of Reinforcement Learning.}
In RL, each run collects its own training data through its policy, whereas
the data distribution is a fixed part of the learning condition in our
theoretical models. In the three-feature task of \citet{shi2026self}, a
policy is the probability of choosing branch $E$ rather than branch $S$,
and a code assigns each of three features a unit direction in two
dimensions. All codes that best fit the data of one policy share the
squared overlaps between these directions, and the codes for the two
deterministic policies differ in which features share a direction. We
assume exact adaptation, i.e., a code that best fits the data of the
current policy at every step. Under the policy dynamics of that work, each
run then converges to the deterministic policy on the same side of an
unstable policy as its initial policy. In the limit, two independent runs therefore
meet the exact criterion $Q$ on these squared overlaps with probability
$p_Q=q^2+(1-q)^2$, where $q$ is the probability that the initial policy of
a run lies below the unstable policy. This probability is at most $0.8$,
the threshold of unreliable agreement in Section~\ref{sec:exp-protocol},
whenever $q$ lies in $[0.113,0.887]$.

In the PPO experiments of \citet{shi2026self}, agents with a rank-two
bottleneck initialized toward different actions also develop different
interference between features. Agreement between the representations of RL
agents trained from different seeds can therefore depend on the policies
reached by these runs. With exact adaptation in the three-feature task,
replay can meet the sufficient condition of \citet{shi2026self} when the
base reward of branch $E$ exceeds that of branch $S$ by more than $7/16$.
Such replay brings every run with a positive initial probability of the
better branch $E$ to this branch, and in the limit the codes of any two
such runs share their squared overlaps. Replay is a candidate intervention
for making agreement across RL runs reliable, in addition to raising
returns.

\subsection{Agents and Harness Engineering}
\label{app:op-agents}

\paragraph{Redundancy and Diversity of Agents.}
Studies of multi-agent systems measure the redundancy
\citep{kim2025towards} or diversity \citep{subramaniam2025multiagent} of
agents by the cosine similarity between embeddings of their outputs. In
the scaling model of agent systems of \citet{kim2025towards}, redundancy
is one predictor, with a main effect that is not significant and an
interaction with the number of agents that the authors call marginal. The
mean cosine between the output embeddings of two agents on the same tasks
has the inner product of their mean normalized embeddings as its
random-pairing baseline, which is large when all outputs share a format.
As in Corollary~\ref{cor:aggregation}, after subtraction of these
baselines, an agent system with the higher raw score ranks below another
system whenever its baseline exceeds the other baseline by more than the
gap between the raw scores. With tasks as evaluation samples, mKNN instead
compares the neighbor sets of each task under the output embeddings of two
agents, with the baseline $k/(m-1)$ for every pair of agents and every
embedding model (Proposition~\ref{prop:baselines}). We propose to refit
the scaling model of \citet{kim2025towards} with redundancy measured after
subtraction of the random-pairing baseline of the mean cosine, or by mKNN.

\paragraph{Communication through Internal States.}
Agents can also share internal states, such as key-value caches within one
backbone \citep{zou2026latent}, activations passed between two models
\citep{ramesh2025communicating}, or caches fused across two models by a
trained network \citep{fu2026cache}. Over six pairs of models, the gain of
activation communication on one task correlates positively with an
uncentered matrix analogue of the squared cosine between the activations
of the two models \citep{ramesh2025communicating}. The random-pairing
baseline of this score grows with the share of the average squared norm of
the activations carried by their mean. A calibrated version of this score
and mKNN on the same probe sentences would test whether the correlation
comes from relations among these sentences shared by the two models or
from their mean activations. For a channel through a linear map fitted by
least squares, an invertible linear change of the sender's activations
leaves the fitted outputs and the projector of Theorem~\ref{thm:gauge}
unchanged and can still lower CKA. Canonical correlations compare the
column spaces of the activations of the two models, and on more probe
sentences than dimensions we therefore expect them to predict the gain of
such a channel better than CKA.

\paragraph{Harnesses and Agent Evaluation.}
The harness of an agent comprises the tools and control code around its
model, and \citet{yang2024swe} compare variants of the agent--computer
interface of SWE-agent by the fraction of resolved instances. The resolved
fractions of six runs of one variant lie within $1.4$ percentage points of
each other, whereas the runs resolve partly different sets of instances
\citep{yang2024swe}. By inclusion and exclusion, the reported pass@1 and
pass@2 of these runs imply that about half of the instances resolved by
either of two runs are resolved by both. The pass$^k$ of $\tau$-bench
measures reliability over such variation as the chance that $k$
independent trials on a task all succeed \citep{yao2025tau}. With the
sampling randomness of an agent in place of training randomness, the
probability $p_Q$ of Section~\ref{sec:learning-law} extends to the chance
that two runs of an agent meet a criterion $Q$ on their outputs. A change
of harness then matters under $Q$ when runs of two variants meet $Q$ less
often than two runs of one variant, with $Q$ ranging from equal resolved
fractions to equal sets of resolved instances.

\subsection{Generative Models}
\label{app:op-generative}

\paragraph{Alignment with a Pretrained Encoder.}
Representation alignment (REPA) speeds up the training of diffusion
transformers by aligning a projection of their hidden states at an early
layer with DINOv2 features of the clean image, patch by patch
\citep{yu2025repa}. \citet{yu2025repa} measure the alignment between SiT
diffusion transformers and DINOv2 by CKNNA at $k=10$ on $10{,}000$ images,
and find that this CKNNA rises with model size and training in diffusion
transformers trained without REPA. REPA-E tunes the autoencoder together
with the diffusion transformer and uses CKNNA as a proxy for generation
quality
\citep{leng2025repae}. The pairs selected by CKNNA change with the scale
$k/m$ of Section~\ref{sec:neighbors}. Since mKNN at a fixed $k$ stays flat
or falls with $m$ according to the model pair
(Section~\ref{sec:exp-scale}), it remains open whether trends and rankings
of CKNNA at one scale hold at other scales.

\paragraph{Losses on Patch Similarities.}
Across 27 target encoders, the generation quality reached with REPA
correlates far more with the spatial structure of the cosine similarities
among the patch features of an encoder than with its linear-probe accuracy
\citep{singh2026irepa}. Losses matching the similarities among the patch
features of a model to those of the encoder improve generation when
applied to a tokenizer \citep{yao2025reconstruction} or to a diffusion
transformer \citep{chen2025sara}. The values of these losses change under
strictly increasing maps of the target similarities, which leave neighbor
overlap unchanged (Proposition~\ref{prop:lattice}) and can change CKA through
the monotone link between the two sets of similarities
(Proposition~\ref{prop:cka}). An experiment that passes the target
similarities through such maps, or replaces these losses with a loss on
the order of the similarities, would test whether generation benefits from
agreement beyond this order.

\paragraph{The Gauge of a Latent Autoencoder.}
In latent diffusion \citep{rombach2022high}, the first convolution of the
decoder can absorb the inverse of an invertible linear map of the latent
channels, which then leaves the reconstruction unchanged. Since the
diffusion process adds isotropic noise to the mapped latent, this map
changes the ratio of signal to noise along each latent direction. As in
the linear autoencoder of Theorem~\ref{thm:gauge}, the training procedure
of the autoencoder can then select this map among those allowed by the
data. Like the weight decay of part~3 of this theorem, the KL penalty of
this autoencoder changes under maps that are not orthogonal. We propose to
test whether the weight of this penalty, the initialization, the
optimizer, or weight decay changes generation quality at equal
reconstruction.

\paragraph{Reproducibility across Training Runs.}
Independently trained diffusion models can generate nearly the same image
from the same noise
\citep{kadkhodaie2024generalization,zhang2024emergence}. A one-step neural
operator trained on the noise paths of a pretrained diffusion sampler also
reconstructs the images of this sampler with high fidelity
\citep{shi2026expanding}. These comparisons concern generated images, and
it remains open which criteria on internal representations are met
reliably by two independent training runs of such models.

\subsection{Mechanistic Interpretability}
\label{app:op-interp}

\paragraph{Recurring Latents and Subspaces.}
Studies in mechanistic interpretability ask whether neurons and the
latents of sparse autoencoders (SAEs) recur across training runs and
models \citep{gurnee2024universal,paulo2026sparse}. For an unstable
latent, i.e., a latent without a counterpart in other runs, a description
inferred in one run need not transfer to another run
\citep{shi2026superposition}. SAEs trained from different seeds share part
of their latents under one-to-one
matching of directions \citep{paulo2026sparse}, and the latents rarely
matched across seeds concentrate in lower-rank subspaces that recur across
runs \citep{gerasimov2026unstable}. \citet{gerasimov2026unstable} suggest
that this seed dependence reflects ambiguity among bases, and the survey
of \citet{shi2026superposition} concludes that latent reproducibility
depends on the objective, the initialization, and the level of comparison.

In a linear autoencoder, all global minimizers on one sample share one projector while
the training procedure selects the invertible map undone by the linear
head (Theorem~\ref{thm:gauge}). The exact symmetries of an SAE are
permutations and positive scalings of its latents, and an approximate
gauge can remain within a group of co-active latents, such as a recurring
subspace. For SAEs, we expect the initialization and the regularization to
decide whether runs from different seeds agree reliably on single
directions within these subspaces, and to matter less for their agreement
on projectors onto the subspaces. SAEs are trained with adaptive
optimizers, and by analogy with Remark~\ref{rem:optimizers}, runs with and
without weight decay would give a direct test of this expectation.

\paragraph{Universal Neurons.}
The study of universal neurons in GPT2 \citep{gurnee2024universal} scores
a neuron by its maximum correlation with the neurons of another seed minus
its maximum correlation with randomly rotated neurons of that seed,
averaged over four other seeds. By this score, GPT2-medium has a smaller
share of universal neurons than GPT2-small. As in
Corollary~\ref{cor:aggregation}, this smaller share can come from the
subtracted maximum if the maximum over rotated neurons grows with model
size faster than the maximum over the neurons of the other seed. A share
computed with a baseline fixed across sizes would separate the effect of
the subtracted maximum from a difference in universality.

\paragraph{Top-activating Tokens.}
Descriptions of latents are often inferred from their top-$k$ activating
tokens \citep{shi2026superposition}. With the negative activation of a
latent as the distance, these tokens form a neighbor set at the scale
$k/m$ in a corpus of $m$ tokens. As in Theorem~\ref{thm:crossing}, the
overlap between the top-$k$ tokens of two matched latents is unchanged by
strictly increasing maps of either activation, and at a fixed $k$ this
overlap can fall with the corpus size.

\paragraph{CKA and Neighbor Sets under Superposition.}
Under superposition, CKA between two networks storing the same sparse
features can stay low even where these features remain recoverable
\citep{liu2026similarity}, and it remains open whether such a pair meets
neighbor criteria. Suitably scaled random projections with enough
dimensions for this recovery multiply every squared distance between
sparse feature vectors by a factor in $[1-\gamma,1+\gamma]$ with high
probability \citep{baraniuk2008simple}. By the argument of
Corollary~\ref{cor:gauge-neighbors}, the representations given by two such
projections share the $k$-neighbor set of every input $i$ whose ratio
$D_{A,i,(k+1)}/D_{A,i,(k)}$ exceeds $((1+\gamma)/(1-\gamma))^2$. Since
this ratio can be large at the boundary of a group of inputs sharing most
of their features, we expect mKNN at the scale of such groups to stay high
where CKA is low.

\subsection{Proteins, Molecules, and Materials}
\label{app:op-science}

\paragraph{Conflicting Studies of Interatomic Potentials.}
Studies of scientific foundation models compare the representations of
molecules, materials, and proteins across models and take the agreement
found as evidence that these models share a representation of matter
\citep{edamadaka2025universally,li2026platonicmlip}. For interatomic
potentials, \citet{edamadaka2025universally} find that the representations
of conservative Orb-v3 models align strongly with those of MACE, whereas
\citet{li2026platonicmlip} find near-zero neighbor overlap between Orb-v3
models and equivariant models including MACE. The first study computes
CKNNA with $k=25$ on features averaged over the atoms of each structure.
The second compares the neighbor sets of single atoms in MP-20 crystals
and leaves unstated whether these sets are computed on the original
embeddings or on their projection onto $100$ shared anchor atoms. Under
part~1 of Hypothesis~\ref{hyp:lrh}, such a conflict can come from criteria
applied at different points of the lattice, and it can also come from the
data, the level of the features, or the anchor projection. A comparison of
both criteria on shared samples and features, in the original and in the
projected spaces, would separate these causes.

\paragraph{Convergence with Performance.}
The representations of interatomic potentials also grow more similar to
those of the best-performing model with decreasing energy error, by CKNNA
on $1{,}000$ OMat24 structures and by distance correlation on $1{,}000$
QM9 molecules \citep{edamadaka2025universally}. These scores are computed
without calibration against random pairings, although calibration can
reverse a trend when the random-pairing baseline grows faster than the raw
score along the compared models (Corollary~\ref{cor:aggregation}).
For CKA and mKNN, this baseline can be forecast from single-model features
(Proposition~\ref{prop:baselines}), whereas the baselines of distance
correlation and CKNNA need permutations. Recomputed with CKA and mKNN,
these trends and the forecast baselines would show whether the convergence
survives calibration.

\paragraph{Homology among Proteins.}
For proteins, \citet{edamadaka2025universally} use $k=25$ on $50{,}000$
structures drawn at random from the Protein Data Bank, which sets a fine
scale $k/m$ at which the neighbors of a protein can be its close homologs.
A comparison of each model with a reference representation built from
sequence identity would show how much of the agreement between sequence
and structure models comes from homology.

\paragraph{Linear Heads in Molecule--text and Protein Models.}
Molecule--text models such as MoleculeSTM \citep{liu2023multimodal} and
structure-aligned protein language models \citep{chen2025structure} pass
encoder features through linear heads into a shared space. An invertible
map of these features undone by the heads then leaves the training loss
unchanged, a gauge as in Section~\ref{sec:gauge}. In our CIFAR-5m
networks trained without weight decay, the coordinates of two runs before
the last linear map differed more in mean-normalized squared distances
than the outputs of the networks, and weight decay removed this excess
(Section~\ref{sec:exp-gauge}). In the released code of MoleculeSTM, the
default weight decay is zero, and runs from several seeds with and without
weight decay would test whether its encoder features carry the same
excess.

\subsection{Dynamical Systems and Operator Learning}
\label{app:op-dynamics}

\paragraph{The Gauge of DeepONet.}
DeepONet predicts through the inner product of two learned factors, its
branch and trunk outputs \citep{lu2021learning}, and an invertible map
applied to the branch outputs and undone on the trunk outputs leaves every
prediction unchanged. We conjecture that joint training selects this map
as in Theorem~\ref{thm:gauge}. The branch outputs of two runs would then
share their column space, while CKA between these outputs could stay below
one for large samples. Two-step training \citep{lee2024training} and
POD-DeepONet \citep{lu2022comprehensive} instead make the trunk outputs
orthonormal over the output locations. The chaos-expansion operators SDENO
and F-SPDENO of \citet{shi2026expanding} combine learned propagators with
fixed orthonormal Wick features of the driving noise. In two runs whose orthonormal factors span the same space,
least-squares fits of the other factor differ by an orthogonal map, so CKA
between these fits equals one.
Tied weights restrict the gauge in the same way in part~3 of
Theorem~\ref{thm:gauge}. Fixed Wick features remove the gauge in these
operators, and their learned propagators can be compared across runs
coordinate by coordinate. The gauge remains in variants that multiply two
learned factors, such as G-SDENO.

\paragraph{Resolution Transfer.}
Across grid resolutions, \citet{bartolucci2023representation} find
aliasing errors in Fourier neural operators and note that discretization
invariance \citep{kovachki2023neural} concerns the limit of infinite
resolution. With the two discretizations of each test function as one
evaluation sample, a neighbor criterion compares an operator at two finite
resolutions. By Theorem~\ref{thm:crossing}, its outcome depends on the
test functions moved across each neighbor boundary by the change of
resolution. Such criteria give a test of resolution transfer at the
resolutions used in practice and at a chosen scale $k/m$.

\paragraph{Criteria on Dynamics.}
Recurrent networks trained on the same task can share the topology of
their fixed points and their linearized dynamics while their states differ
in geometry under SVCCA, CKA, and Procrustes analysis
\citep{maheswaranathan2019universality,ostrow2023beyond}. A criterion on
dynamics can thus accept a pair of networks rejected by a criterion on
geometry, which suggests extending the lattice of
Section~\ref{sec:criteria} from Gram matrices on fixed inputs to criteria
on dynamics. For chaotic systems, \citet{jiang2023training} add an
optimal-transport loss or a contrastive feature loss to a relative
root-mean-squared error of neural operators and report smaller errors in
long-run statistics. Under an extension of part~2 of
Hypothesis~\ref{hyp:lrh} to criteria on long-run statistics, we expect the
objective to decide whether two independently trained operators share
these statistics with probability at least $1-\delta$ over training runs.

\paragraph{Criteria Indexed by the Chaos Order.}
For dynamics driven by Brownian noise, the output of a surrogate has an
expansion in the orthonormal Wick features used by
\citet{shi2026expanding}, so the mean squared difference between two
surrogates splits into one term per chaos order. This split gives a set of
criteria indexed by the chaos order, from the Gaussian response to the
noise at order one to nonlinear interactions of the noise at higher
orders.

\end{document}